\documentclass{article}
\usepackage{ijcai24}

\usepackage{times}
\usepackage{soul}
\usepackage{url}
\usepackage[hidelinks]{hyperref}
\usepackage[utf8]{inputenc}
\usepackage[small]{caption}
\usepackage{graphicx}
\usepackage{amsmath,amsthm,amssymb,amsfonts}
\usepackage{booktabs, array}
\usepackage{algorithm}
\usepackage{algorithmic}
\usepackage{float}
\usepackage[switch]{lineno}
\usepackage{multirow,makecell}
\usepackage{bm}
\usepackage{comment}
\usepackage[short]{optidef}

\usepackage{mathtools}

\DeclarePairedDelimiterX{\infdivx}[2]{(}{)}{%
  #1\;\delimsize\|\;#2%
}
\DeclarePairedDelimiterX{\infdist}[2]{(}{)}{%
  #1, #2%
}
\newcommand{\infdiv}{D\infdivx}
\newcommand{\tvinfdiv}{D_{\text{TV}}\infdist}
\newcommand{\tvpiinfdiv}{D_{\text{TV}}^\pi\infdist}
\newcommand{\klinfdiv}{D_{\text{KL}}\infdivx}
\newcommand{\jsinfdiv}{D_{\text{JS}}\infdivx}
\newcommand{\winfdiv}{D_{\text{W}}\infdist}

\newcommand\T{\rule{0pt}{2.6ex}}       
\newcommand\B{\rule[-1.6ex]{0pt}{0pt}} 
\newcommand\TSMALL{\rule{0pt}{2.3ex}}       
\newcommand\BSMALL{\rule[-1.2ex]{0pt}{0pt}} 

\newcolumntype{P}[1]{>{\centering\arraybackslash}p{#1}}

\newtheorem{lemma}{Lemma}

\newtheorem{corollary}{Corollary}
\newtheorem{proposition}{Proposition}

\title{A Conservative Approach for Few-Shot Transfer in Off-Dynamics Reinforcement Learning}

\author{
Paul Daoudi$^{1}$\and Christophe Prieur$^2$\and
Bogdan Robu$^{2}$\and Merwan Barlier$^{1}$\And Ludovic Dos Santos$^{3}$\\
\affiliations
$^1$Huawei Noah’s Ark Lab \\
$^2$GIPSA Lab \\
$^3$Criteo AI Lab \\
\emails
\{paul.daoudi1, \ merwan.barlier\}@huawei.com, 
l.dossantos@criteo.com, \\
\{christophe.prieur, \ bogdan.robu\}@gipsa-lab.grenoble-inp.fr
}

\begin{document}

\maketitle

\begin{abstract}
    Off-dynamics Reinforcement Learning (ODRL) seeks to transfer a policy from a source environment to a target environment characterized by distinct yet similar dynamics. In this context, traditional RL agents depend excessively on the dynamics of the source environment, resulting in the discovery of policies that excel in this environment but fail to provide reasonable performance in the target one. In the few-shot framework, a limited number of transitions from the target environment are introduced to facilitate a more effective transfer. Addressing this challenge, we propose a new approach inspired by recent advancements in Imitation Learning and conservative RL algorithms. The proposed method introduces a penalty to regulate the trajectories generated by the source-trained policy. We evaluate our method across various environments representing diverse off-dynamics conditions, where access to the target environment is extremely limited. Across most tested scenarios, our proposed method demonstrates performance improvements compared to existing baselines.
\end{abstract}

\section{Introduction}

Traditional online Reinforcement Learning (RL) is a promising path to obtain a near-optimal policy for complex systems within  a general trial-and-error framework. However, this standalone mechanism faces numerous challenges \cite{dulac2021challenges} when applied to real-world systems, especially in scenarios where interactions can be prohibitively expensive due to safety \cite{garcia2015comprehensive,achiam2017constrained} or time \cite{weiss2016survey} considerations. In such context, one alternative is to leverage a cheap source environment, usually built as a simplification of the target environment. While these environments may also differ w.r.t. their observations \cite{gamrian2019transfer} and rewards \cite{barreto2017successor}, we study the off-dynamics setting where the mismatch is between their dynamics \cite{hofer2021sim2real}. It is a relevant consideration when it is possible to estimate or simplify the physics of the target environment, for instance within a simulator. We also consider that the transition probabilities of both the source and target environments are unavailable to reflect any kind of simulator.

The off-dynamics framework is particularly challenging: direct policy transfer from the source to the target environment usually fails \cite{muratore2019assessing,ju2022transferring} due to compounding errors. The small discrepancies between the transition probabilities may accumulate over time, leading to increasing deviations between the trajectories of the source and target environments over time. Worse, modern optimization-based agents may exploit these discrepancies to find policies that perform exceptionally well in the source but result in trajectories that are impossible to replicate in the target.

In general, in addition to relying on a source environment, it is still possible to deploy the agent in the target system to collect data. However, this deployment is limited due to the mentioned safety and time considerations: the available data is this often limited to a few narrow trajectories. Two orthogonal approaches are possible to include this data in the derivation of the policy. The first one - well studied \cite{abbeel2006using,ZhuKBB18,desai2020imitation,hanna2021grounded} - leverages this data to improve the source domain, then learns a traditional RL agent on the upgraded system. The second approach maintains the source environment fixed and biases the learning process to account for the dynamics discrepancies \cite{koos2012transferability}. This second line of work is complementary to the first one, as both could be combined to make best use of the limited target samples. To the best of our knowledge, only a few works have taken this purely off-dynamics direction, and even fewer have focused on the low data regime scenario. Currently, the most prominent approach is DARC \cite{eysenbach2020off} which modifies the reward function to search for parts of the source that behave similarly to the target system. Although this method is effective for a few classes of problems, such as the "broken" environments, we have found that it may fail in others, limiting its application to a restrictive class of discrepancies between the environments.

In this paper, we introduce the \textbf{F}ew-sh\textbf{O}t \textbf{O}ff \textbf{D}ynamics (FOOD) algorithm, a conservative method that penalizes the derived policy to be around the trajectories observed in the target environment. We theoretically justify this method, which directs the policy towards feasible trajectories in the target system, and thus mitigates the potential trajectory shifts towards untrustworthy regions of the source system. Our regularization takes the form of a divergence between visitation distributions and can be practically implemented using state-of-the-art techniques from the Imitation Learning (IL) literature \cite{hussein2017imitation}. Our method is validated on a set of environments with multiple off-dynamics disparities. We show that, compared to other baselines, our approach is the most successful in exploiting the few available data. Our agent is also shown to be relevant for a wider range of dynamic discrepancies.

\section{Related Work}


The off-dynamics setting has been studied in two distinct contexts, depending on the accessibility of the agent to transitions from the target environment, referred to as "zero-shot" and "few-shot" off-dynamics RL.

\paragraph{Zero-Shot Off-Dynamics RL.} Sampling data from the target environment can be impossible due to strict safety constraints or time-consuming interactions. In such cases, the source environment is used to ensure robustness to guarantee a certain level of performance without sampling from the target system. It can take many forms. One possible choice is domain randomization \cite{mordatch2015ensemble} where relevant parts of the source system are randomized to make it resilient to changes. Another line of work focuses on addressing the worst-case scenarios under stochastic source dynamics \cite{abdullah2019wasserstein}. Robustness can also be achieved w.r.t. actions \cite{jakobi1995noise,tessler2019action}, that arise when certain controllers become unavailable in the target environment. These techniques are outside the scope of this paper as they do not involve any external data in the learning process.

\paragraph{Few-Shot Off-Dynamics RL.} When data can be sampled from the target environment, two orthogonal approaches have been developed to propose efficient agents. The first approach, well established, is to improve the accuracy of the source environment. The parameters of the source physics can be optimized directly if available \cite{ZhuKBB18,tan2018sim}. Otherwise, expressive models can be introduced to improve the source dynamics model \cite{abbeel2006using}. Within this category, a family of methods builds an action transformation mechanism that - when taken in the source system - produces the same transition that would have occurred in the target environment \cite{hanna2021grounded}. In particular, GARAT \cite{desai2020imitation} leverages recent advances in Imitation Learning from Observations \cite{TorabiWS19a} to learn this action transformation and ground the source environment with only a few trajectories. All these algorithms are orthogonal to our work since once the source system has been improved, a new RL agent has to be trained.

The second approach, more closely aligned with our work, is the line of inquiry that modifies the learning process of the RL policy in the source to be efficient in the target environment. One group of approaches creates a policy - or policies - that can quickly adapt to a variety of dynamic conditions \cite{arndt2020meta,yu2020learning,KumarFPM21}. It requires the ability to set the parameters of the source dynamics model which may not always be feasible, e.g., if the model is a black box. A more general algorithm is DARC \cite{eysenbach2020off}. It learns two classifiers to distinguish transitions between the source and target environments and incorporates them into the reward function to account for the dynamics shift. Learning the classifiers is easier than correcting the dynamics of the source environment, but as we will see in the experiments, this technique seems to work mainly when some regions of the source environment accurately model the target environment and others don't. Another related work is H2O \cite{niu2022trust} which extends the approach by considering access to a fixed dataset of transitions from a target environment. It combines the regularization of the offline algorithm CQL \cite{kumar2020conservative} with the classifiers proposed by DARC. However, the performance of H2O depends on the amount of data available. In fact, it performed similarly, or worse, to the pure offline algorithm when only a small amount of target data was available \cite[Appendix~C.3]{niu2022trust}.

\section{Background}

\subsection{Preliminaries}

Let $\Delta(\cdot)$ be the set of all probability measures on $(\cdot)$. The agent-environment interaction is modeled as a Markov Decision Process (MDP) $\left( \mathcal{S}, \mathcal{A}, r, P, \gamma, \rho_0 \right)$, with a state space $\mathcal{S}$, an action space $\mathcal{A}$, a transition kernel $P : \mathcal{S} \times \mathcal{A} \to \Delta(\mathcal{S})$, a reward function $r : \mathcal{S} \times \mathcal{A} \times \mathcal{S} \to \left[ R_{\text{min}}, R_{\text{max}} \right]$, the initial state distribution $\rho_0$ and a discount factor $\gamma\in \left[0,1\right)$. A policy $\pi : \mathcal{S} \to \Delta(\mathcal{A})$ is a decision rule mapping a state over a distribution of actions. The value of a policy $\pi$ is measured through the value function $V^\pi_P(s) = \mathbb{E}_{\pi, P}\left[ \sum_{t=0}^\infty \gamma^t r(s_t, a_t, s_{t+1}) \vert s_0=s \right]$. The objective is to find the optimal policy maximizing the expected cumulative rewards $J^\pi_P = \mathbb{E}_{\rho_0}\left[ V^\pi_P(s) \right]$. We also define the $Q$-value function $Q^\pi_P(s, a) = \mathbb{E}_{\pi, P}\left[ \sum_{t=0}^\infty \gamma^t r(s_t, a_t, s_{t+1}) \vert s_0=s, a_0=a \right]$ and the advantage value function $A^\pi_P(s, a) = Q^\pi_P(s, a) - V^\pi_P(s)$. Finally, let $d^\pi_P(s) = (1 - \gamma) \mathbb{E}_{\rho_0, \pi, P}\left[  \sum_{t=0}^\infty \gamma^t \mathbb{P}\left( s_t=s \right) \right]$ the state visitation distribution, as well as its extension to state-action $\mu^\pi_P(s,a)$ and transition $\nu^\pi_P(s,a, s')$. All these quantities are expectations w.r.t. both the policy and the transition probabilities.

The off-dynamics setting involves two MDPs: the source system $\mathcal{M}_{\text{s}}$ and the target environment $\mathcal{M}_{\text{t}}$. We hypothesize that they are identical except for their transition probabilities $P_{\text{s}} \neq P_{\text{t}} $. Our hypothesis states that while most of the MDP parameters are known, the underlying physics of the environment are only estimated. It encapsulates many real-world applications: a model of the dynamics may have been previously learned, or practitioners may have created a simulator based on a simplification of the system's physics. We do not assume access to any parameter modifying the transition probabilities $P_{\text{s}}$ of the source physics to encompass black-box simulators. For readability purposes, we drop the $P$ subscripts for the value functions as they are always associated with the source environment $\mathcal{M}_{\text{s}}$.

Many few-shot off-dynamics agents \cite{abbeel2006using,desai2020imitation,hanna2021grounded} typically employ the following procedure to handle complex environments. First, the policy and value functions are initialized in the source system. The choice of objective at this stage can vary, although a classical approach is to solely maximize the rewards of the source MDP. At each iteration, the policy is verified by experts. If it is deemed safe, $N$ trajectories are gathered in the target environment and saved in a replay buffer $\mathcal{D}_{\text{t}}$. These trajectories are then used to potentially correct the source dynamics and/or the training objective and induce a new policy. This process is repeated until a satisfactory policy is found. This setup is time-consuming and may be risky even when the policy is verified by experts, hence the need to learn with as few data as possible from the target environment.

This work focuses on how to best modify the objective with few trajectories. If handled properly, this could reduce the number of interactions required by the whole process overcoming the need to build a perfect source environment. For the purpose of our study, we assume that $\mathcal{M}_{\text{t}}$ remains fixed throughout the process.

\subsection{Conservative Algorithms and the Visitation Distribution Constraint}

Due to their efficiency and stability, conservative algorithms \cite{schulman2015trust,schulman2017proximal,kumar2020conservative} have shown strong efficiency in various RL settings. Many of them are based on \cite{kakade2002approximately}, where an iteration scheme improves the policy by maximizing a lower bound on the true objective. This process has been extended by refining the lower bound \cite{terpin2022trust,moskovitz2020efficient}, for example by introducing a behavior $b^\pi_P$ that encapsulates any additional property of the MDP \cite{pacchiano2020learning,touati2020stable}.

This family of algorithms is formalized as follows. A policy and a value function are parametrized with respective weights $\theta\in\Theta$ and $\omega\in\Omega$, that we denote from now on $\pi_\theta$ and $V^{\pi_\theta}_\omega$. At each iteration $k$, the policy is improved using the advantage function built from the approximated value function $V^{\pi_{\theta_k}}_{\omega_k}$, with a penalization to respect the lower bound:

\begin{maxi}{\theta\in\Theta}{\mathbb{E}_{\substack{s\sim d^{\pi_{\theta_k}}_{P_{\text{s}}}(\cdot),\\ a\sim\pi_\theta(\cdot|s)}} \left[ A^{\pi_{\theta_k}}_{\omega_k}(s,a) \right] - \alpha_k \ \infdiv*{b^{\pi_{\theta}}_P}{b^{\pi_{\theta_k}}_P},}{}{}
\end{maxi}
where $D$ is any kind of similarity metric and $\alpha_k$ is a hyper-parameter, often set to a constant $\alpha$. TRPO \cite{schulman2015trust} can be retrieved with $b^\pi_P = \pi$, by setting $D$ to be the Kullback-Leibler (KL) divergence and enforcing the penalization as a constraint with a step-size $\epsilon_k$. Alternative behavior options can be found in \cite{pacchiano2020learning,touati2020stable,moskovitz2020efficient}. In particular, \cite{touati2020stable} proposed to encapsulate the whole trajectories induced by $\pi$ and $P$ by setting $b^\pi_P = d^{\pi}_{P}$. It resulted in better results both in terms of sample efficiency and final cumulative rewards than most of its counterparts. This is natural as the new constraint between the state visitation distributions takes the whole trajectories induced by the policy into account, providing more information than the policy alone.

\section{Few-Shot Off-Dynamics Reinforcement Learning}

In this section, we propose a new objective to better transfer a policy learned in the source to the target environment. We extend the conservative objective to the off-dynamics setting. Then, we remind necessary results on Imitation Learning (IL) before deriving our practical algorithm \textbf{F}ew-sh\textbf{O}t \textbf{O}ff \textbf{D}ynamics (FOOD) RL.

\subsection{A New Conservative Off-Dynamics Objective}

Given the discrepancies between the dynamics of the source and the target environment, applying the same policy to both environments may result in different trajectories. This poses a challenge as the agent may make the most of these differences to find policies that produce excellent trajectories in the source environment but are impossible to replicate in the target system.

We analyze the difference between the objectives $J^\pi_P$ associated with the target and source environments, depending on a metric between visitation distributions. For this, we first apply directly the tools from traditional conservative methods \cite{schulman2015trust,achiam2017constrained} to the off-dynamics setting, and propose the following lower bound using state visitation distributions.

\begin{proposition}
\label{prop:trust_region_1}
    Let $J^\pi_P = \mathbb{E}_{\rho_0}\left[ V^\pi_P(s) \right]$ the expected cumulative rewards associated with policy $\pi$, transitions $P$ and initial state distribution $\rho_0$. For any policy $\pi$ and any transition probabilities $P_{\text{t}}$ and $P_{\text{s}}$, the following holds:
    
    \begin{equation}
    J^\pi_{P_{\text{t}}} \geq J^\pi_{P_{\text{s}}} - \frac{2 R_{\text{max}}}{1-\gamma} \left( \tvinfdiv*{d^\pi_{P_{\text{s}}}}{d^\pi_{P_{\text{t}}}} + \tvpiinfdiv{P_{\text{s}}}{P_{\text{t}}} \right),
\end{equation}
with $D_{\text{TV}}$ the Total Variation distance and $\tvpiinfdiv{P_{\text{s}}}{P_{\text{t}}} = \mathbb{E}_{\substack{s\sim d^\pi_{P_{\text{s}}}(\cdot), \\ a\sim\pi(\cdot|s)}} \left[ 
        \tvinfdiv*{P_{\text{s}}\left( \cdot \vert s, a\right)}{P_{\text{t}}\left( \cdot \vert s, a\right)}
        \right]$. 


\end{proposition}

The Proposition also holds by replacing $\tvinfdiv*{d^\pi_{P_{\text{s}}}}{d^\pi_{P_{\text{t}}}}$ with $\tvinfdiv*{\mu^\pi_{P_{\text{s}}}}{\mu^\pi_{P_{\text{t}}}}$. We defer the proof to Appendix~\ref{app:proofs}. It illustrates how the performance of the optimal policy in the target environment may differ from that of the source due to two metrics. The first metric $\tvinfdiv*{d^\pi_{P_{\text{s}}}}{d^\pi_{P_{\text{t}}}}$ quantifies the difference between the visited states of the rollouts in the source and target environments. The second $\tvinfdiv*{d^\pi_{P_{\text{s}}}}{d^\pi_{P_{\text{t}}}}$ describes the difference between the transition probabilities associated with the visited states and the actions following the given policy. These terms must be controlled to allow a good transfer, especially given that they are exacerbated by the factor $\frac{2 R_{\text{max}}}{1-\gamma}$. However, optimizing the second metric is difficult since the transition probabilities of both $\mathcal{M}_{\text{t}}$ and $\mathcal{M}_{\text{s}}$ are unknown. Hence, we propose the following simpler lower bound that considers transition visitation distributions instead.

\begin{proposition}
\label{prop:trust_region}
    Let $J^\pi_P = \mathbb{E}_{\rho_0}\left[ V^\pi_P(s) \right]$ the expected cumulative rewards associated with policy $\pi$, transitions $P$ and initial state distribution $\rho_0$. For any policy $\pi$ and any transition probabilities $P_{\text{t}}$ and $P_{\text{s}}$, the following holds:
    
    \begin{equation}
    J^\pi_{P_{\text{t}}} \geq J^\pi_{P_{\text{s}}} - \frac{2 R_{\text{max}}}{1-\gamma}  \tvinfdiv*{\nu^\pi_{P_{\text{s}}}}{\nu^\pi_{P_{\text{t}}}},
\end{equation}
with $D_{\text{TV}}$ the Total Variation distance.   
\end{proposition}

We also defer the proof to Appendix~\ref{app:proofs}. Here, the lower bound depends on the sole metric $\tvinfdiv*{\nu^\pi_{P_{\text{s}}}}{\nu^\pi_{P_{\text{t}}}}$ that directly quantifies the difference in trajectories. As we will see, this term is easily minimized, especially given that the Total Variation distance could be replaced by other divergences. For instance, the Kullback-Leibler divergence or the Jensen-Shannon divergence could be used thanks to Pinsker's inequality \cite{csiszar1981coding} or the one in \cite[Proposition 3.2]{corander2021jensen}, provided the minimal assumptions of having a finite state-action space and the absolute continuity of the considered measures. Complete details can be found in Appendix~\ref{app:proofs}.

Overall, this lower bound highlights a good transfer between the source and target environment is possible when $\tvinfdiv*{\nu^\pi_{P_{\text{s}}}}{\nu^\pi_{P_{\text{t}}}}$ is small, as the policy induces similar objectives $J^\pi_P$. Inspired by this insight, we adapt conservative methods to the off-dynamics setting and propose a new regularization between trajectories by setting the behaviors $b^\pi_P$ to be the transition visitation distribution respectively associated with the transition probabilities of the source and target environment $\nu^\pi_P$:

\begin{maxi}{\theta\in\Theta}{\mathbb{E}_{\substack{s\sim d^{\pi_{\theta_k}}_{P_{\text{s}}}(\cdot),\\ a\sim\pi_\theta(\cdot|s)}} \left[ A^{\pi_{\theta_k}}_{\omega_k}(s,a) \right] - \alpha \ \infdiv*{\nu^{\pi_\theta}_{P_{\text{s}}}}{\nu^{\pi_{\theta_k}}_{P_{\text{t}}}}.}{}{}
\label{eq:true_objective}
\end{maxi}
The new penalization ensures that the policy is optimized for trajectories that are feasible in the target system, thus preventing the RL agent from exploiting any potential hacks that may exist in the source environment. In addition, remaining close to the data sampled from the target environment can be beneficial when the source system has been constructed using that data, as querying out-of-distribution data can yield poor results \cite{kang2022lyapunov}.

Unfortunately, the difference between the transition probabilities makes the regularization in Equation~\ref{eq:true_objective} difficult to compute. The previous work of \cite{touati2020stable} addressed this by restricting $D$ to $f$-divergences $D_f\left( \mu^{\pi_\theta}_{P} || \mu^{\pi_{\theta_k}}_{P} \right) = \mathbb{E}_{(s, a)\sim \pi_\theta}\left[ f\left( \frac{\mu^{\pi_\theta}_{P}}{\mu^{\pi_{\theta_k}}_{P}} \right) \right]$ and by considering state-action visitation distributions. \cite{touati2020stable} used the DualDICE algorithm \cite{nachum2019dualdice} to directly estimate the relaxed ratio $\frac{\mu^{\pi_\theta}_{P}}{\mu^{\pi_{\theta_k}}_{P}}$ for any policy $\pi_\theta$ sufficiently close to $\pi_{\theta_k}$, eliminating the need to sample data for each policy. However, this method is not applicable to our setting because DualDICE relies on a modified joined Bellman operator, which assumes that both distributions follow the same transition probabilities. Another solution would be to collect at least one trajectory per update. While this would not pose any safety concerns for the data would be sampled in the source system, it can be time-consuming in practice.

\subsection{Practical Algorithm}

In order to devise a practical algorithm for addressing Equation \ref{eq:true_objective}, we aim to get a surrogate objective for $\infdiv*{\nu^{\pi_\theta}_{P_{\text{s}}}}{\nu^{\pi_{\theta_k}}_{P_\text{t}}}$. To construct such proxy, we leverage the recent results from Imitation Learning (IL) \cite{hussein2017imitation} that we briefly recall in this paragraph. In this field, the agent aims to reproduce an expert policy $\pi_e$ using limited data sampled by that expert in the same MDP with generic transition probabilities $P$. Most current algorithms tackle this problem by minimizing a certain similarity metric $D$ between the learning policy's state-action visitation distribution $\mu^{\pi_\theta}_{P}$ and the expert's $\mu^{\pi_e}_{P}$. The divergence minimization problem is transformed into a reward $r_{\text{imit}}$ maximization one, resulting in an imitation value function $V^{\pi}_{\text{imit}} = \mathbb{E}_{\pi, P_{\text{s}}}\left[ \sum_{t=0}^\infty \gamma^t r_{\text{imit}}(s_t, a_t, s_{t+1}) \vert s_0=s \right] $. Since these algorithms are based on data, they can be used to minimize the chosen similarity metric $D$ between two transition visitation distributions with different transition probabilities. Applied to our setting where the target trajectories would be the expert ones, this is formalized as:

\begin{equation} \label{eq:imitation}
    \arg\max_{\theta\in\Theta} V^{\pi_\theta}_{\text{imit}} = \arg\min_\pi \infdiv*{\nu^{\pi_\theta}_{P_{\text{s}}}}{\nu^{\pi_{\theta_k}}_{P_{\text{t}}}}.
\end{equation}
The choices for the divergence $D$ are numerous, leading to different IL algorithms \cite{ho2016generative,fu2017learning,dadashi2020primal}, some of which are summarized in Table~\ref{table:imitation_distances}. For example, Equation~\ref{eq:imitation} is exactly Theorem 1 in \cite{desai2020imitation} in association with GAIL \cite{ho2016generative}, and is also straigthforward with PWIL \cite{dadashi2020primal}.


\begin{table}[ht!]
\centering 
\resizebox{0.99\columnwidth}{!}{\begin{tabular}{|c|c|c|}
  \hline
  GAIL & AIRL & PWIL \TSMALL\BSMALL \\ \hline
  $\jsinfdiv*{X^{\pi_{\theta}}_{P_{\text{s}}}}{X^{\pi_{\theta_k}}_{P_{\text{t}}}}$ 
  & $\klinfdiv*{X^{\pi_{\theta}}_{P_{\text{s}}}}{X^{\pi_{\theta_k}}_{P_{\text{t}}}}$
  & $\winfdiv*{X^{\pi_{\theta}}_{P_{\text{s}}}}{X^{\pi_{\theta_k}}_{P_{\text{t}}}}$ \TSMALL\BSMALL \\ \hline
  
\end{tabular} }
\caption{Objective function for well-known Imitation Learning (IL) algorithms. The variable $X$ can be chosen as either $d$, $\mu$, or $\nu$. Other IL agents can be found in \protect\cite{ghasemipour2020divergence}.}%
\label{table:imitation_distances}
\end{table}

These IL techniques enable efficient estimation of this value function using a small number of samples from $\nu^{\pi_{\theta_k}}_{P_{\text{t}}}$ and unlimited access to $\mathcal{M}_{\text{s}}$. Let $\xi\in\Xi$ be the weights of this parametrized value function. The new regularization is $A^{\pi_{\theta_k}, {\xi_k}}_{\text{imit}}(s,a)$, which can be learned with any suitable IL algorithm. It leads to the practical objective:

\begin{maxi}{\theta\in\Theta}{\mathbb{E}_{\substack{s\sim d^{\pi_{\theta_k}}_{P_{\text{s}}}(\cdot),\\ a\sim\pi_\theta(\cdot|s)}} \left[ A^{\pi_{\theta_k}}_{\omega_k}(s,a) - \alpha \ A^{\pi_{\theta_k}, {\xi_k}}_{\text{imit}}(s,a) \right].}{}{}
\label{eq:proxy_objective}
\end{maxi}

\begin{algorithm}[t]
\hspace*{\algorithmicindent} \textbf{Input:} Algorithms $\mathcal{O}$ and $\mathcal{I}$
\begin{algorithmic}
\STATE Initialize policy and value weights $\theta_0$ and $\omega_0$ with $\mathcal{O}$
\STATE Randomly initialize the weights $\xi_0$
\FOR {$k \in (0, \dots, K-1)$} 
    \STATE Gather $N$ trajectories $\{ \tau_i, \dots, \tau_N \}$ with $\pi_{\theta_k}$ on the target environment $\mathcal{M}_t$ and add them in $\mathcal{D}_{\text{t}}$
    \STATE Learn the value function weights $\omega_{k+1}$ with $\mathcal{O}$ in the source environment $\mathcal{M}_{\text{s}}$
    \STATE Learn the imitation value function weights $\xi_{k+1}$ with $\mathcal{I}$ in $\mathcal{M}_{\text{s}}$ using $\mathcal{D}_{\text{t}}$
    \STATE Learn the policy maximizing \eqref{eq:proxy_objective} using $\mathcal{D}_{\text{t}}$ and $\mathcal{M}_{\text{s}}$ with $\mathcal{O}$
\ENDFOR
\end{algorithmic}
\caption{Few-shOt Off Dynamics (FOOD).}
\label{alg:main_algo}
\end{algorithm}

This new agent is quite generic as it could be optimized with different divergences. It takes as input an online RL algorithm \cite{babaeizadehreinforcement,schulman2017proximal} denoted $\mathcal{O}$ and an Imitation Learning algorithm denoted $\mathcal{I}$. The whole off-dynamics algorithm process, which we denote Few-shOt Off Dynamics (FOOD) RL, is described as follows. First, the policy and the value weights are initialized in the source environment with $\mathcal{O}$. At each iteration $k$, the agent samples $N$ new trajectories with $\pi_{\theta_k}$\footnote{These trajectories could first be used to improve the source system.}. Subsequently, the policy, traditional, and imitation value functions are retrained on the source environment with $\mathcal{O}$ and $\mathcal{I}$ according to Equation~\ref{eq:true_objective}. The whole algorithm is summarized in Algorithm~\ref{alg:main_algo}.

\section{Experiments} 
\label{sec:experiments}

In this section, we evaluate the performance of the FOOD algorithm in the off-dynamics setting in environments presenting different dynamics discrepancies, treated as black box simulators. The code can be found at \url{https://github.com/PaulDaoudi/FOOD}.

The environments are based on Open AI Gym \cite{brockman2016openai} and the Minitaur environment \cite{coumans2021} where the target environment has been modified by various mechanisms. These include gravity, friction, and mass modifications, as well as broken joint(s) systems for which DARC is known to perform well \cite[Section~6]{eysenbach2020off}. We also add the Low Fidelity Minitaur environment, highlighted in previous works \cite{desai2020imitation,yu2018policy} as a classical benchmark for evaluating agents in the off-dynamics setting. In this benchmark, the source environment has a linear torque-current relation for the actuator model, and the target environment - proposed by \cite{tan2018sim} - uses accurate non-linearities to model this relation.

All of FOOD experiments are carried out using both GAIL \cite{ho2016generative}, a state-of-the-art IL algorithm, as $\mathcal{I}$. We find that GAIL performed similarly, or better than other IL algorithms such as AIRL \cite{fu2017learning} or PWIL \cite{dadashi2020primal}. FOOD is tested with its theoretically motivated metric between transition visitation distributions $\nu^\pi_P$, as well as with $d^\pi_P$ and $\mu^\pi_P$ for empirically analyzing the performance associated with the different visitation distributions. The performance of our agent with the different IL algorithms can be found in Appendix~\ref{app:il_comparison}. We compare our approach against various baselines modifying the RL objective, detailed below. They cover current domain adaptation, robustness, or offline Reinforcement Learning techniques applicable to our setting. Further details of the experimental protocol can be found in Appendix~\ref{app:exp_details}. Especially, a study with respect to the number of available target data is done in Appendix~\ref{app:study_data}.

\begin{itemize}
    
    \item \textbf{DARC} \cite{eysenbach2020off} is our main baseline. It is a state-of-the-art off-dynamics algorithm that introduces an importance sampling term in the reward function to cope with the dynamics shift. In practice, this term is computed using two classifiers that distinguish transitions from the source and the target environment. In this agent, an important hyper-parameter is the standard deviation $\sigma_{\text{DARC}}$ of the centered Gaussian noise injected into the training data to stabilize the classifiers \cite[Figure~7]{eysenbach2020off}. We draw inspiration from the open-source code \cite{pytorchh2odarc}.

    \item \textbf{Action Noise Envelope (ANE)} \cite{jakobi1995noise} is a robust algorithm that adds a centered Gaussian noise with standard deviation $\sigma_{\text{ANE}}$ to the agent's actions during training. Although simple, this method outperformed other robustness approaches in recent benchmarks \cite{desai2020imitation} when the source environment is a black box.

    \item \textbf{CQL} \cite{kumar2020conservative} is a purely offline RL algorithm that learns a policy using target data. It does not leverage the source system in its learning process and thereby serves as a lower bound to beat. This algorithm inserts a regularization into the $Q$-value functions, with a strength $\beta$. We use \cite{cqlcode} to run the experiments.

    \item \textbf{H2O} \cite{niu2022trust} is another off-dynamics algorithm that leverages data from the target system. It combines the classifiers from DARC to the CQL regularization. Similarly to FOOD, the agent is encentisized to stay close to the target data, but does so with a combination of DARC and CQL regularization. We also use \cite{pytorchh2odarc} to run these experiments. However, this method performed poorly, which was expected considering it was proposed for a setting where a large amount of target data is available. Hence, its results are omitted from Table~\ref{table:global_results} and deferred in Appendix~\ref{app:h2o_results}.    

    \item We also consider two RL agents, \textbf{$\text{RL}_\text{s}$} trained solely on the source system (without access to any target data) and \textbf{$\text{RL}_\text{t}$} trained solely on the target environment. Both algorithms were trained to convergence. Even though the latter baselines do not fit in the off-dynamics setting they give a rough idea of how online RL algorithms would perform in the target environment. The online RL algorithm $\mathcal{O}$ depends on the environment: we use A2C \cite{babaeizadehreinforcement} for Gravity Pendulum and PPO \cite{schulman2017proximal} for the other environments.
    
\end{itemize}

\begin{table*}[ht!]
\centering 
\resizebox{1.999\columnwidth}{!}{\begin{tabular}{|p{3.1cm}||P{1.9cm}|P{1.9cm}||c|c|c|c|c|c|}
  \hline
  \multirow{2}{*}{\makecell{\parbox{3.1cm}{\centering Environment}}}
  & \multirow{2}{*}{$\text{RL}_\text{s}$} 
  & \multirow{2}{*}{$\text{RL}_\text{t}$} 
  & \multirow{2}{*}{CQL} 
  & \multirow{2}{*}{ANE} 
  & \multirow{2}{*}{DARC}
  & \multicolumn{3}{c|}{FOOD (Ours)} \T\B \\ 
  \cline{7-9}
  & & & & & & $\mu^\pi_P$ & $d^\pi_P$ & $\nu^\pi_P$ \T\B \\ 
  \hline
  Gravity Pendulum
  & $-1964 \pm 186$ 
  & $-406 \pm 22$
  & $-1683 \pm 142$
  & $-2312 \pm 11$
  & $-3511 \pm 865$
  & $-2224 \pm 43$
  & $\bm{-485 \pm 54^*}$
  & $-2327 \pm 14$
  \T\B \\

  Broken Joint Cheetah
  & $1793 \pm 1125$ 
  & $5844 \pm 319$
  & $143 \pm 104$
  & $3341 \pm 132$
  & $2501 \pm 211$
  & $3801 \pm 155$
  & $\bm{3888 \pm 201}$
  & $\bm{3921 \pm 85^*}$
  \T\B \\

  Broken Joint Minitaur
  & $7.4 \pm 4.1 $
  & $20.8 \pm 4.8$
  & $0.25 \pm 0.09$
  & $7.8 \pm 6$
  & $12.6 \pm 1.12$
  & $\bm{14.9 \pm 3}$
  & $\bm{13.6 \pm 3.8}$
  & $\bm{16.9 \pm 4.7^*}$
  \T\B \\

  Heavy Cheetah
  & $3797 \pm 703 $
  & $11233 \pm 1274$
  & $41 \pm 34$
  & $\bm{7443 \pm 330^*}$
  & $4165 \pm 257$
  & $4876 \pm 181$
  & $4828 \pm 553$
  & $4519 \pm 240$
  \T\B \\

  Broken Joint Ant
  & $5519 \pm 876 $
  & $6535 \pm 352$
  & $1042 \pm 177$
  & $3231 \pm 748$
  & $5446 \pm 162$
  & $\bm{6145 \pm 98^*}$
  & $5547 \pm 204$
  & $\bm{6135 \pm 122}$
  \T\B \\

  Friction Cheetah
  & $1427 \pm 674 $
  & $9455 \pm 3554$
  & $-466.4 \pm 13$
  & $\bm{6277 \pm 1405^*}$
  & $3302 \pm 591$
  & $3690 \pm 1495$
  & $3212 \pm 2279$
  & $3289 \pm 236$
  \T\B \\

  Low Fidelity Minitaur
  & $8.9 \pm 5.8 $
  & $27.1 \pm 8$
  & $10.2 \pm 1$
  & $6.4 \pm 3$
  & $3.2 \pm 1.8$
  & $\bm{17 \pm 2}$
  & $15.7 \pm 2.8$
  & $\bm{17.6 \pm 0.4^*}$
  \T\B \\

  Broken Leg Ant
  & $1901 \pm 981 $
  & $6430 \pm 451$
  & $830 \pm 8$
  & $2611 \pm 220$
  & $2305 \pm 175$
  & $2652 \pm 356$
  & $2345 \pm 806$
  & $\bm{2977 \pm 85^*}$
  \T\B \\
  \cline{1-9}

  Median NAR and std \T\B
  & $ 0 \ ; \ 0.25$
  & $ 1 \ ; \ 0.26$
  & $ -0.32 \ ; \ 0.02 $
  & $ 0.1 \ ; \ 0.11 $
  & $ 0.04 \ ; \ 0.04 $
  & $ \bm{0.37 \ ; \ 0.09^*} $
  & $ \bm{0.29 \ ; \ 0.17} $
  & $ \bm{0.36 \ ; \ 0.03} $
  \T\B \\ 
  \cline{1-9}

\end{tabular} }
\caption{Returns over $4$ seeds for CQL, $\text{RL}_{\text{s}}$, $\text{RL}_{\text{t}}$ and ANE, and $8$ seeds for the other agents on the tested environments. The best agent w.r.t. the mean is highlighted with boldface and an asterisk. We perform an unpaired t-test with an asymptotic significance of 0.1 w.r.t. the best performer and highlight with boldface the ones for which the difference is not statistically significant.}
\label{table:global_results}
\end{table*}

\paragraph{Experimental Protocol.} Our proposed model and corresponding off-dynamics/offline baselines require a batch of target data. To provide such a batch of data, we first train a policy and a value function of the considered RL agent until convergence by maximizing the reward on the source environment. After this initialization phase, $5$ trajectories are sampled from the target environment to fit the restricted target data regime. They correspond to $500$ data points for Pendulum and $1000$ data points for the other environments. If some trajectories perform poorly in the target environment, we remove them for FOOD, DARC, and CQL to avoid having a misguided regularization. FOOD, DARC, and ANE are trained for $5000$ epochs in the source environment and were optimized with the same underlying agent. Both $\text{RL}_\text{s}$ and $\text{RL}_\text{t}$ are trained until convergence. CQL is trained for $100 000$ gradient updates for Gravity Pendulum and $500000$  gradient updates for all other environments. Additional details can be found in Appendix~\ref{app:exp_details}.

\paragraph{Hyper-parameters Optimization.} We optimize the hyper-parameters of the evaluated algorithms through a grid search for each different environment. Concerning DARC and ANE, we perform a grid search over their main hyper-parameter $\sigma_{\text{DARC}}\in\{0.0, 0.1, 0.5, 1\}$ and $\sigma_{\text{ANE}}\in\{0.1, 0.2, 0.3, 0.5\}$. The remaining hyper-parameters were set to their default values according to the experiments reported in the open-source code \cite{pytorchh2odarc}. For CQL, we perform a grid search over the regularization strength $\beta \in\{5, 10\}$, otherwise we keep the original hyper-parameters of \cite{cqlcode}. For $\text{RL}_\text{t}$ and $\text{RL}_\text{s}$ we used the default parameters specific to each environment according to \cite{pytorchrl} and trained them over $4$ different seeds. We then selected the seed with the best performance in the source environment. For our proposed algorithm FOOD, the regularization strength hyper-parameter $\alpha$ is selected over a grid search depending on the underlying RL agent, $\alpha\in\{ 0, 1, 5, 10 \}$ for A2C and  $\alpha\in\{ 0.5, 1, 2, 5 \}$ for PPO. This difference in choice is explained by the fact that the advantages are normalized in PPO, giving a more refined control over the regularization weight.

\paragraph{Results.} We monitor the evolution of the agents' performance by evaluating their average return in the target environment during training. We do not gather nor use the data from those evaluations in the learning process since we work under a few-shot framework. The return of all methods is computed and averaged over $4$ seeds for CQL, $\text{RL}{\text{s}}$, $\text{RL}{\text{t}}$ and ANE, and $8$ seeds for FOOD and  DARC. For clarity, the standard deviation in Figure~\ref{fig:comparison_food_darc_hyperparams} is divided by $2$ for readable purposes. In all figures, the $x$-axis represents the number of epochs where each epoch updates the policy and value functions with $8$ different trajectories from the source environment.

\begin{figure*}[htp]
\centering
\includegraphics[width=.999\textwidth, trim=0cm 0.8cm 0cm 0.6cm, clip]{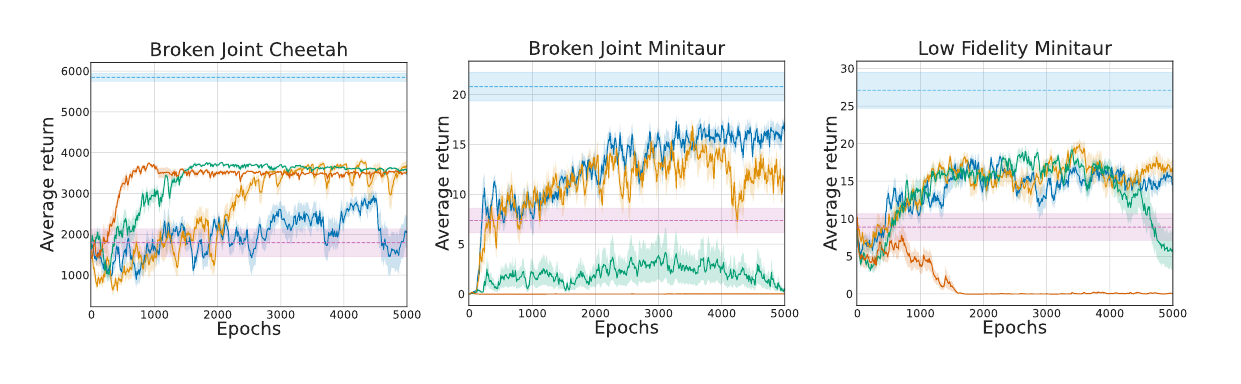}
\includegraphics[width=.8\textwidth]{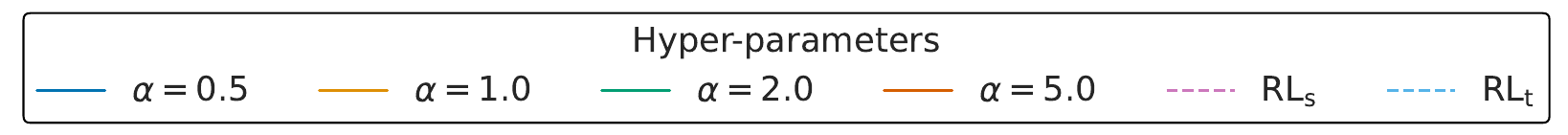}
\caption{Hyper-parameter sensibility analysis for FOOD on three environments.}
\label{fig:comparison_food_darc_hyperparams}
\end{figure*}


\subsection{Comparison Between the Different Agents}

We evaluate the mentioned algorithms on the proposed environments. These experiments provide an overview of the efficiency of the different objectives in finetuning the policy, given reasonably good trajectories. Results are summarized in Table~\ref{table:global_results}, where we also report the median of the normalized average return (NAR) $\frac{J^{\pi_{\text{agent}}}_{P_{\text{t}}} - J^{\pi_{\text{RL}_{\text{s}}}}_{P_{\text{t}}}}{J^{\pi_{\text{RL}_{\text{t}}}}_{P_{\text{t}}} - J^{\pi_{\text{RL}_{\text{s}}}}_{P_{\text{t}}}}$ \cite{desai2020imitation} as well as the median of the NAR's standard deviations. The associated learning curves can be found in Appendix~\ref{app:learning_curves}.

All the experiments demonstrate the insufficiency of training traditional RL agents solely on the source environment. The optimal policy for the source is far from optimal for the target system as we observe a large drop in performance from $\text{RL}_\text{t}$ to $\text{RL}_\text{s}$ on all benchmarked environments. For example, the $\text{RL}_\text{s}$ exploits the linear torque-current relation in Low Fidelity Minitaur and fails to learn a near-optimal policy for the target environment. Furthermore, $\text{RL}_\text{s}$ often exhibits a large variance in target environments as it encounters previously unseen situations. This is welcome as relevant trajectories are gathered to guide the agent in the source environment.

Overall, we can see that our algorithm FOOD exhibits the best performances across all considered environments against all other baselines, whether it is constrained by state, state-action or transition visitation distributions. Two exceptions are on Heavy and Friction Cheetah where ANE has very good results. We also note that FOOD associated with its regularization $\nu_P^\pi$ or $\mu_P^\pi$ has better results than when it is associated with $d_P^\pi$. This is expected for $\nu_P^\pi$ as it optimizes the whole lower bound, and it seems using $\mu_P^\pi$ implicitely mimizes the second term of Proposition~\ref{prop:trust_region_1}.

In addition, we find that the prominent baseline DARC is not efficient in all the use cases. It seems to be particularly good at handling sharp dynamics discrepancies, e.g., when one or two joints are broken or when friction is introduced but struggles for more subtle differences. In fact, it deteriorates over the naive baseline $\text{RL}_\text{s}$ by a large margin for the Gravity Pendulum and the Low Fidelity Minitaur environments. This may be explained by their reward modification $\Delta_r$ (see Appendix~\ref{app:darc_details}) which prevents the agent from entering dissimilar parts of the source environment but seems unable to handle systems with a slight global dynamics mismatch. Even when DARC improves over $\text{RL}_\text{s}$, our algorithm FOOD is able to match or exceed its performance. The robust agent ANE is a strong baseline in most environments but may degrade the performance of traditional RL agents, as seen in Low Fidelity Minitaur, Broken Joint Ant, and Gravity Pendulum. CQL did not provide any good results, except on Gravity Pendulum and Low Fidelity Minitaur, but this was to be expected given the few target trajectories the agent has access to. Finally, we note that the three agents FOOD, DARC, and ANE often reduce the variance originally presented in $\text{RL}_\text{s}$.

We attribute FOOD's success to its ability to force the agent to improve the rewards of the source environment along trajectories from the target environment. Its regularization seems to be efficient even in the low data regime.

\subsection{Hyper-parameter Sensitivity Analysis}

We previously reported the results of the best hyper-parameters of the different methods. In practice, it is important to have a robust range of hyper-parameters in which the considered method performs well. Indeed, to the best of our knowledge, there currently exists no accurate algorithm for selecting such hyper-parameters in a high dimensional environment when the agent has access to limited data gathered with a different policy \cite{fu2020benchmarks}. In this section, we detail the sensitivity of FOOD associated with PPO and GAIL to its hyper-parameter $\alpha$ in $3$ environments. They were specifically chosen to illustrate the relevant range for $\alpha$. FOOD's complete hyper-parameter sensitivity analysis, as well as the one of DARC and ANE, can respectively be found in Appendix~\ref{app:food_hyperparams_analysis}, Appendix~\ref{app:darc_hyperparams_analysis} and Appendix~\ref{app:ane_hyperparams_analysis}. They attest that selecting the right hyper-parameters for the baselines DARC and ANE does not have an intuitive pattern.

The hyper-parameter $\alpha$ controls the strength of the regularization in FOOD. If it is too low, the agent will focus mainly on maximizing the rewards of the source environment and becomes very close to the naive baseline $\text{RL}_\text{s}$. This can be seen in Broken Joint Cheetah. On the other hand, setting $\alpha$ to a high value may induce the agent to solely replicate the trajectories from the target environment, in which case the agent may also be close to $\text{RL}_\text{s}$. Even worse, a hard regularization may degrade over $\text{RL}_\text{s}$, as shown in Low Fidelity Minitaur for $\alpha = 5$. However, $5$ is an extremely high value as advantages are normalized in PPO, and this may increase the gradient too much and corrupt learning.

In any case, we have found that FOOD provides the best results when the regularization has the same scale as the traditional objective. This is also verified for the environments not displayed in this sub-section. We conclude that FOOD is relatively robust to this range of hyper-parameter, and recommend using PPO with $\alpha$ close to $1$ as the underlying RL agent. It is a natural choice given that PPO normalizes its advantage functions.

\section{Conclusion}

In this work, we investigated different objectives to optimize a policy in different few-shot off-dynamics scenarios, including the state-of-the-art method DARC. We found that these objectives are either too simplistic or unable to cope with complex dynamics discrepancies, thereby limiting their application to real-world systems. To address this challenge, based on theoretical insights, we introduced a novel conservative objective along with a practical algorithm leveraging imitation learning techniques. Through experimentations in different off-dynamics use cases, we have shown that our approach often outperforms the existing methods and seems to be more robust to dynamics changes. 

Our agent could also benefit from new advances in the Imitation Learning literature to gain control in building its penalization. Finally, this penalization can be useful when the source environment has been improved using the available target trajectories as it avoids querying the source environment for Out-of-Distribution samples. This will be the primary focus of our future work.



\section*{Contribution Statement}

Paul Daoudi wrote the main parts of the article and conducted the experiments. Merwan Barlier and Ludovic Dos Santos participated in the development of the main ideas and extensively revised the article. Christophe Prieur and Bogdan Robu oversaw the project and provided additional revisions to the article.

\bibliographystyle{named}
\bibliography{ijcai24}

\clearpage
\onecolumn
\appendix
\section{Proofs} \label{app:proofs}

\subsection{Proof of Proposition 1}

We start by recalling an important Lemma of \cite{achiam2017constrained}.

\begin{lemma} \label{lem:achiam}
    For any function $f : \mathcal{S} \to \mathbb{R}$, policy $\pi$ and $\delta_f(s, a, s') = r(s,a, s') + \gamma f(s') - f(s)$: 

    \begin{equation}
        J^\pi_{P} = \mathbb{E}_{s\sim\rho_0} \left[ f(s) \right] + \frac{1}{1-\gamma} \mathbb{E}_{\substack{s\sim d^\pi_P(\cdot), \\ a\sim\pi(\cdot|s), \\ s'\sim P(\cdot|s, a)}} \left[ \delta_f(s, a, s') \right]. 
    \end{equation}
\end{lemma}

Then, we propose this general Lemma that serves as a basis for our Proposition~\ref{prop:trust_region_1}.

\begin{lemma} \label{lem:general_lemma}
    For any function $f : \mathcal{S} \to \mathbb{R}$, let:

    \begin{align}
        L_f^{\pi, P_{\text{t}}, P_{\text{s}}} 
        &= \mathbb{E}_{\substack{s\sim d^\pi_{P_{\text{s}}}(\cdot), \\ a\sim\pi(\cdot|s)}} \left[
        \mathbb{E}_{s'\sim P_{\text{t}}(\cdot|s, a)} \left[ \delta_f(s, a, s') \right]
         - \mathbb{E}_{s'\sim P_{\text{t}}(\cdot|s, a)} \left[ \delta_f(s, a, s') \right]
        \right] \\
        \epsilon_f^{P_{\text{t}}} &= \max_{s\in\mathcal{S}}\left\lvert 
        \mathbb{E}_{a\sim\pi, s'\sim P_{\text{t}}} \left[ \delta_f(s, a, s') \right]
        \right\rvert. 
    \end{align}

    The following bound holds:

    \begin{equation}
    J^\pi_{P_{\text{t}}} \geq J^\pi_{P_{\text{s}}} + \frac{1}{1-\gamma} \left(
    L_f^{\pi, P_{\text{t}}, P_{\text{s}}} - 2 \epsilon_f^{P_{\text{t}}} \tvinfdiv*{d^\pi_{P_{\text{s}}}}{d^\pi_{P_{\text{t}}}}
    \right).
    \end{equation}
\end{lemma}

\begin{proof}

    According to Lemma~\ref{lem:achiam}:

    \begin{align}
        J^\pi_{P_{\text{t}}} - J^\pi_{P_{\text{s}}} = \frac{1}{1-\gamma} \left( 
        \mathbb{E}_{\substack{s\sim d^\pi_{P_{\text{t}}}(\cdot), \\ a\sim\pi(\cdot|s) \\ s'\sim P_{\text{t}}(\cdot|s, a)}} \left[ \delta_f(s, a, s') \right] - \mathbb{E}_{\substack{s\sim d^\pi_{P_{\text{s}}}(\cdot), \\ a\sim\pi(\cdot|s), \\ s'\sim P_{\text{s}}(\cdot|s, a)}} \left[ \delta_f(s, a, s') \right]
        \right).
    \end{align}

    The first term can be written, with $\bar{\delta}_f^{P_{\text{t}}}(s) = \mathbb{E}_{\substack{ a\sim\pi(\cdot|s)\\ s'\sim P_{\text{t}}(\cdot|s, a)}} \left[ \delta_f(s, a, s')\right]$:

    \begin{align}
        \mathbb{E}_{\substack{s\sim d^\pi_{P_{\text{t}}}(\cdot)\\ a\sim\pi(\cdot|s)\\ s'\sim P_{\text{t}}(\cdot|s, a)}} \left[ \delta_f(s, a, s') \right] 
        &= \langle d^\pi_{P_{\text{t}}} \,,\, \bar{\delta}_f^{P_{\text{t}}} \rangle \\
        &= \langle d^\pi_{P_{\text{s}}} \,,\, \bar{\delta}_f^{P_{\text{t}}} \rangle + \langle d^\pi_{P_{\text{t}}} - d^\pi_{P_{\text{s}}} \,,\, \bar{\delta}_f^{P_{\text{t}}} \rangle.
    \end{align}

    We apply Holder's inequality with $p=1$ and $q=\infty$, and get:

    \begin{align}
        \mathbb{E}_{\substack{s\sim d^\pi_{P_{\text{t}}}(\cdot)\\ a\sim\pi(\cdot|s)\\ s'\sim P_{\text{t}}(\cdot|s, a)}} \left[ \delta_f(s, a, s') \right] \geq 
        \langle d^\pi_{P_{\text{s}}} \, \bar{\delta}_f^{P_{\text{t}}} \rangle - 2 \epsilon_f^{P_{\text{t}}} D_{\text{TV}}\left( d^\pi_{P_{\text{s}}}, d^\pi_{P_{\text{t}}} \right),
    \end{align}
    with $\epsilon_f^{P_{\text{t}}} = \max_{s\in\mathcal{S}}\left\lvert 
        \mathbb{E}_{a\sim\pi, s'\sim P_{\text{t}}} \left[ \delta_f(s, a, s') \right]
        \right\rvert$. The Total Variation distance comes from the 1-norm resulting from the application of Holder's inequality. We obtain:

    \begin{equation}
        \begin{split}
        (1 - \gamma) \left( J^\pi_{P_{\text{t}}} - J^\pi_{P_{\text{s}}} \right) &\geq 
        \mathbb{E}_{\substack{s\sim d^\pi_{P_{\text{s}}}(\cdot), \\ a\sim\pi(\cdot|s)}} \left[
        \mathbb{E}_{s'\sim P_{\text{t}}(\cdot|s, a)} \left[ \delta_f(s, a, s') \right]
         - \mathbb{E}_{s'\sim P_{\text{t}}(\cdot|s, a)} \left[ \delta_f(s, a, s') \right]
        \right] \\ & \quad \quad - 2 \epsilon_f^{P_{\text{t}}} \tvinfdiv*{d^\pi_{P_{\text{s}}}}{d^\pi_{P_{\text{t}}}}. 
        \end{split}
    \end{equation}

\end{proof}

To conclude the proof of Proposition~\ref{prop:trust_region_1}, we choose $f$ as the null function $f : \mathcal{S} \to 0$ and upper bound the remaining term by reusing Holder's inequality:

\begin{align}
    \mathbb{E}_{\substack{s\sim d^\pi_{P_{\text{s}}}(\cdot), \\ a\sim\pi(\cdot|s), \\ s'\sim P_{\text{t}}(\cdot|s, a)}} \left[ r(s, a, s') \right] - \mathbb{E}_{\substack{s\sim d^\pi_{P_{\text{s}}}(\cdot), \\ a\sim\pi(\cdot|s), \\ s'\sim P_{\text{t}}(\cdot|s, a)}} \left[ r(s, a, s') \right]
        &\geq - 2 R_{\text{max}} \mathbb{E}_{\substack{s\sim d^\pi_{P_{\text{s}}}(\cdot), \\ a\sim\pi(\cdot|s)}} \left[ 
        \tvinfdiv*{P_{\text{s}}\left( \cdot \vert s, a\right)}{P_{\text{t}}\left( \cdot \vert s, a\right)}
        \right] \\
        &\geq - 2 R_{\text{max}} \tvpiinfdiv{P_{\text{s}}}{ P_{\text{t}}}.
\end{align}

\paragraph{Other choice for $f$} The function $f$ could also be chosen as the value function associated with the source system $V^{\pi}_{P_{\text{s}}}$. In which case, we get with Lemma~\ref{lem:general_lemma}:

\begin{equation}
\begin{split}
    J^\pi_{P_{\text{t}}} \geq J^\pi_{P_{\text{s}}} + \frac{1}{1-\gamma} \left( \mathbb{E}_{\substack{s\sim d^\pi_{P_{\text{s}}}(\cdot), \\ a\sim\pi(\cdot|s), \\ s'\sim P_{\text{s}}(\cdot|s, a)}} \left[
        \frac{P_{\text{t}}(s'|s, a)}{P_{\text{s}}(s'|s, a)} \left( r(s, a, s') + \gamma V^{\pi}_{P_{\text{s}}}(s') - V^{\pi}_{P_{\text{s}}}(s)  \right) \right] \right. \\
    \mspace{-50mu} \left. - 2 \epsilon_f^{P_{\text{t}}} \tvinfdiv*{d^\pi_{P_{\text{s}}}} {d^\pi_{P_{\text{t}}}}
    \right).
\end{split}
\end{equation}

It also introduces an additional term than Proposition~\ref{prop:trust_region}. Here, it is an importance sampling term between the transition probabilities that is difficulty optimized. In principle, it could be estimated with the classifiers proposed by DARC but would introduce a new level of complexity to the algorithm. Hence, we preferred focusing on proposing the simpler Proposition~\ref{prop:trust_region}.

\subsection{Proof of Proposition 2}

We present here the proof of our simpler Proposition~\ref{prop:trust_region} that we restate below, as well as its extensions using different discrepancy measures.

\begin{proposition}
\label{prop:trust_region_appendix}
    Let $\nu^\pi_P(s, a, s')$ the state-action-state visitation distribution, where $\nu^\pi_P(s, a, s') = (1 - \gamma) \mathbb{E}_{\rho_0, \pi, P}\left[  \sum_{t=0}^\infty \gamma^t \mathbb{P}\left( s_t=s, a_t=a, s_{t+1}=s' \right) \right]$. For any policy $\pi$ and any transition probabilities $P_{\text{t}}$ and $P_{\text{s}}$, the following holds:
    
    \begin{equation}
    J^\pi_{P_{\text{t}}} \geq J^\pi_{P_{\text{s}}} - \frac{2 R_{\text{max}}}{1-\gamma}  \tvinfdiv*{\nu^\pi_{P_{\text{s}}}}{\nu^\pi_{P_{\text{t}}}},
\end{equation}
with $D_{\text{TV}}$ the Total Variation distance. 
\end{proposition}

\begin{proof}
    
    It is known that $J^\pi_{P} = \frac{1}{1-\gamma} \mathbb{E}_{(s,a,s')\sim\nu^\pi_P} \left[ r(s, a, s') \right]$. Now:

    \begin{align}
        \left\lvert J^\pi_{P_{\text{t}}} - J^\pi_{P_{\text{s}}} \right\rvert
        &= \frac{1}{1-\gamma} \left\lvert \left( \mathbb{E}_{(s,a,s')\sim\nu^\pi_{P_{\text{t}}}} \left[ r(s, a, s') \right] - \mathbb{E}_{(s,a,s')\sim\nu^\pi_{P_{\text{s}}}} \left[ r(s, a, s') \right] \right) \right\rvert \\
        &= \frac{1}{1-\gamma} \left\lvert \int_{s, a, s'} \left( r(s, a, s') \nu^\pi_{P_{\text{t}}}(s, a, s') - r(s, a, s') \nu^\pi_{P_{\text{s}}}(s, a, s') \right) \mathrm{d}\left\{sas'\right\}  \right\rvert \\
        &= \frac{1}{1-\gamma} \left\lvert \int_{s, a, s'} r(s, a, s') \left( \nu^\pi_{P_{\text{t}}}(s, a, s') - \nu^\pi_{P_{\text{s}}}(s, a, s') \right) \mathrm{d}\left\{sas'\right\} \right\rvert \\
        &\leq \frac{2 R_{\text{max}}}{1-\gamma} \tvinfdiv*{\nu^\pi_{P_{\text{s}}}}{\nu^\pi_{P_{\text{t}}}}.
    \end{align}

    The last inequality is an application of Holder's inequality, by setting $p$ to $\infty$ and $q$ to $1$.
    
\end{proof}

An application of Pinsker inequality \cite{csiszar1981coding} provides a similar upper bound with the ullback Leibleir divergence.

\begin{corollary}
\label{thm:trust_region_appendix_kl}
    Let $\nu^\pi_P(s, a, s')$ the state-action-state visitation distribution, where $\nu^\pi_P(s, a, s') = (1 - \gamma) \mathbb{E}_{\rho_0, \pi, P}\left[  \sum_{t=0}^\infty \gamma^t \mathbb{P}\left( s_t=s, a_t=a, s_{t+1}=s' \right) \right]$. For any policy $\pi$ and any transition probabilities $P_{\text{t}}$ and $P_{\text{s}}$ such that $\nu^\pi_{P_{\text{s}}}$ is absolutely continuous with respect to $\nu^\pi_{P_{\text{t}}}$, the following holds:
    
    \begin{equation}
    J^\pi_{P_{\text{t}}} \geq J^\pi_{P_{\text{s}}} - \frac{\sqrt{2} R_{\text{max}}}{1-\gamma}  \sqrt{\klinfdiv*{\nu^\pi_{P_{\text{s}}}}{\nu^\pi_{P_{\text{t}}}}},
    \end{equation}
with $D_{\text{KL}}$ the Kullback Leibleir divergence. 
    
\end{corollary}

A lower bound with the Jensen Shannon divergence can also be found thanks to \cite[Proposition 3.2]{corander2021jensen}.

\begin{corollary}
\label{thm:trust_region_appendix_js}
    We assume the state-action space. Let $\nu^\pi_P(s, a, s')$ the state-action-state visitation distribution, where $\nu^\pi_P(s, a, s') = (1 - \gamma) \mathbb{E}_{\rho_0, \pi, P}\left[  \sum_{t=0}^\infty \gamma^t \mathbb{P}\left( s_t=s, a_t=a, s_{t+1}=s' \right) \right]$. We assume the support of $\nu^\pi_{P_{\text{s}}}$ and $\nu^\pi_{P_{\text{s}}}$ is $\mathcal{S}\times\mathcal{A}\times\mathcal{S}$. Then, for any policy $\pi$ and any transition probabilities $P_{\text{t}}$ and $P_{\text{s}}$, the following holds:

    \begin{equation}
    J^\pi_{P_{\text{t}}} \geq J^\pi_{P_{\text{s}}} - \frac{4 R_{\text{max}}}{ \left(1-\gamma\right)}  \sqrt{\jsinfdiv*{\nu^\pi_{P_{\text{s}}}}{\nu^\pi_{P_{\text{t}}}}},
    \end{equation}
with $D_{\text{JS}}$ the Jensen Shannon divergence. 

\end{corollary}

\section{Algorithms Details}

In this section, we further present the different algorithms used in this paper.

\subsection{Domain Adaptation with Rewards from Classifiers (DARC)} \label{app:darc_details}

We introduce our main baseline Domain Adaptation with Rewards from Classifiers (DARC), which is the prominent state-of-the-art algorithm that tackles the off-dynamics task by modifying the RL objective.

DARC takes a variational perspective to this problem. Given a trajectory $\tau = \left( s_0, a_0, s_1, a_1, \dots \right)$, the target distribution $p(\tau)$ over trajectories is defined as the one inducing trajectories that maximize the exponentiated rewards in the target environment:

\begin{equation}
    p(\tau) = \rho(s_0) \left( \prod_t P_{\text{t}}(s_{t+1}\vert s_t, a_t) \right) \exp\left( \sum_t r(s_t, a_t, s_{t+1}) \right).
\end{equation}

Let the agent's distributions over trajectories in the source environment $q^{\pi_\theta}(\tau)$ be:

\begin{equation}
    q^{\pi_\theta}(\tau) = \rho(s_0) \left( \prod_t P_{\text{s}}(s_{t+1}\vert s_t, a_t) \right) \pi_\theta (a_t\vert s_t).
\end{equation}

DARC minimizes the reversed KL-divergence between $q^{\pi_\theta}(\tau)$ and $p(\tau)$, which results in the following objective expression:

\begin{gather}
    - \klinfdiv*{q^{\pi_\theta}(\tau)}{p(\tau)} = \mathbb{E}_{\tau\sim q^{\pi_\theta}(\cdot)} \left[ \sum_{t=1}^T r(s_t, a_t, s_{t+1}) + \mathcal{H}\left( \pi_\theta(\cdot|s_t) \right) +  \Delta r(s_t, a_t, s_{t+1}) \right],
\end{gather}
with $\Delta r(s_t, a_t, s_{t+1}) = \log P_{\text{t}}(s_{t+1}\vert s_t, a_t) - \log P_{\text{s}}(s_{t+1}\vert s_t, a_t)$ and $\mathcal{H(\cdot)}$ the entropy.

The additional reward term incentivizes the agent to select transitions from the source that are similar to the target environment. Since the transition probabilities are unknown, DARC uses a pair of binary classifiers to infer whether transitions come from the source or target environment. These classifiers are then used to create a proxy equivalent to $\Delta r$.


\subsection{Generative Adversarial Imitation Learning applied for transition distributions}

Generative Adversarial Imitation Learning (GAIL) \cite{ho2016generative} is a state-of-the-art Imitation Learning algorithm. Its goal is to recover an expert policy $\pi_e$ by minimizing the Jensen-Shanon divergence between the state-action visitation distributions of the expert and the learning policy. It has been proved that it is able to handle transition visitation distributions in \cite{desai2020imitation} as follows. To comply with our previous notations, $\pi_e$ is now denoted as $\pi_{\theta_k}$ (fixed).

The authors define the general objective to solve by introducing a convex cost function regularizer $\psi : \mathbb{R}^{S\times A\times S} \to \mathbb{R}$ and its convex conjugate $\psi^*$:
\begin{mini}{\theta\in\Theta}{ \psi^*(\nu_{P_{\text{s}}}^{\pi_\theta} - \nu_{P_{\text{t}}}^{\pi_{\theta_k}}).}{}{}
\label{eq:gail_general_objective_app}
\end{mini}

Following Equation~13 of \cite{ho2016generative} which defines $\psi_{\text{GAIL}}$, the authors establish the following equivalence:
\begin{equation}
    \psi^*_{\text{GAIL}}(\nu_{P_{\text{s}}}^{\pi_\theta} - \nu_{P_{\text{t}}}^{\pi_{\theta_k}}) = \sup_{D\in\left( 0, 1 \right)^{S\times A\times S}} \mathbb{E}_{(s,a,s')\sim\nu_{P_{\text{s}}}^{\pi_\theta}} \left[ \log \left( D(s,a, s') \right) \right] + \mathbb{E}_{(s,a,s')\sim\nu_{P_{\text{t}}}^{\pi_{\theta_k}}}\left[ \log \left( 1 - D(s,a,s') \right) \right]
\end{equation}

where $D : \mathcal{S}\times\mathcal{A}\times\mathcal{S} \to \left( 0, 1 \right)$ is a classifier. Finally, it is demonstrated this specific convex cost function induces the following objective: 
\begin{equation}
\min_{\theta\in\Theta}{ \psi^*_{\text{GAIL}}(\nu_{P_{\text{s}}}^{\pi_\theta} - \nu_{P_{\text{t}}}^{\pi_{\theta_k}}) = \min_{\theta\in\Theta} \jsinfdiv{\nu_{P_{\text{s}}}^{\pi_\theta}}{\nu_{P_{\text{t}}}^{\pi_{\theta_k}}} .}
\label{eq:gail_general_objective}
\end{equation}

In practice, the classifier $D$ is trained to distinguish between samples $(s,a,s')\in(\mathcal{S}\times\mathcal{A}\times\mathcal{S})$ from $\nu_{P_\text{s}}^{\pi_\theta}$ and $\nu_{P_\text{t}}^{\pi_{\theta_k}} $. The reward used for optimizing the RL agent is given by $r_{\text{imit}} = -\log \left( D(s, a, s') \right)$.

\subsection{Conservative Q-Learning (CQL)}

In the offline setting, agents aim to learn a good policy from a fixed data set of $M$ transitions $\mathcal{D} = \{ (s_i, a_i, s_{i+1} \}_{i=0}^M$ that was collected with an unknown behavioral policy $\pi_\beta$, which is here $\pi_{\theta_k}$. Offline RL algorithms have demonstrated impressive results when the data set is gathered with a sufficiently good policy and possesses enough transitions, often outperforming the behavioral policy.

Conservative Q-Learning (CQL) \cite{kumar2020conservative} is a state-of-the-art offline RL algorithm. It modifies the learning procedure of the $Q$-functions to favor transitions appearing in the data set. At iteration $k$, the $Q$-values are updated as follows at step $j$:

\begin{mini}{\omega\in\Omega}{ \beta \, \mathbb{E}_{s\sim\mathcal{D}}\left[ \left(
    \log \sum_{a\in\mathcal{A}} \exp\left( Q_\omega^{\pi_{\theta_j}}(s,a) \right)
    - \mathbb{E}_{a\sim\pi_{\theta_k}(\cdot\vert s)} \left[ Q_\omega^{\pi_{\theta_j}}(s,a) \right] 
    \right) \right]
    + \mathcal{E}\left( Q_\omega^{\pi_{\theta_j}} \right),}{}{}
\label{eq:eq_cql}
\end{mini}

where $\mathcal{E}\left( Q \right)$ represents the traditional Bellman loss associated with the $Q$-functions. The regularization, controlled by the hyper-parameter $\beta$, penalizes the $Q$-values associated with state-action pairs not appearing in the data set.

\section{Experimental Details} \label{app:exp_details}

In this section, in addition to the values of the hyperparameters necessary to replicate our experiments, we provide further details of the experimental protocol and training. In this section, considering the possible high variance of $\text{RL}_{\text{s}}$, the standard deviation is multiplied by a factor of $0.3$. The original variance can be found in Table~\ref{table:global_results}.

\subsection{Environment Details} \label{app:env_details}

In all the considered environments, one property is modified in the target environment.

\paragraph{Gravity Pendulum} Gravity is increased to $14$ instead of $10$. Since the pendulum requires more time to reach the objective, we also increase the length of each episode to 500 time-steps in the target environment, while keeping the original length of 200 time-steps in the source system.

\paragraph{Broken Joint or Leg environments} In these environments, the considered robot - either HalfCheetah or Ant - is crippled in the target domain, where the effect of one or two joints is removed. In practice, this means that it sets one or two dimensions of the action to $0$. These environments were extracted from the open source code of \cite{eysenbach2020off}.

\paragraph{Heavy Cheetah} The total mass of the HalfCheetah MuJoCo robot is increased from $14$ to $20$.

\paragraph{Friction Cheetah} The friction coefficient of the HalfCheetah MuJoCo robot's feet is increased from $0.4$ to $1$.

\paragraph{Low Fidelity Minitaur} The original Minitaur environment uses a linear torque-current linear relation for the actuator model. It has been improved in \cite{tan2018sim} by introducing non-linearities into this relation where they managed to close the Sim-to-Real gap for a real Minitaur environment. In practice, the Minitaur environment can be found in the PyBullet library \cite{coumans2021}. The high fidelity is registered as MinitaurBulletEnv-v0. The low fidelity environment can be recovered by calling MinitaurBulletEnv-v0 and by setting the argument \texttt{accurate motor model enabled} to False and \texttt{pd control enabled} to True.

\subsection{Learning curves} \label{app:learning_curves}

We report in Figure~\ref{fig:learning curves} the learning curves of the different agents mentioned in this paper. For clarity purposes, we keep all baselines fixed except for our agent and DARC, our main competitor. Here, FOOD uses the regularization with $d^\pi_P$ for Gravity Pendulum and $\nu^\pi_P$ for the other environments as GAIL proved to be more stable when FOOD used PPO.

\begin{figure}[htp]
\centering

\includegraphics[width=.325\textwidth]{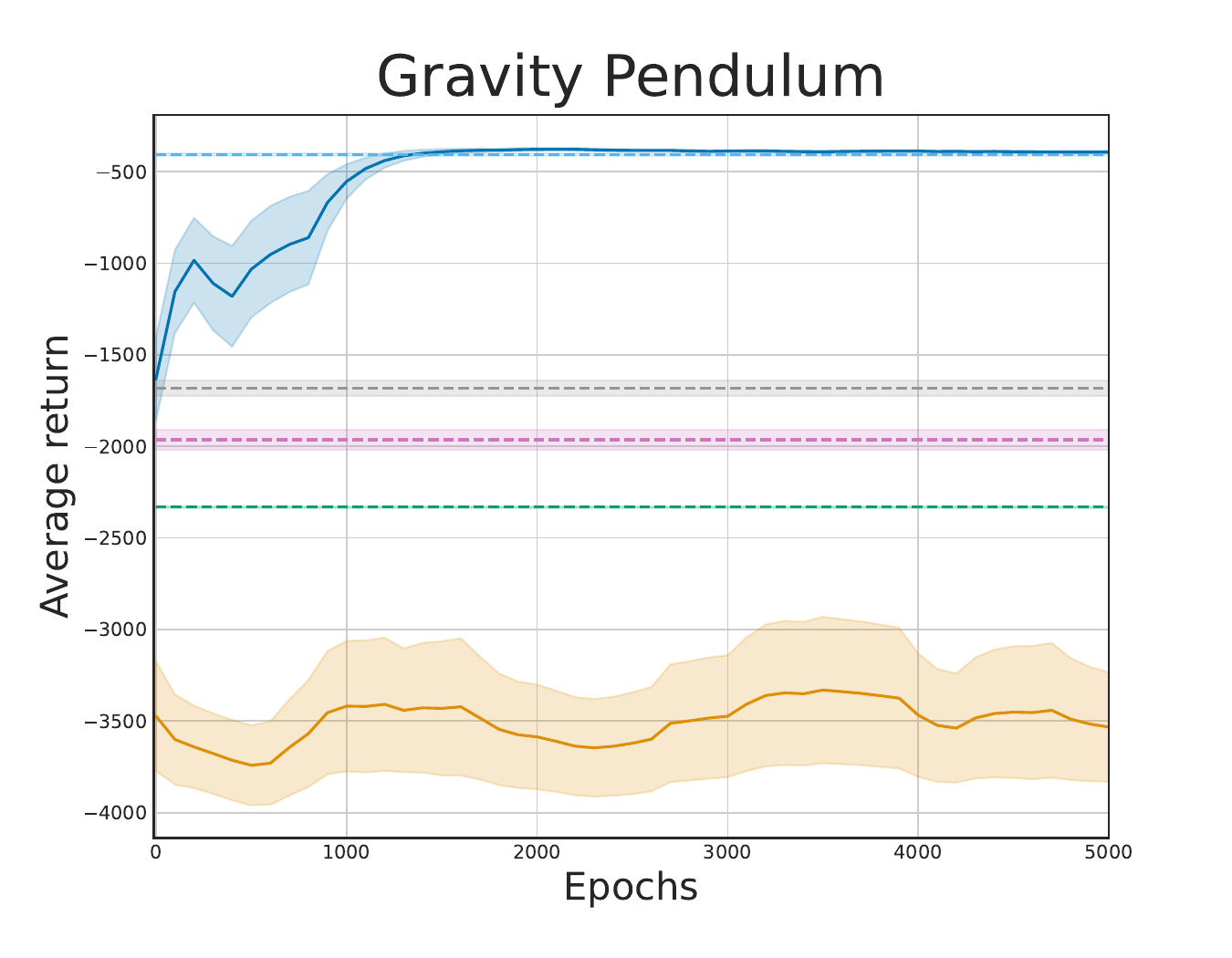}
\includegraphics[width=.325\textwidth]{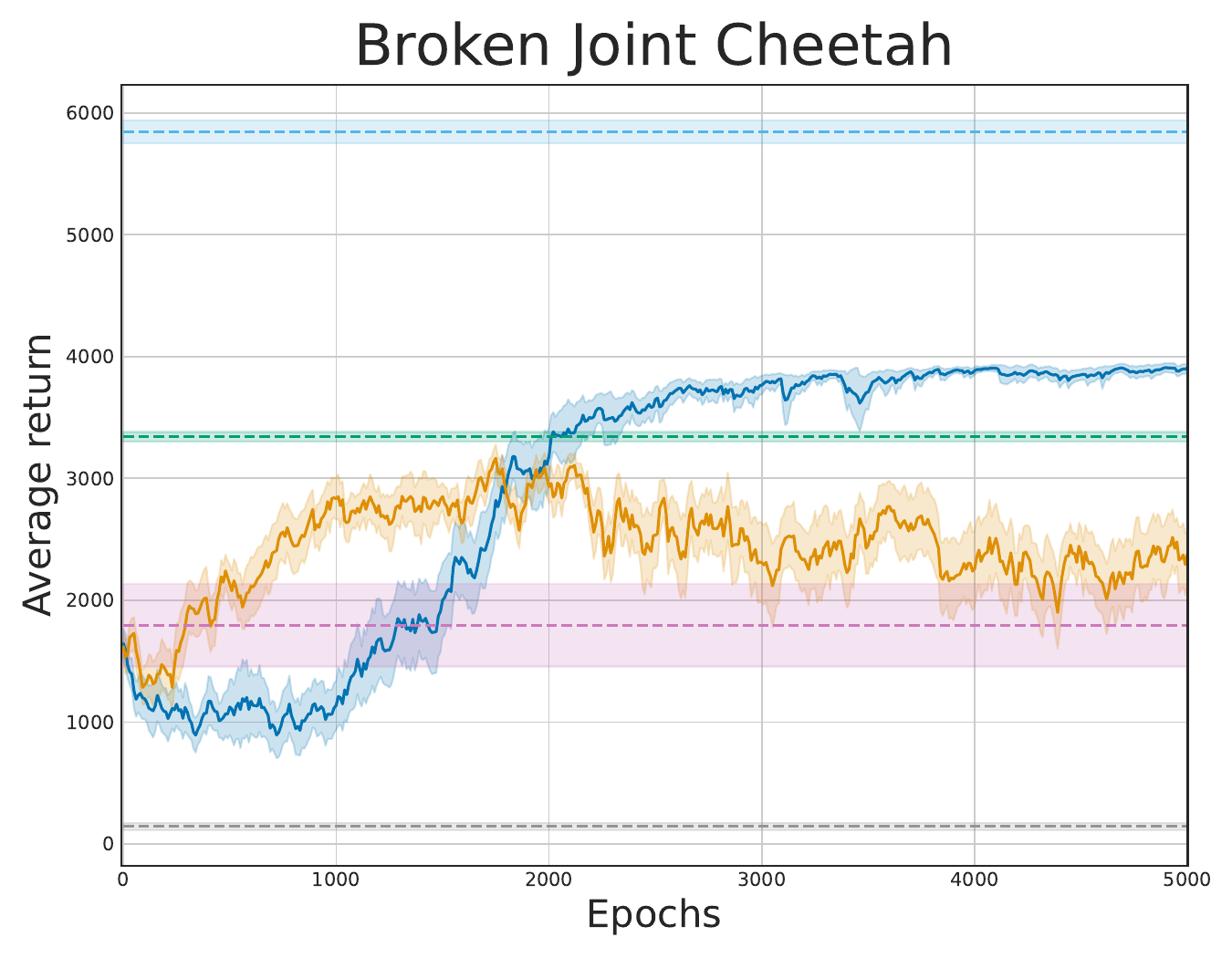}
\includegraphics[width=.325\textwidth]{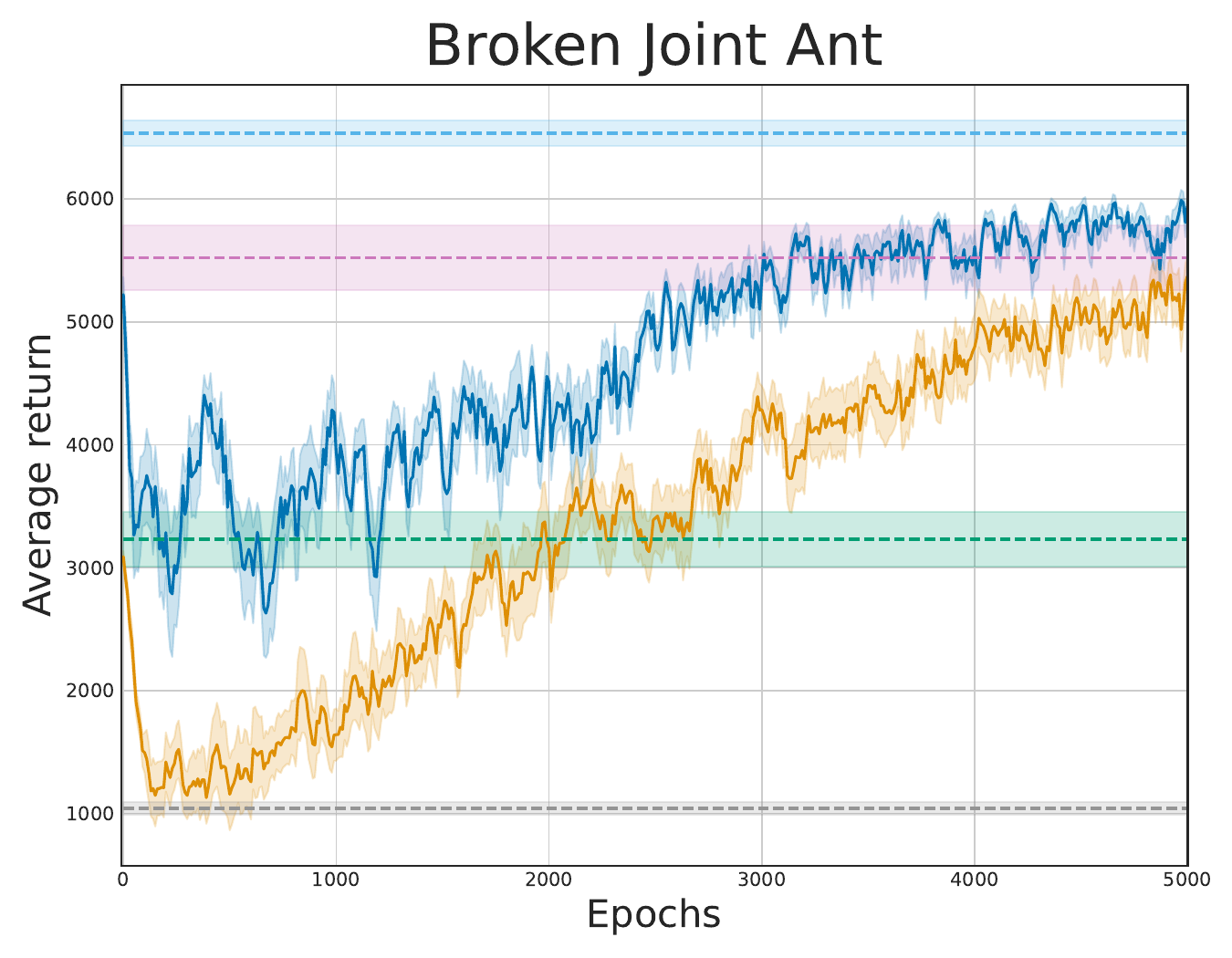}

\medskip

\includegraphics[width=.325\textwidth]{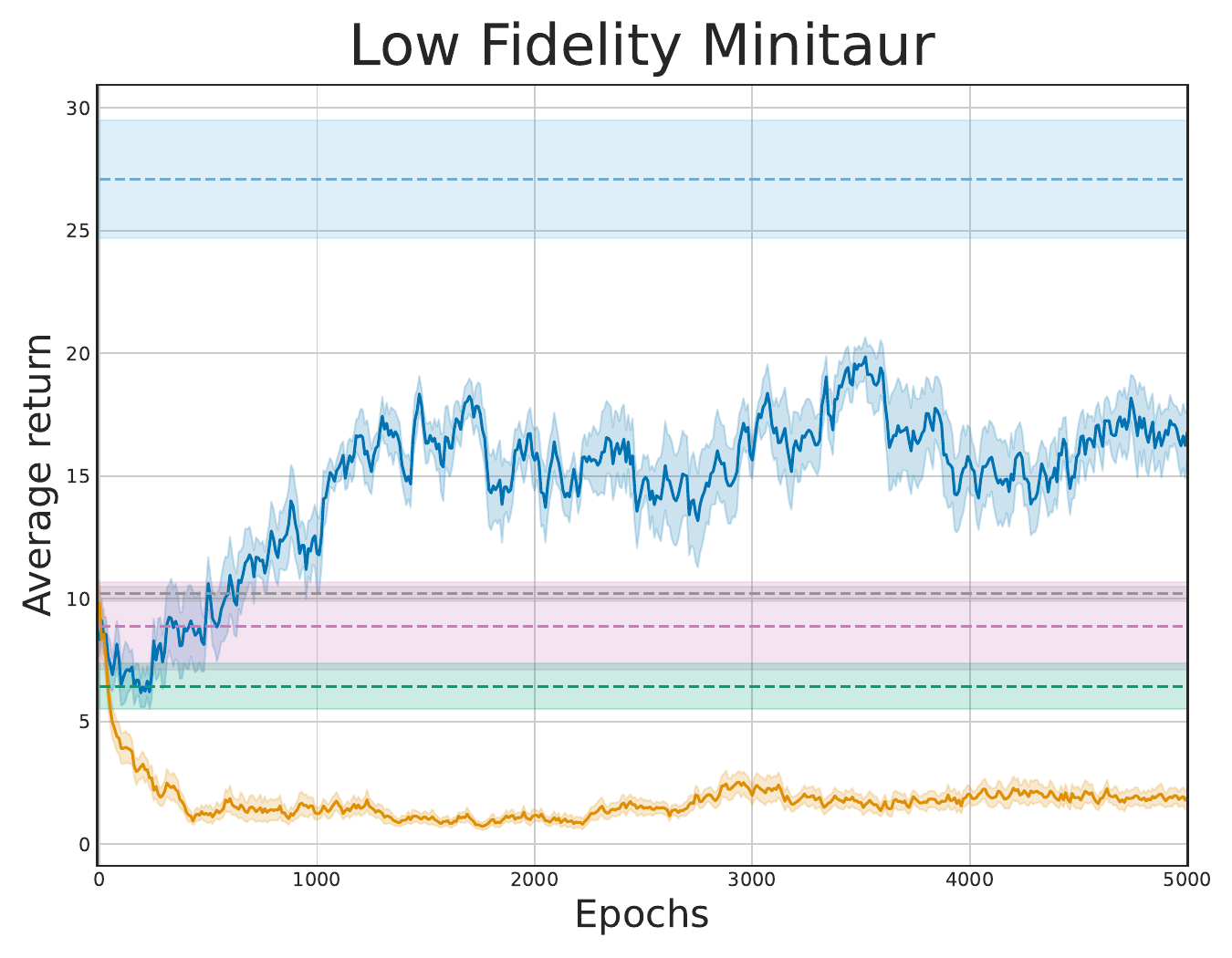}
\includegraphics[width=.325\textwidth]{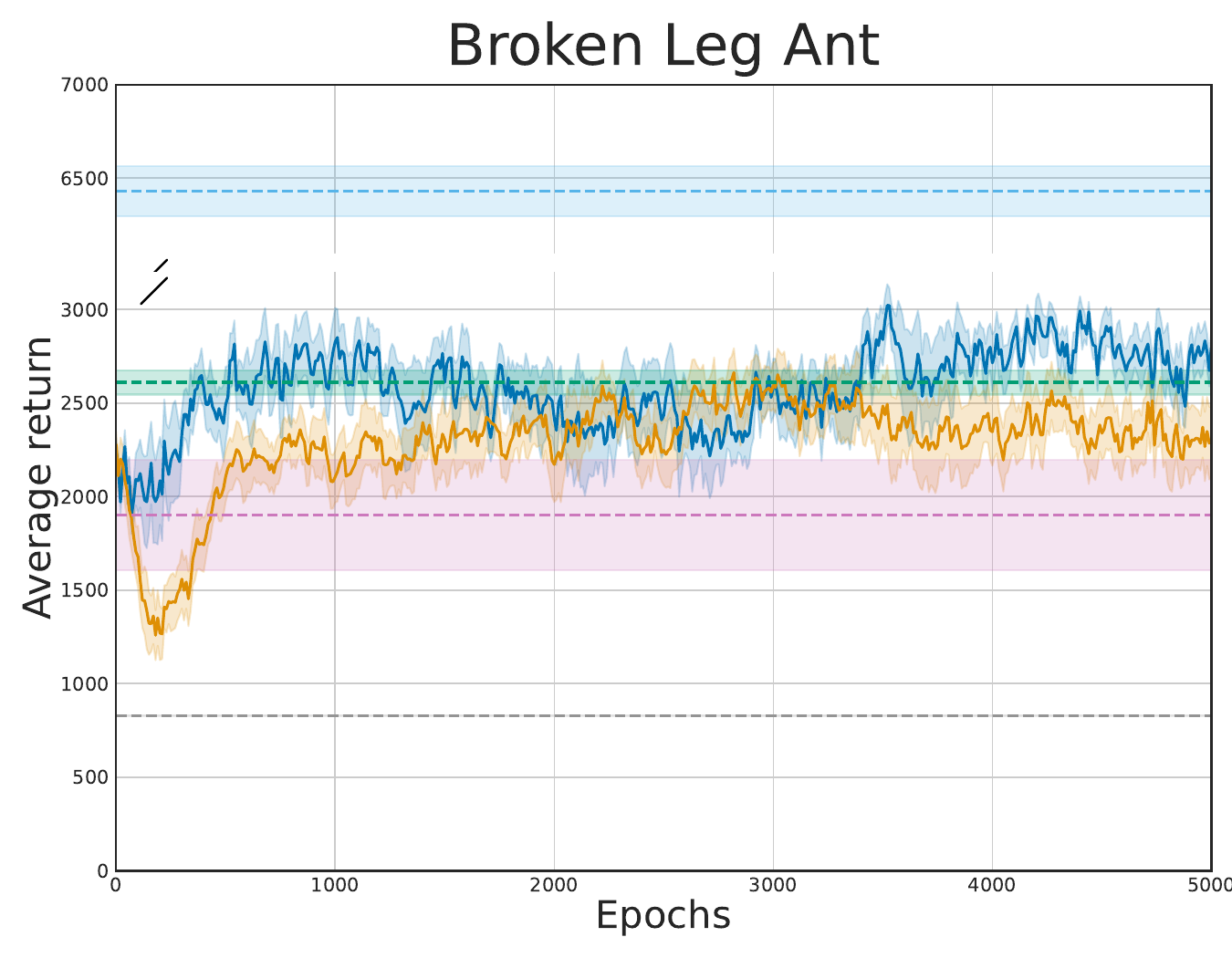}
\includegraphics[width=.325\textwidth]{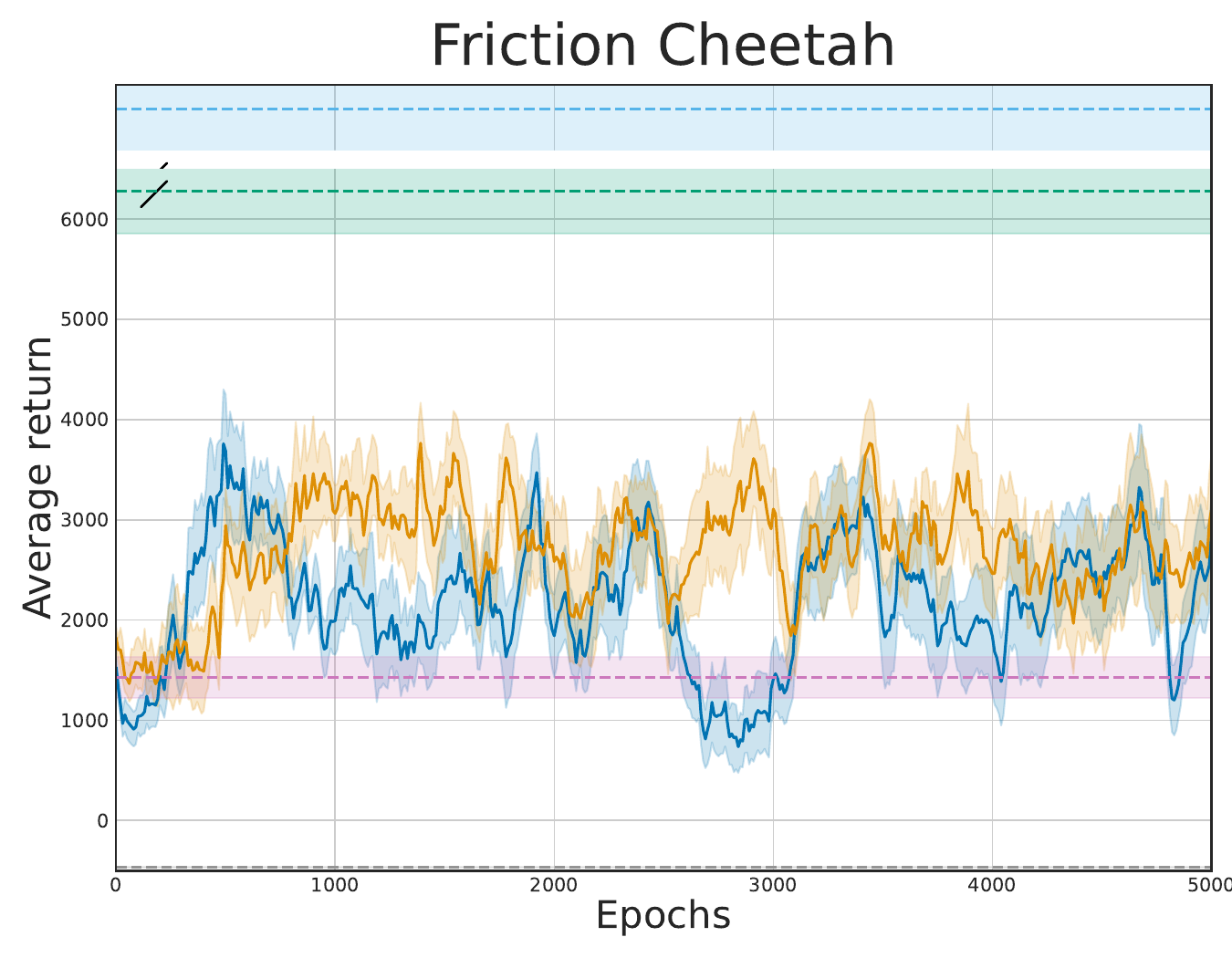}

\medskip

\includegraphics[width=.325\textwidth]{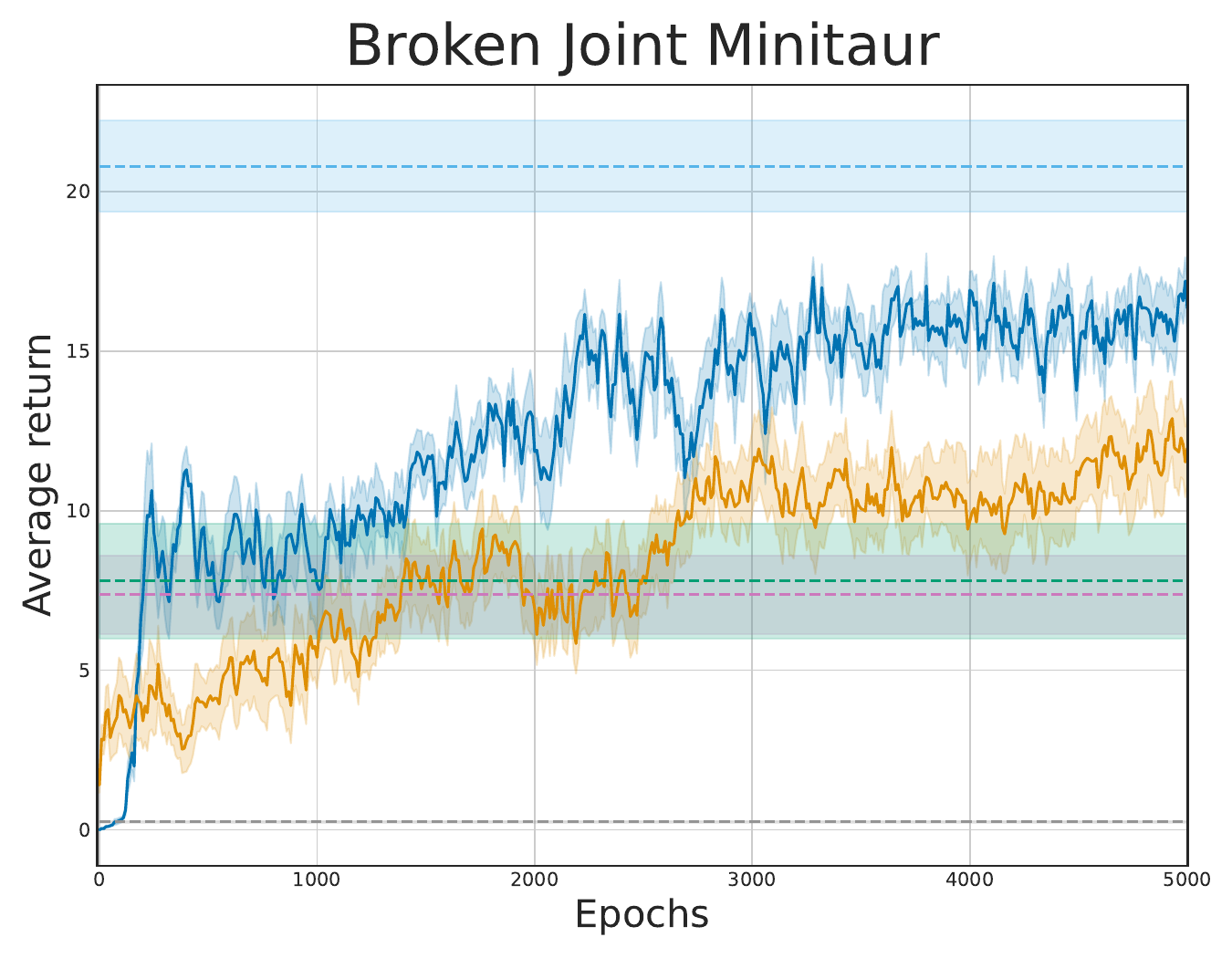}
\includegraphics[width=.325\textwidth]{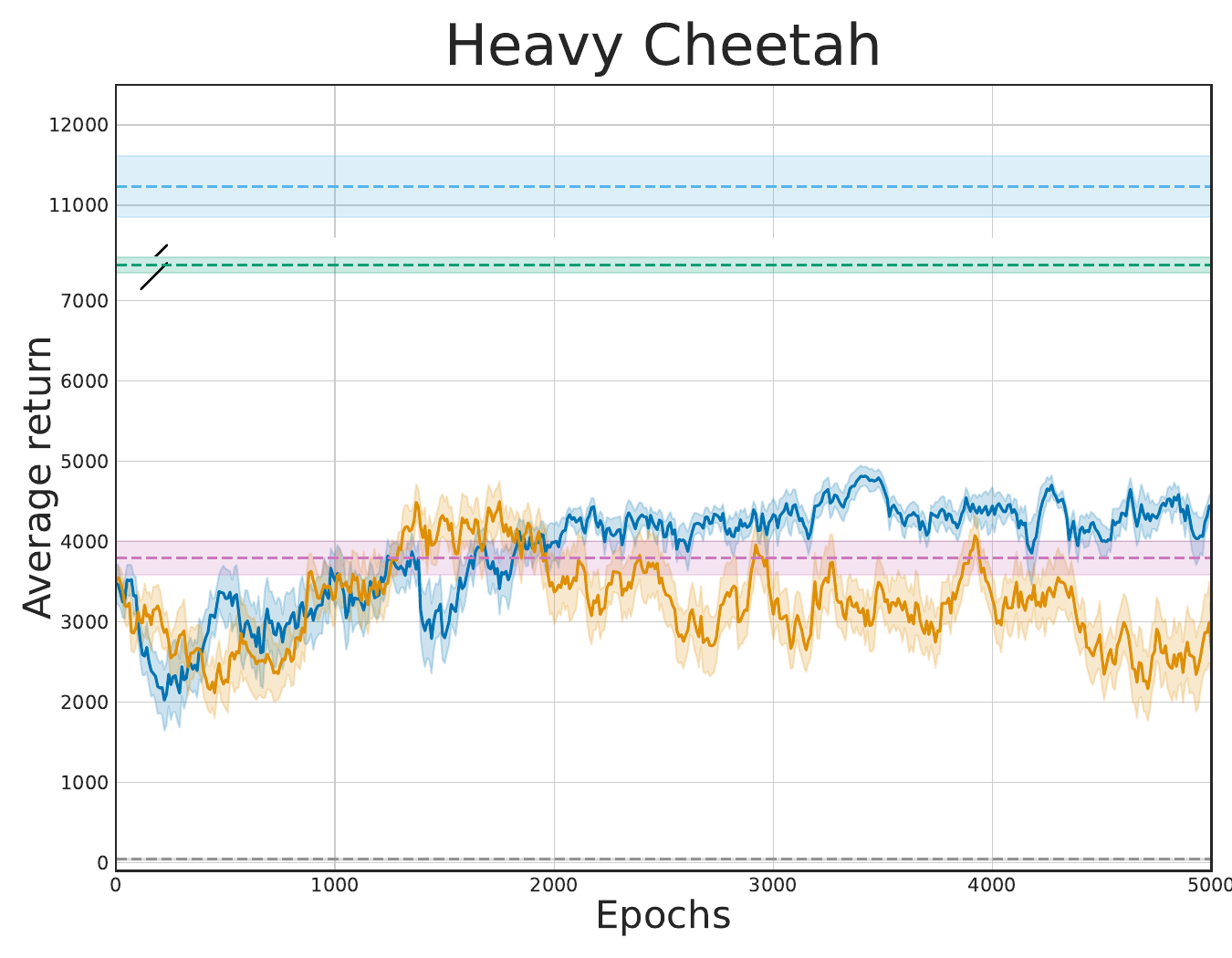}

\medskip

\includegraphics[width=.8\textwidth]{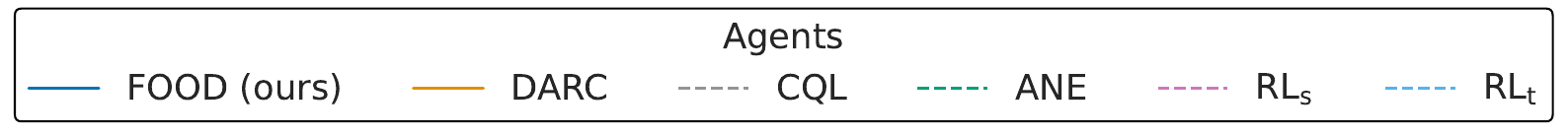}

\caption{Learning curves of FOOD and DARC for all the proposed environments.}
\label{fig:learning curves}
\end{figure}

\subsection{Global hyper-parameters} \label{app:hyperparam_details}

Our experiments are based on the A2C and PPO implementations proposed by the open-source code \cite{pytorchrl}. We also found that it may be profitable to add a TanH function at the end of the network's policy for the PPO agent to increase the performance of $\text{RL}_{\text{s}}$. We have selected their hyper-parameters according to the source \cite{rl_zoo3} and included them in Table~\ref{table:hyperparameters}.

\begin{table}[htbp]
\centering
\caption{Chosen hyper-parameters for both A2C and PPO. The PPO hyper-parameters were fixed for the other environments.}
\label{table:hyperparameters} 
\catcode`,=\active
\renewcommand\arraystretch{1.2}
\begin{tabular}{p{2.8cm}  p{3.5cm} p{2cm} }   
  \toprule
    Hyperparameters            & \textbf{A2C}                    &  \textbf{PPO}           \\  
  \midrule
    num-processes              &  $8$                            & $8$                   \\
    num-steps                  & $200$                           & $1000$                \\
    lr                         & $2.5 * 10^{-4}$                 & $3.0 * 10^{-4}$       \\
    $\gamma$                   & $0.99$                          & $0.99$                          \\
    use-gae                    & True                            & True                  \\
    gae-lambda                 & $0.9$                           & $0.95$                      \\ 
    entropy-coef               & $0.01$                          & $0.001$                        \\ 
    value-loss-coef            & $0.4$                           & $0.5$                       \\ 
    use-linear-lr-decay        & True                            & True                       \\ 
    ppo-epoch                  & N/A                             & $5$                       \\ 
    num-mini-batch             & N/A                             & $32$                       \\
    clip-param                 & N/A                             & $0.1$                       \\
    TanH Squash                & False                            & True                       \\
  \bottomrule
\end{tabular} 
\end{table}

\paragraph{The Minitaur environments} As proposed by the PyBullet library \cite{coumans2021}, $\gamma$ is set to $0.995$ for the Minitaur environments. Besides, unlike the Gym and Mujoco environments, they do not use a Tanh squashing function in their policy and the \texttt{num-processes} hyper-parameter is set to $1$.

\paragraph{Algorithms optimization} To allow a fair comparison between the different agents, FOOD, DARC, and ANE use the same underlying agent to optimize their objective. It is A2C for Gravity Pendulum and PPO for the others.

\paragraph{Discriminators training} Both FOOD and DARC incorporate classifiers in their objective. At each epoch, $1000$ data points are sampled from both source and target transition data sets. The classifiers are then trained with batch sizes of $128$ for Pendulum and $256$ for the MuJoCo environments. They share the same network structure: a $2$ hidden layer MLP with $64$ (for Pendulum) or $256$ (for MuJoCo) units and ReLU activations. We did not find that the size of the networks play an important role in the results.

\subsection{FOOD Hyper-parameters Sensitivity Analysis} \label{app:food_hyperparams_analysis}

This subsection investigates the impact of our main hyper-parameter $\alpha$, which regulates the strength of regularization that defines a threshold between maximizing the rewards of the source MDP and staying close to the target trajectories. All FOOD results are summarized in Figure~\ref{fig:complete_hyperparam_food}, where, similar to the previous section, FOOD uses the regularization with $d^\pi_P$ in Gravity Pendulum and $\nu^\pi_P$ for the other environments. Note that for the Gravity Pendulum environment, $\alpha\in\{ 0, 1, 5, 10 \}$.

\begin{figure}[ht!]
\centering
\includegraphics[width=.32\textwidth]{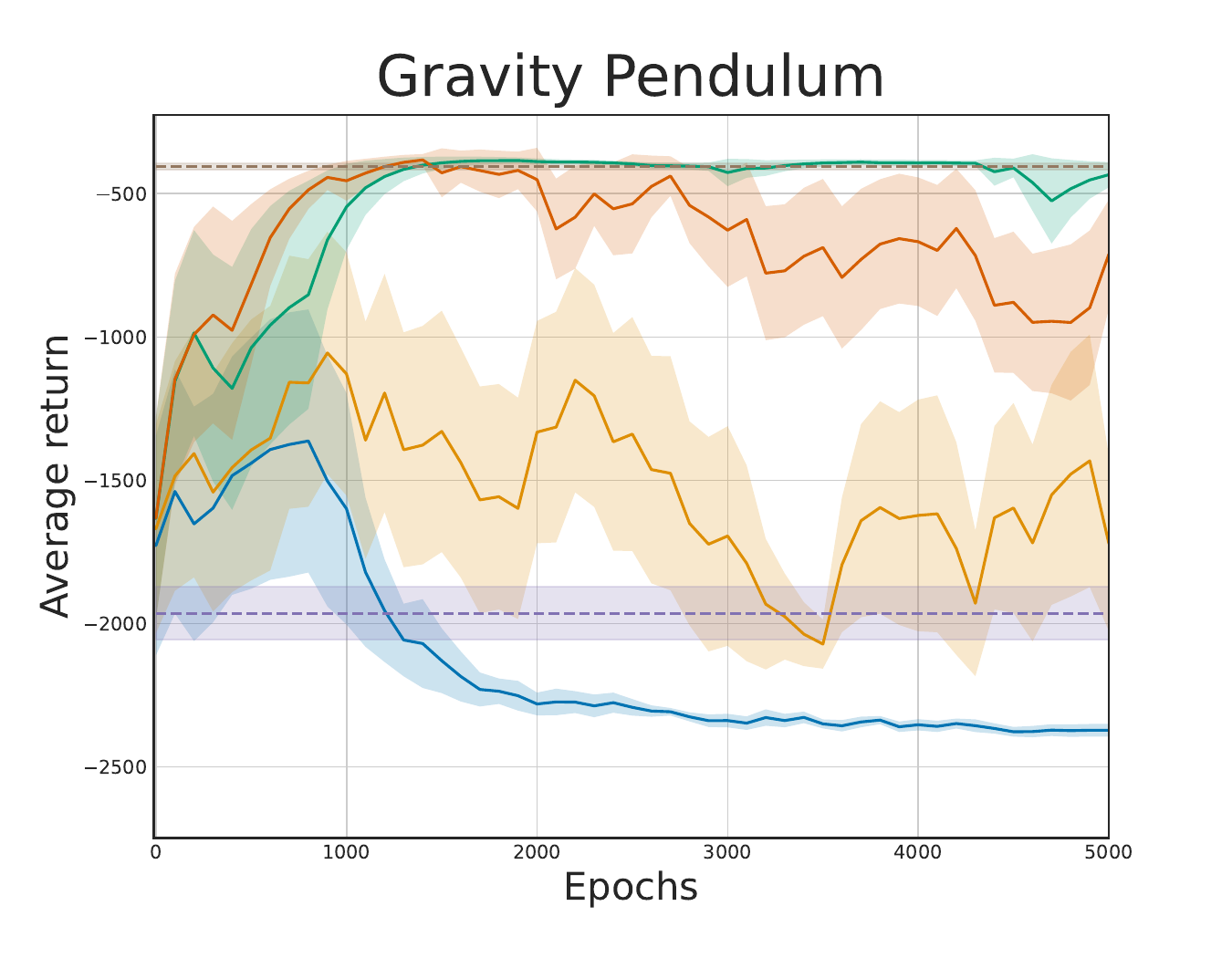}
\includegraphics[width=.32\textwidth]{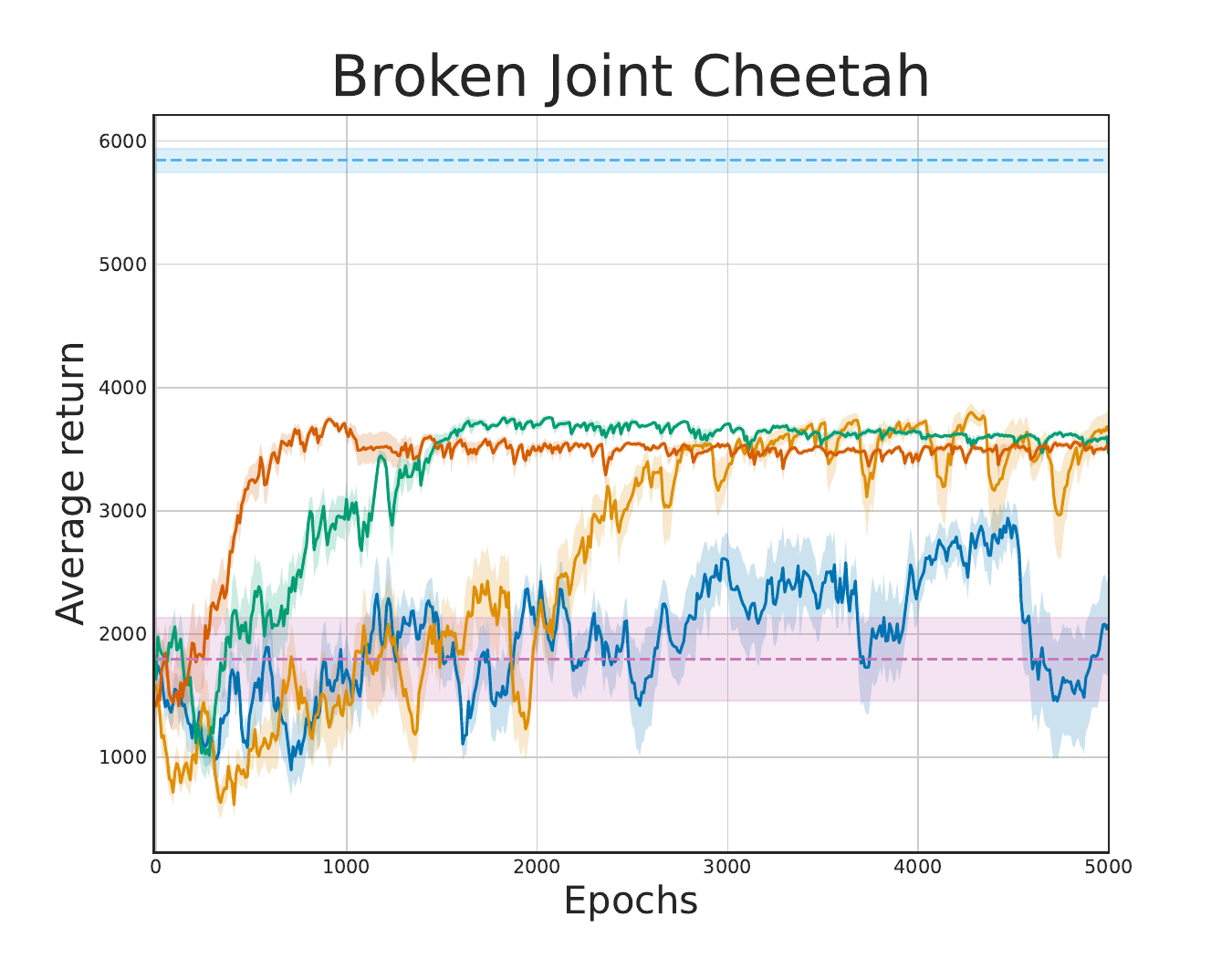}
\includegraphics[width=.32\textwidth]{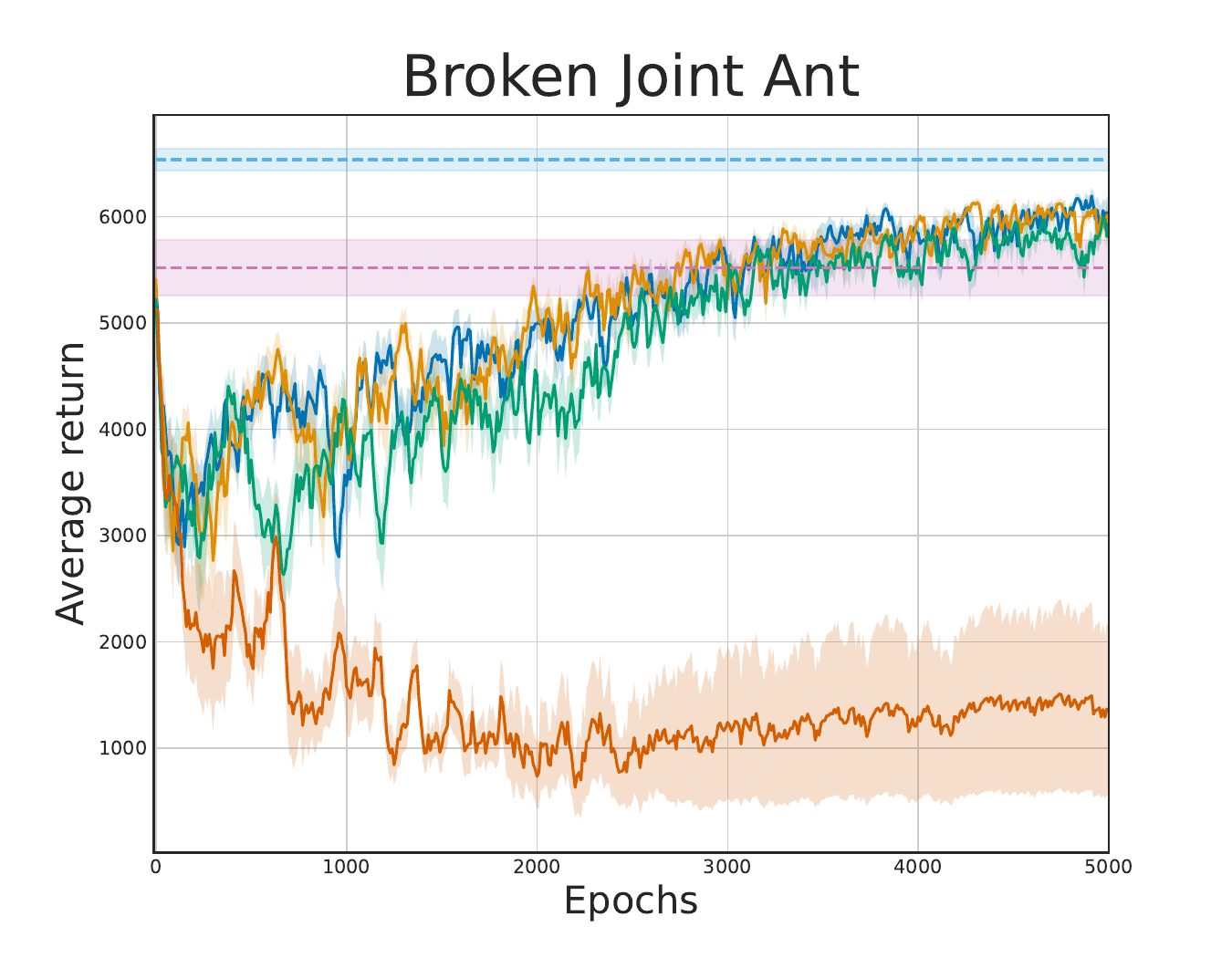}

\medskip

\includegraphics[width=.32\textwidth]{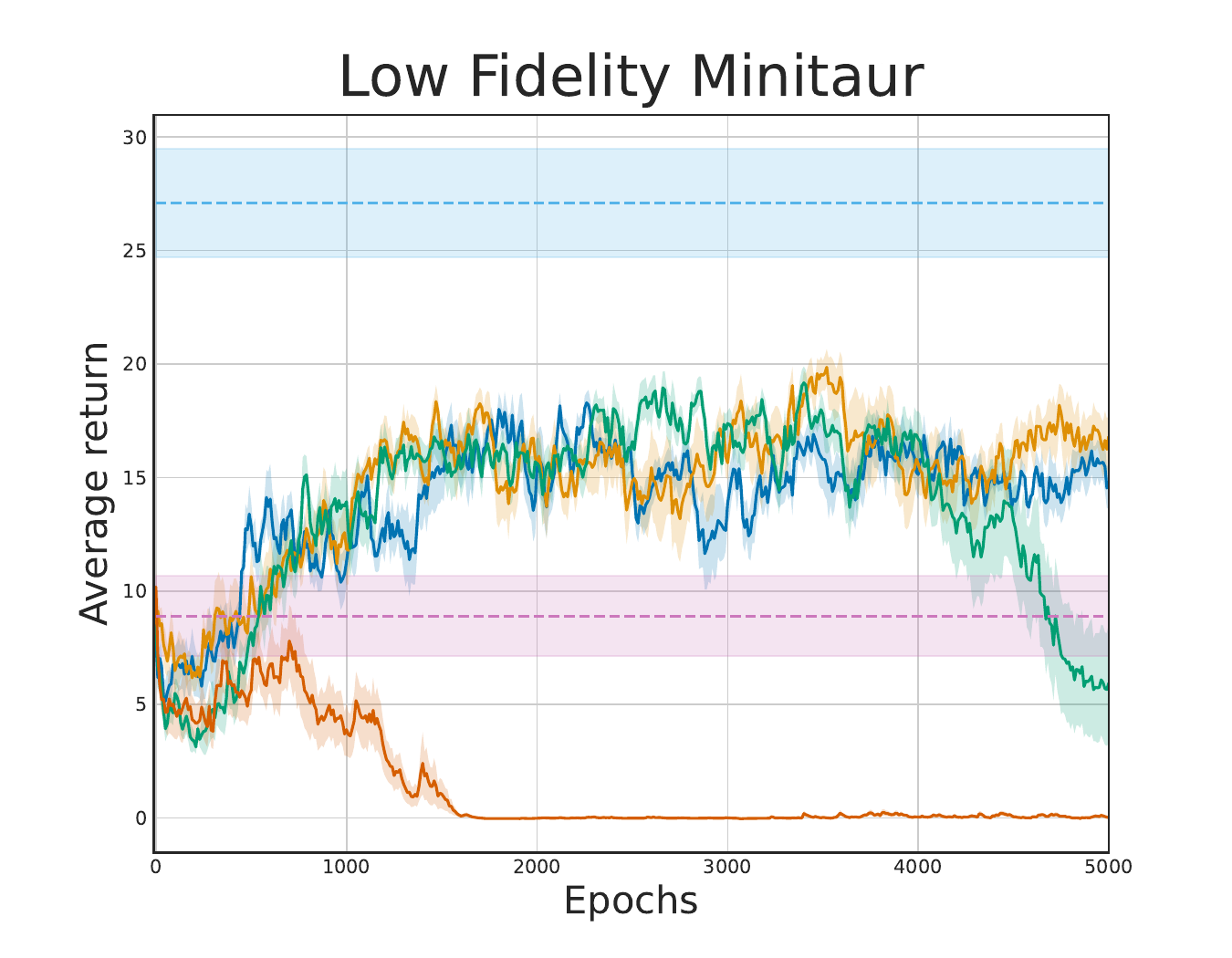}
\includegraphics[width=.32\textwidth]{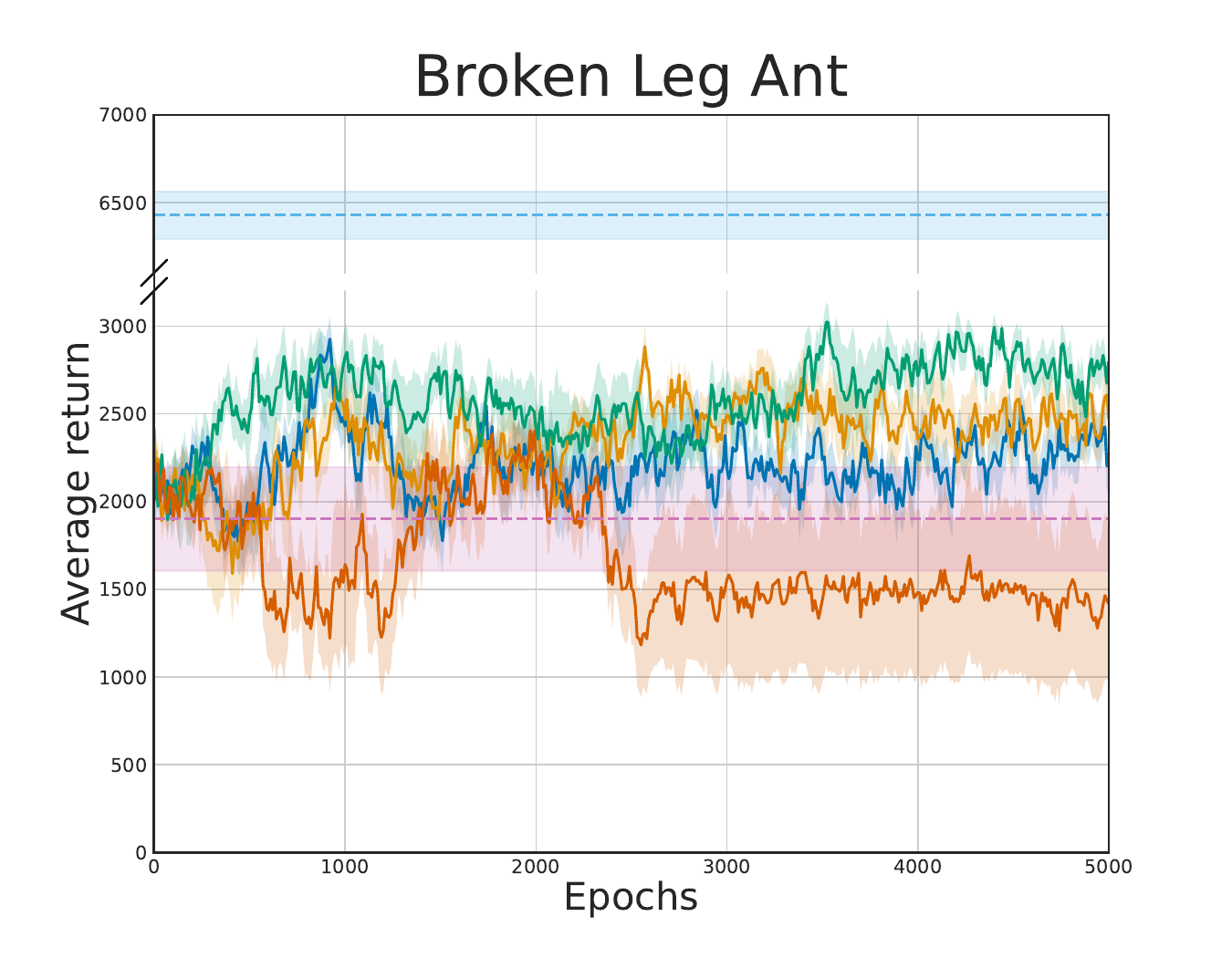}
\includegraphics[width=.32\textwidth]{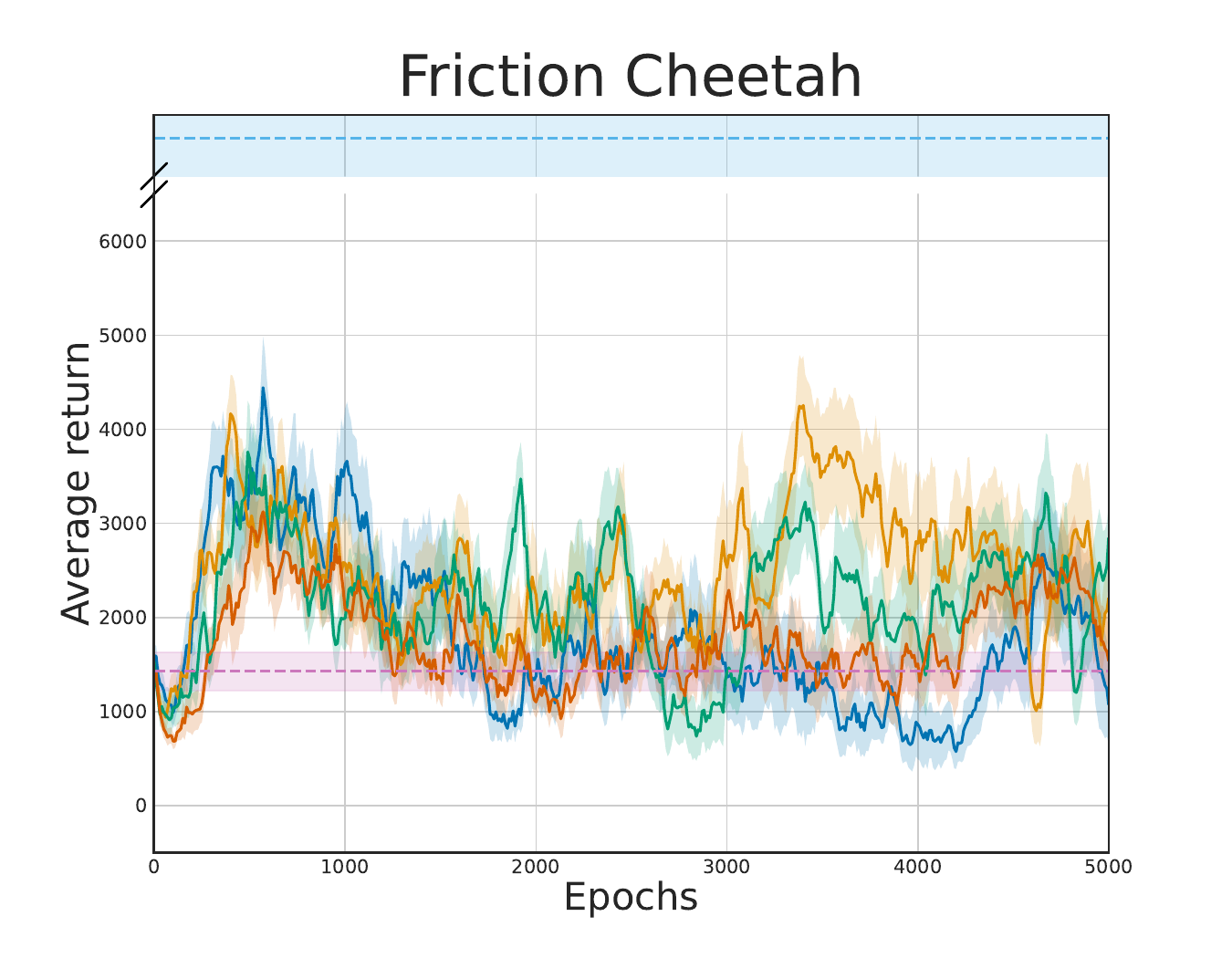}

\medskip

\includegraphics[width=.32\textwidth]{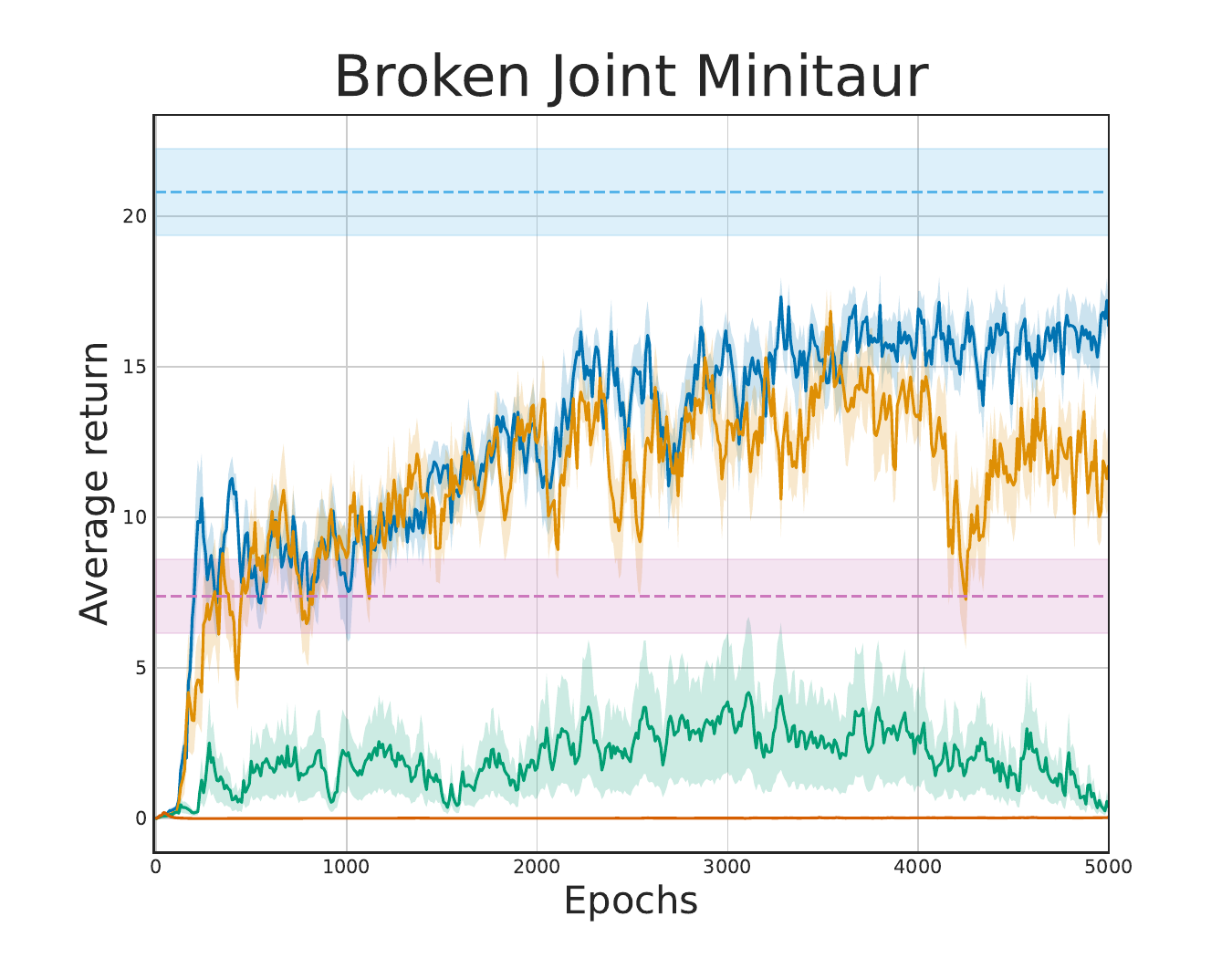}
\includegraphics[width=.32\textwidth]{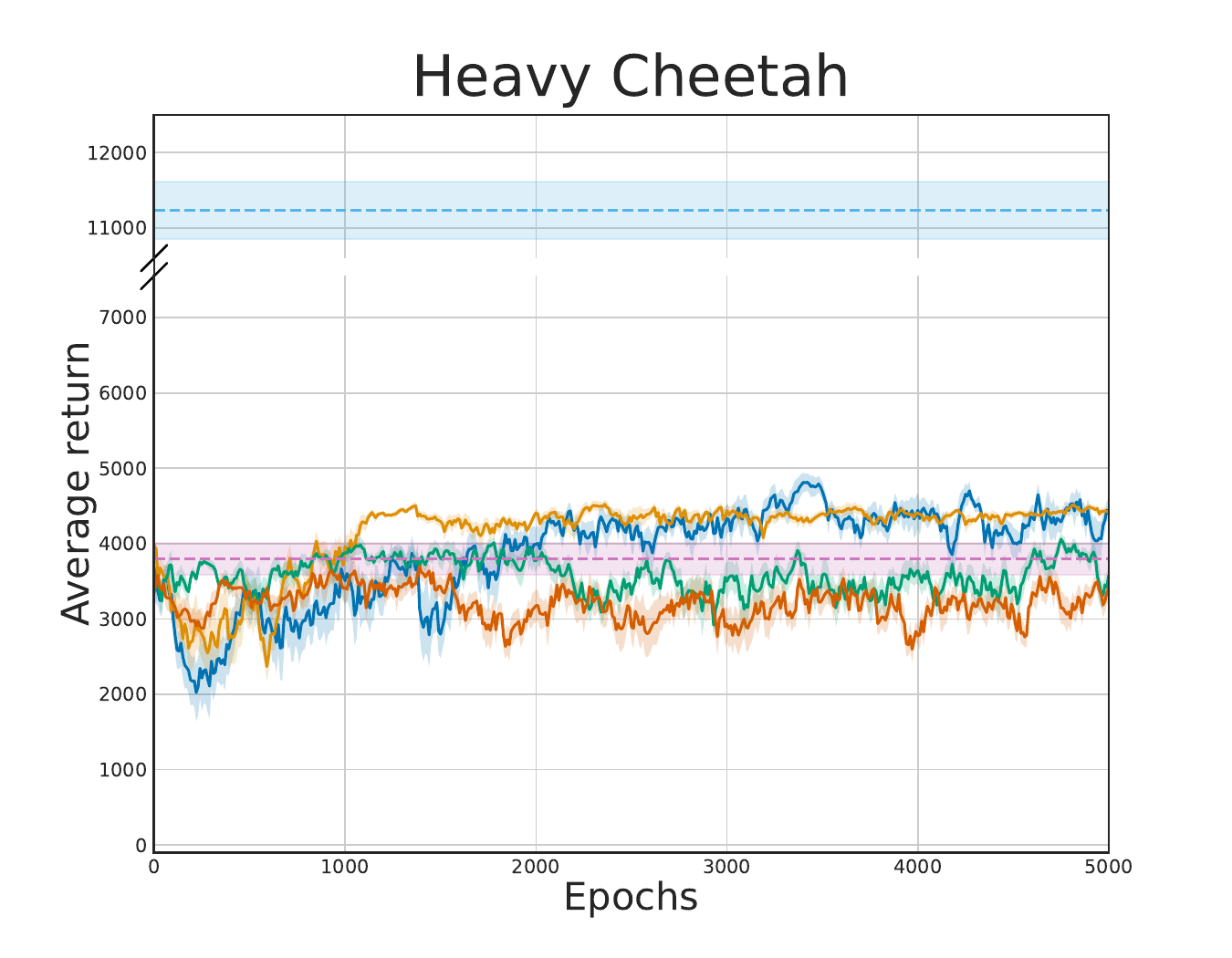}

\medskip

\includegraphics[width=.8\textwidth]{clean_imgs/legend_hyperparams_food_final.pdf}

\caption{Complete hyperparameter sensitivity analysis for the best FOOD agent on the different off-dynamics environments.}
\label{fig:complete_hyperparam_food}
\end{figure}

In all the studied environments where PPO was used, we observe that unless for the low or high values of $\alpha$ ($\alpha \in \{ 0.5, 5 \}$), the FOOD agent improves performance compared to $\text{RL}_\text{s}$. Both cases can be explained. If the value is too high, it may disrupt the gradients and prevent convergence to a good solution. As mentioned in the main paper, this phenomenon also affects the performance in the source environment, so it would be easy for practitioners to remove such bad hyper-parameters. It may also happen that the strength of the regularization is too low. In that case, FOOD has approximately the same performance as $\text{RL}_\text{s}$, as illustrated in Broken Joint HalfCheetah.

Hence, we recommend setting the regularization to have approximately the same weight as the average return. For this, since its advantages are normalized, we recommend using PPO and setting the $\alpha$ parameter to $1$.

\subsection{Comparison between the different IL algorithms for the FOOD agent} \label{app:il_comparison}

FOOD is a general algorithm that may use any chosen Imitation Learning algorithm. Each algorithm minimizes a certain type of divergence between state or state-action visitation distributions, as summarized in Table~\ref{table:imitation_distances}. Here, we investigate which IL is better suited for the considered environments. 

We compare GAIL-$\mu^\pi_P$ \cite{ho2016generative}, GAIL-$d^\pi_P$, GAIL-$\nu^\pi_P$, AIRL-$\mu^\pi_P$ \cite{fu2017learning}, PWIL-$\mu^\pi_P$ \cite{dadashi2020primal}, PWIL-$d^\pi_P$ and PWIL-$\nu^\pi_P$ in Table~\ref{table:comparison_distances}. GAIL and its extensions were extracted directly from \cite{pytorchrl}, AIRL from \cite{airlcode}, and PWIL and its extensions were recoded from scratch.

\begin{table}[ht!]
\centering 
\resizebox{\columnwidth}{!}{\begin{tabular}{|p{3.1cm}||c|c|c|c|c|c|c|}
  \hline
  Environment & GAIL-$d$ & GAIL-$\mu$ & GAIL-$\nu$ & AIRL-$\mu$ & PWIL-$d$ & PWIL-$\mu$ & PWIL-$\nu$  \T\B \\ \hline
  
  Gravity Pendulum
  & $\bm{-485 \pm 54^*}$
  & $-2224 \pm 43$ 
  & $-2327 \pm 14$ 
  & $-1926 \pm 572$
  & $\bm{-980 \pm 838}$
  & $\bm{-948 \pm 789}$
  & $\bm{-978 \pm 816}$
  \T\B \\

  Broken Joint Cheetah
  & $\bm{3888 \pm 201}$
  & $\bm{3801 \pm 155}$ 
  & $\bm{3921 \pm 85^*}$ 
  & $3617 \pm 225$
  & $3537 \pm 248$
  & $2999 \pm 752$
  & $\bm{3797 \pm 389}$
  \T\B \\

  Heavy Cheetah
  & $\bm{4828 \pm 553}$
  & $\bm{4876 \pm 181^*}$
  & $\bm{4519 \pm 240}$
  & $\bm{4604 \pm 184}$
  & $2945 \pm 856$
  & $2771 \pm 1235$
  & $3494 \pm 318$
  \T\B \\

  Broken Joint Ant
  & $5547 \pm 204$
  & $\bm{6145 \pm 98^*}$
  & $\bm{6135 \pm 122}$
  & $5014 \pm 401$
  & $3725 \pm 988$
  & $3483 \pm 747$
  & $3182 \pm 1337$
  \T\B \\

  Friction Cheetah
  & $\bm{3212 \pm 2279}$
  & $\bm{3890 \pm 1495}$
  & $\bm{3289 \pm 236}$
  & $\bm{2957 \pm 1526}$
  & $\bm{3451 \pm 361}$
  & $\bm{3926 \pm 735}$
  & $\bm{4227 \pm 740^*}$
  \T\B \\

  Broken Joint Minitaur
  & $\bm{13.6 \pm 3.8}$
  & $\bm{14.9 \pm 3}$
  & $\bm{16.9 \pm 4.7^*}$
  & $\bm{15.8 \pm 2.3}$
  & $\bm{14.6 \pm 1.9}$
  & $\bm{12.1 \pm 5.2}$
  & $10.5 \pm 6.1$
  \T\B \\

  Low Fidelity Minitaur
  & $15.7 \pm 2.8$
  & $\bm{17 \pm 2 }$
  & $\bm{17.6 \pm 0.4^*}$
  & $7.5 \pm 5.7$
  & $\bm{13.6 \pm 5.1}$
  & $11.4 \pm 3.5$
  & $12.1 \pm 5.5$
  \T\B \\

  Broken Leg Ant
  & $\bm{2345 \pm 806}$
  & $\bm{2652 \pm 356} $
  & $\bm{2977 \pm 85^*} $
  & $1634 \pm 857$
  & $1490 \pm 714$
  & $1554 \pm 886$
  & $1697 \pm 393$
  \T\B \\
  \hline

\end{tabular} }
\caption{FOOD sensitivity analysis with respect to the Imitation Learning agent used. We report the average return over $4$ seeds associated with their best hyper-parameter $\alpha$.}%
\label{table:comparison_distances}
\end{table}

Overall, we observe that all GAIL-associated algorithms have the best results. We attribute this success to the implementation we used, which was optimized for the PPO agent. In addition, FOOD with PWIL has poor results in some environments. This can be attributed to two factors. First, we cannot rule out an error in our code, as we coded it from scratch. Second, this algorithm was introduced in the D4PG agent \cite{barth2018distributed}: it is possible that PPO does not leverage well the PWIL's rewards.

An interesting discussion is about GAIL-$d^\pi_P$, GAIL-$\mu^\pi_P$ and GAIL-$\nu^\pi_P$. Intuitively, the one that focuses on state visitation distributions should give the FOOD agent more freedom to find a better action. This is for example what is observed in the Gravity Pendulum environment. However, in most cases,  GAIL-$\mu^\pi_P$ or GAIL-$\nu^\pi_P$ provide better results as they provide more information regarding the target trajectories. GAIL-$\nu^\pi_P$ is the one directly derived from Proposition~\ref{prop:trust_region}, and it seems GAIL-$\mu^\pi_P$ is implicitely able to optimize the second term in Proposition~\ref{prop:trust_region_1}.


\subsection{Data sensitivity analysis}\label{app:study_data}

In this sub-section, we conduct a comparative analysis between FOOD and DARC across the environments where PPO is used on the number of source trajectories they use. The trained agent $\text{RL}_{\text{s}}$ samples $5$, $10$, $25$ and $50$ trajectories on the source environment. During certain trajectories, the robot directly falls: we exclude them for both FOOD and DARC to avoid misleading regularization.

As depicted in Figure~\ref{fig:data_analysis_agents}, both methods demonstrate relative robustness to the number of source trajectories. Their reliance on a discriminator explains why a small number of trajectories appears to be sufficient for the development of a good agent. Additional insights can be extracted from Figure~\ref{fig:data_analysis_agents}. First, in Friction Cheetah, a larger amount of target data allows DARC to outperform FOOD. Second, in Broken Leg Ant and Heavy Cheetah, an increased number of trajectories decreases FOOD’s performance. This decline may result from including trajectories that have medium to poor performance in the target environment, leading to misguided regularization.

\begin{figure}[ht!]
\centering
\includegraphics[width=.32\textwidth]{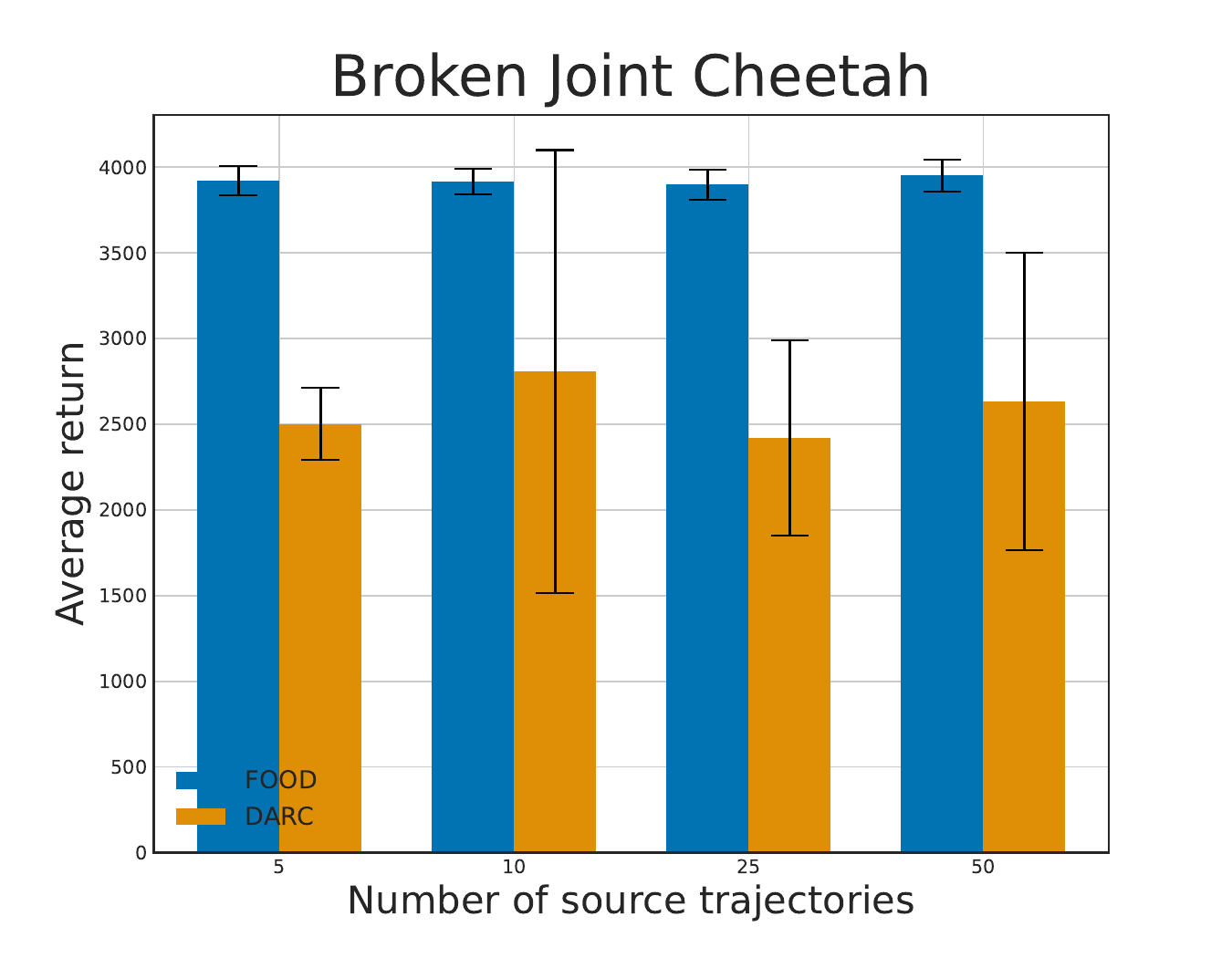}
\includegraphics[width=.32\textwidth]{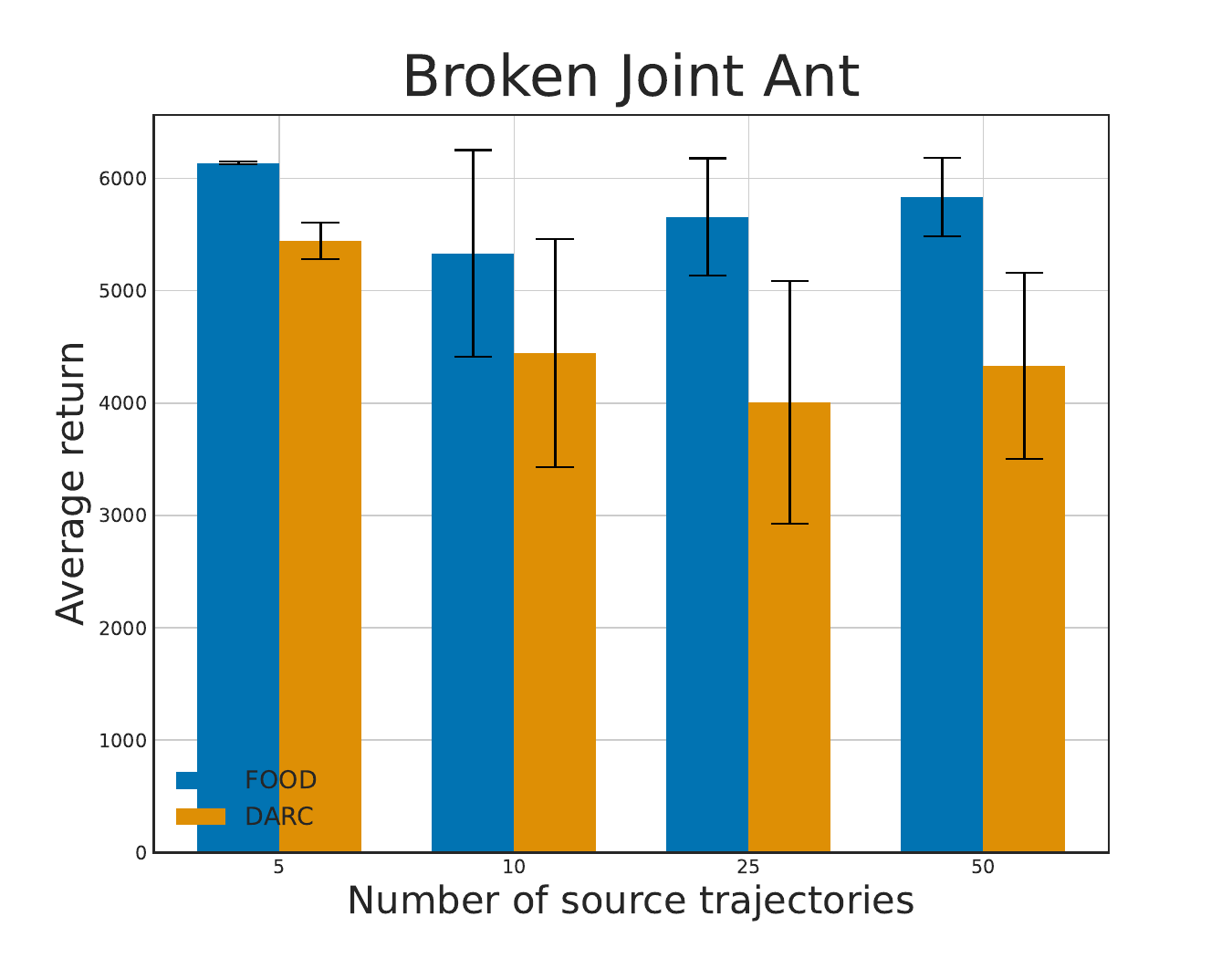}
\includegraphics[width=.32\textwidth]{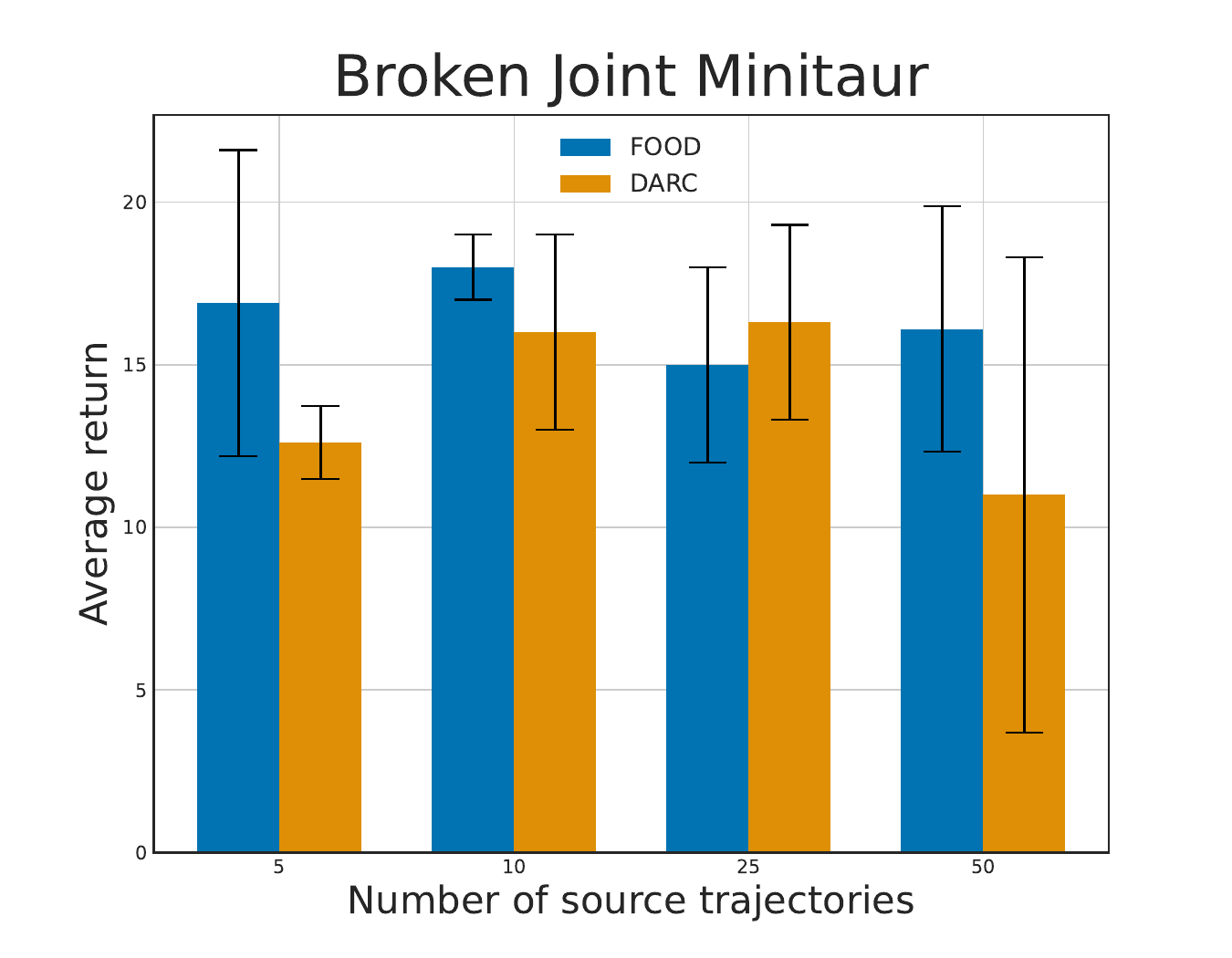}
\medskip
\includegraphics[width=.32\textwidth]{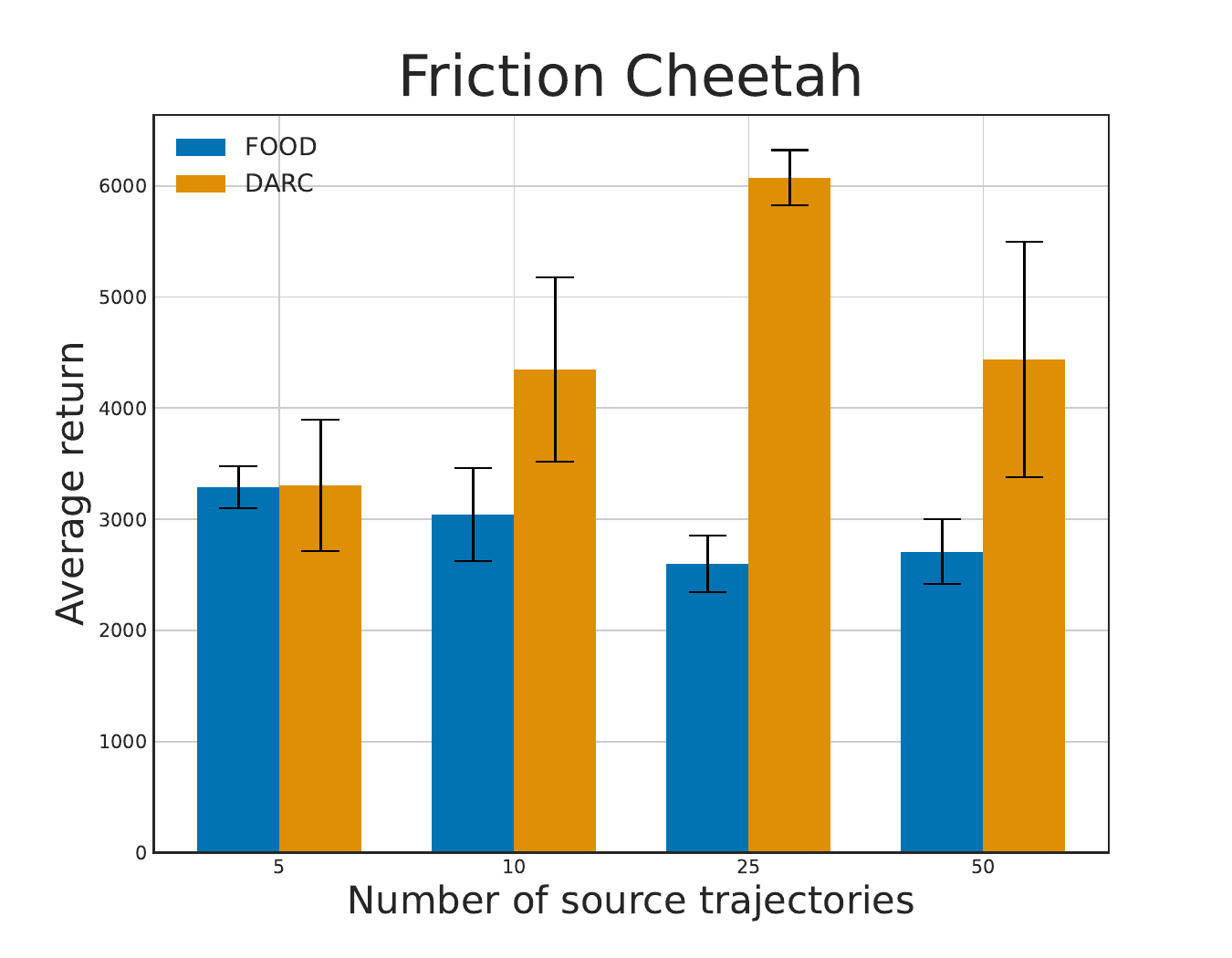}
\includegraphics[width=.32\textwidth]{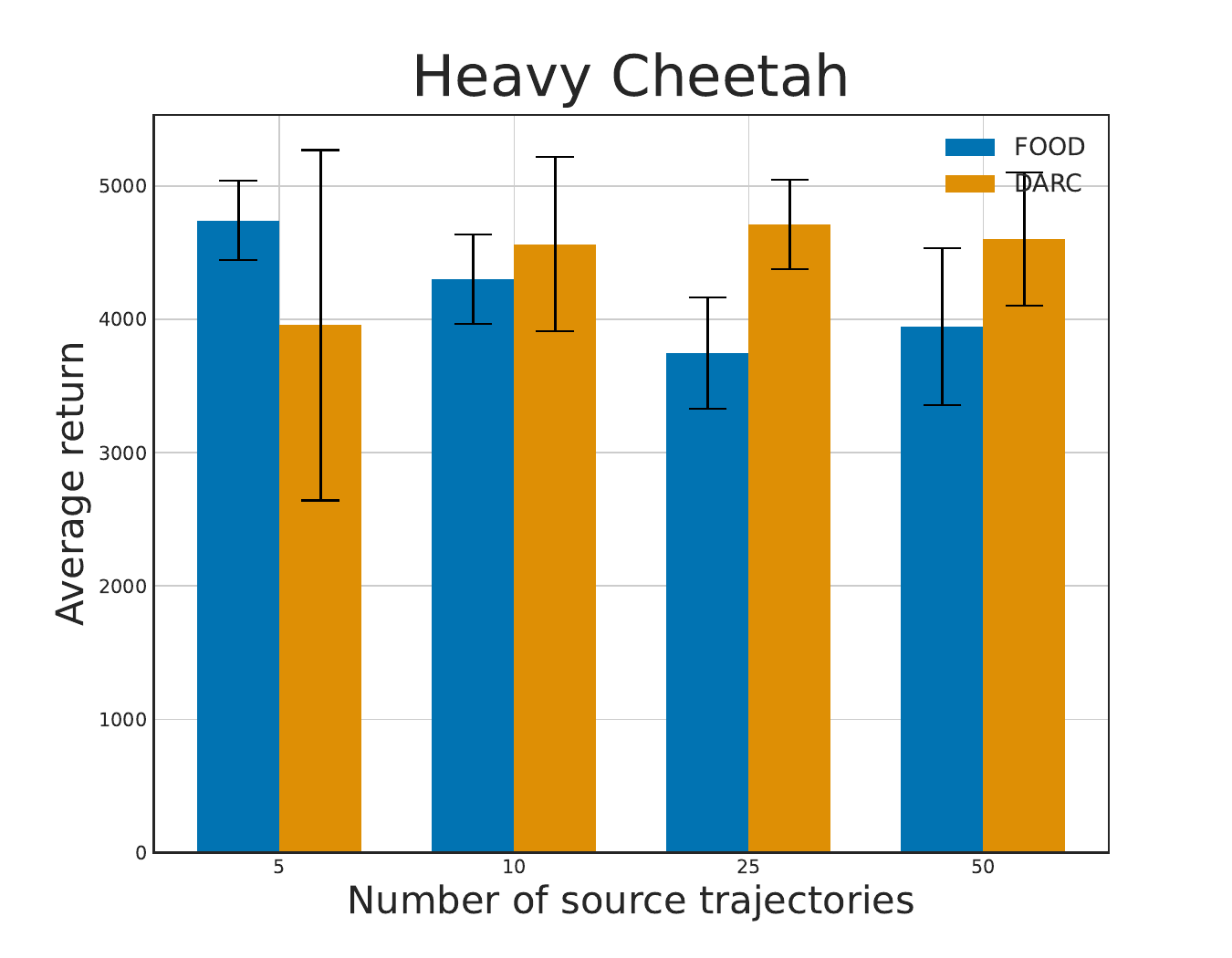}
\includegraphics[width=.32\textwidth]{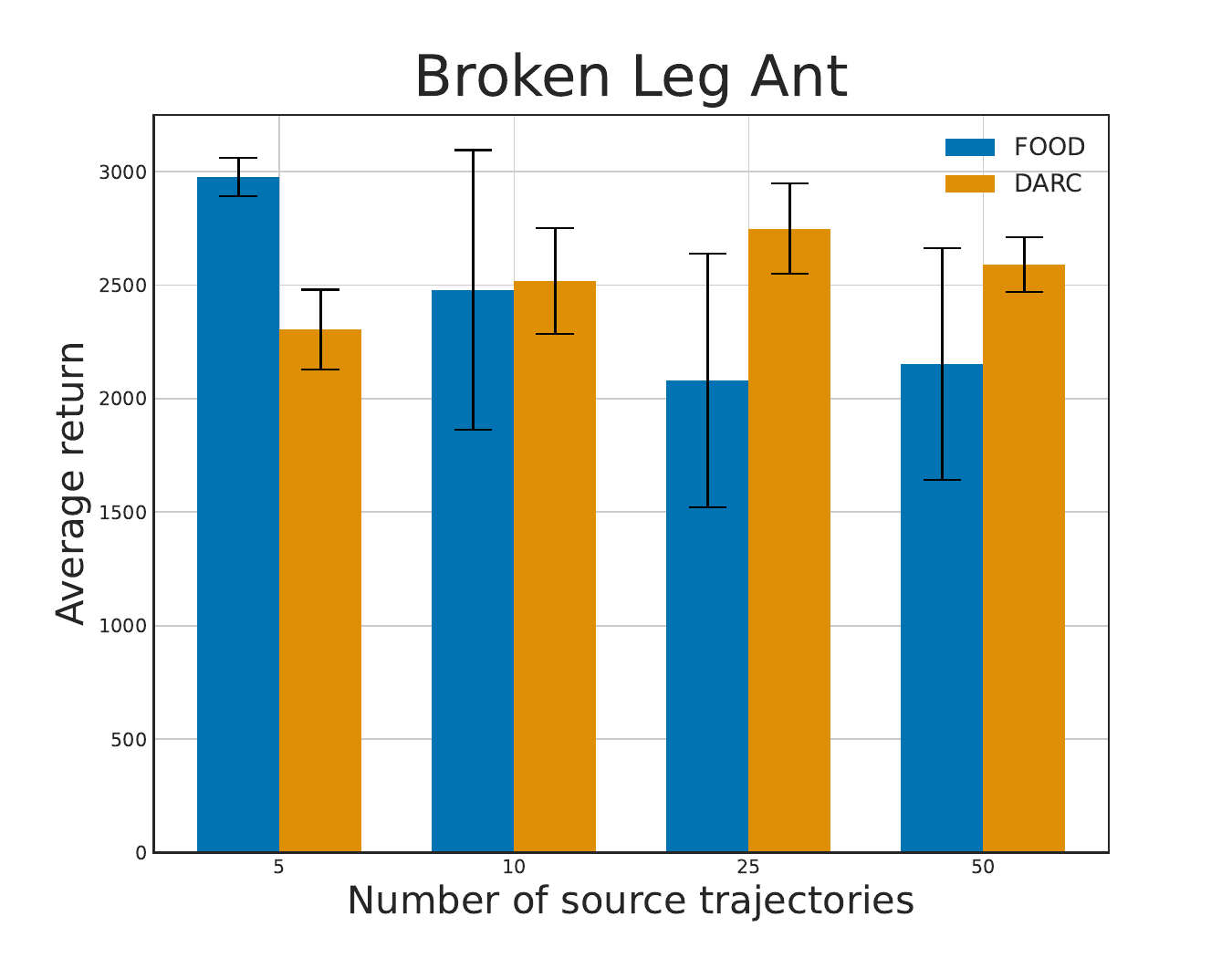}
\medskip
\includegraphics[width=.32\textwidth]{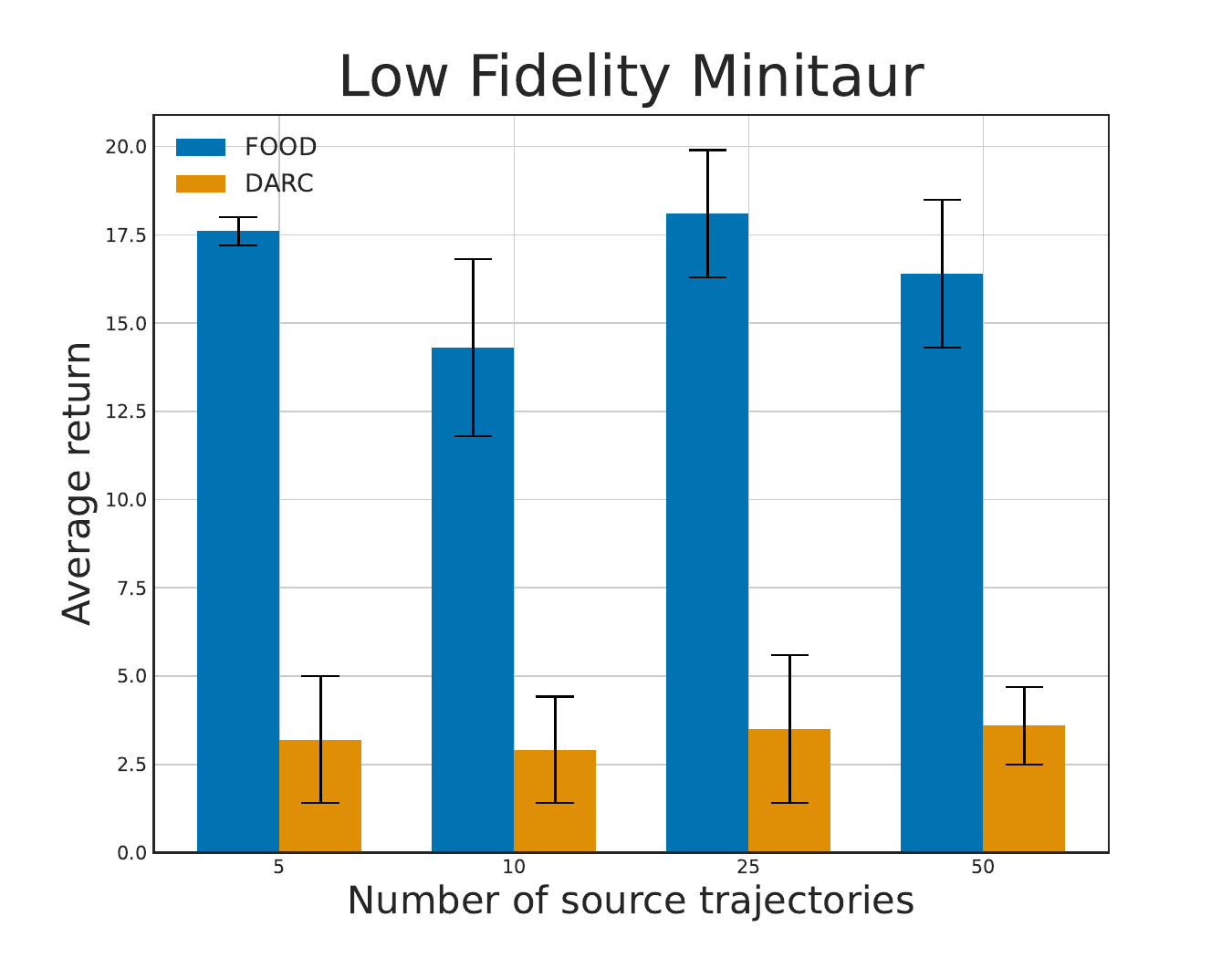}
\caption{Data sensitivity analysis for both FOOD and DARC agents on the environments where PPO is used.}
\label{fig:data_analysis_agents}
\end{figure}

\subsection{DARC Hyperparameters Sensitivity Analysis} \label{app:darc_hyperparams_analysis}

\begin{figure}[ht!]
\centering
\includegraphics[width=.32\textwidth]{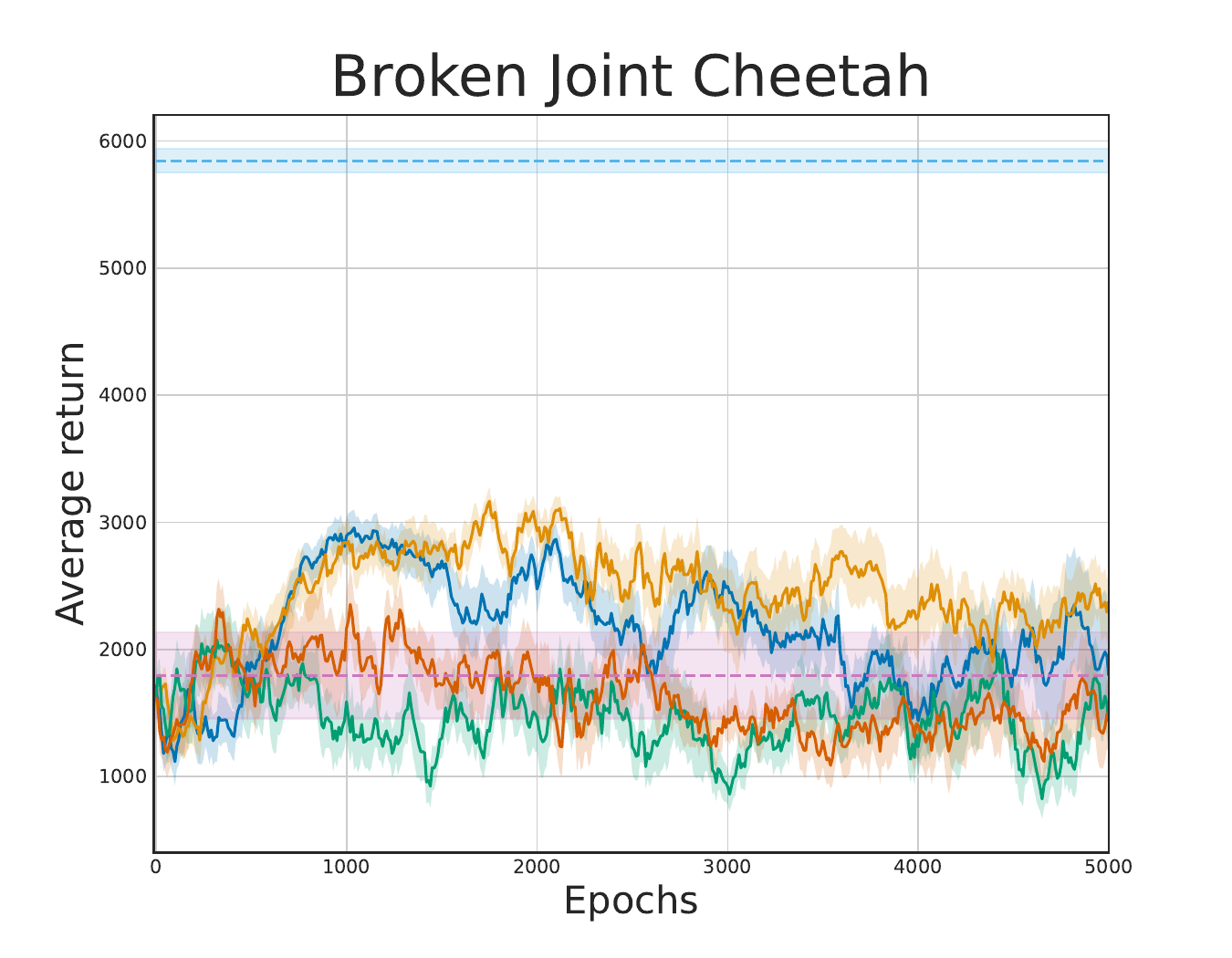}
\includegraphics[width=.32\textwidth]{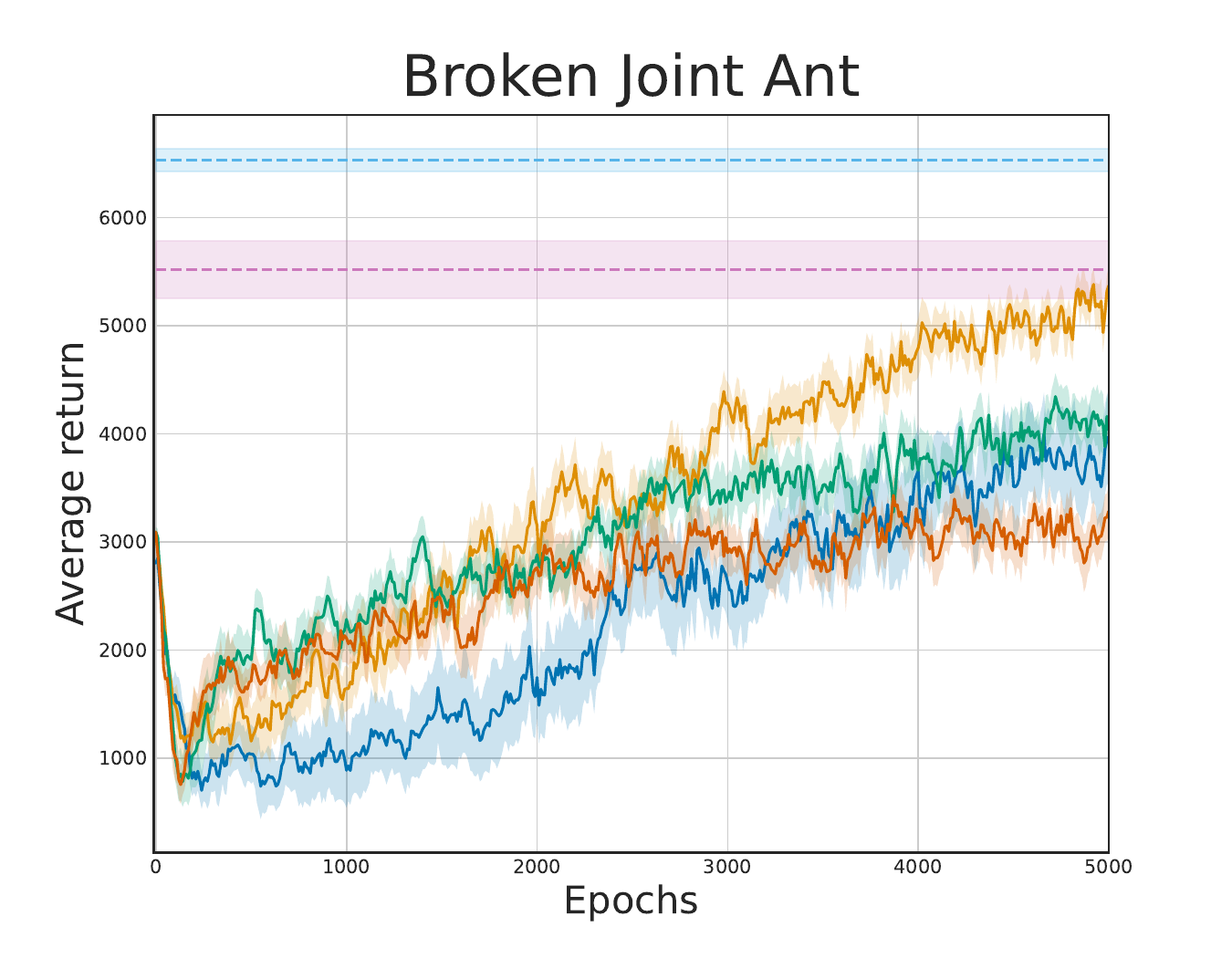}
\includegraphics[width=.32\textwidth]{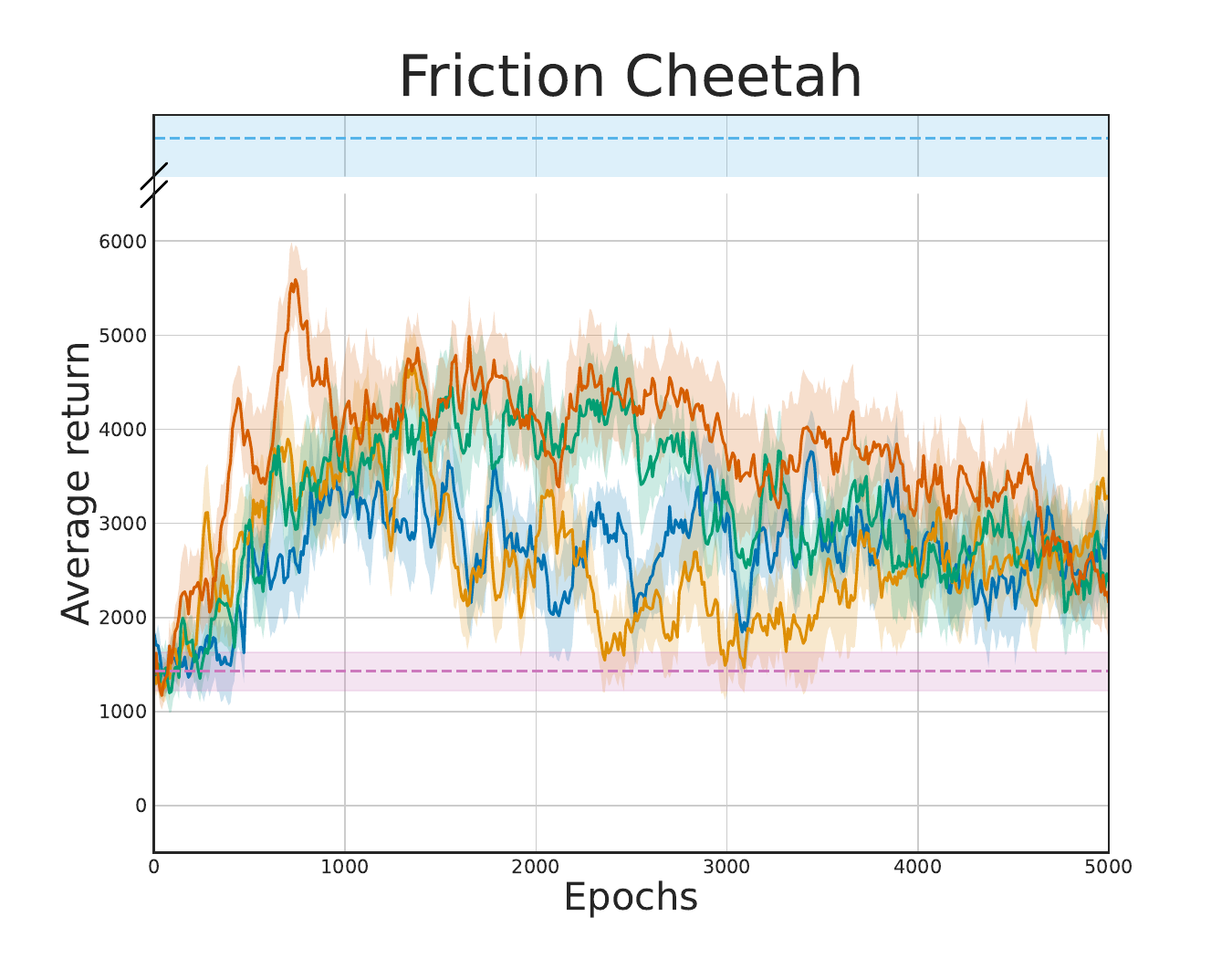}

\medskip

\includegraphics[width=.32\textwidth]{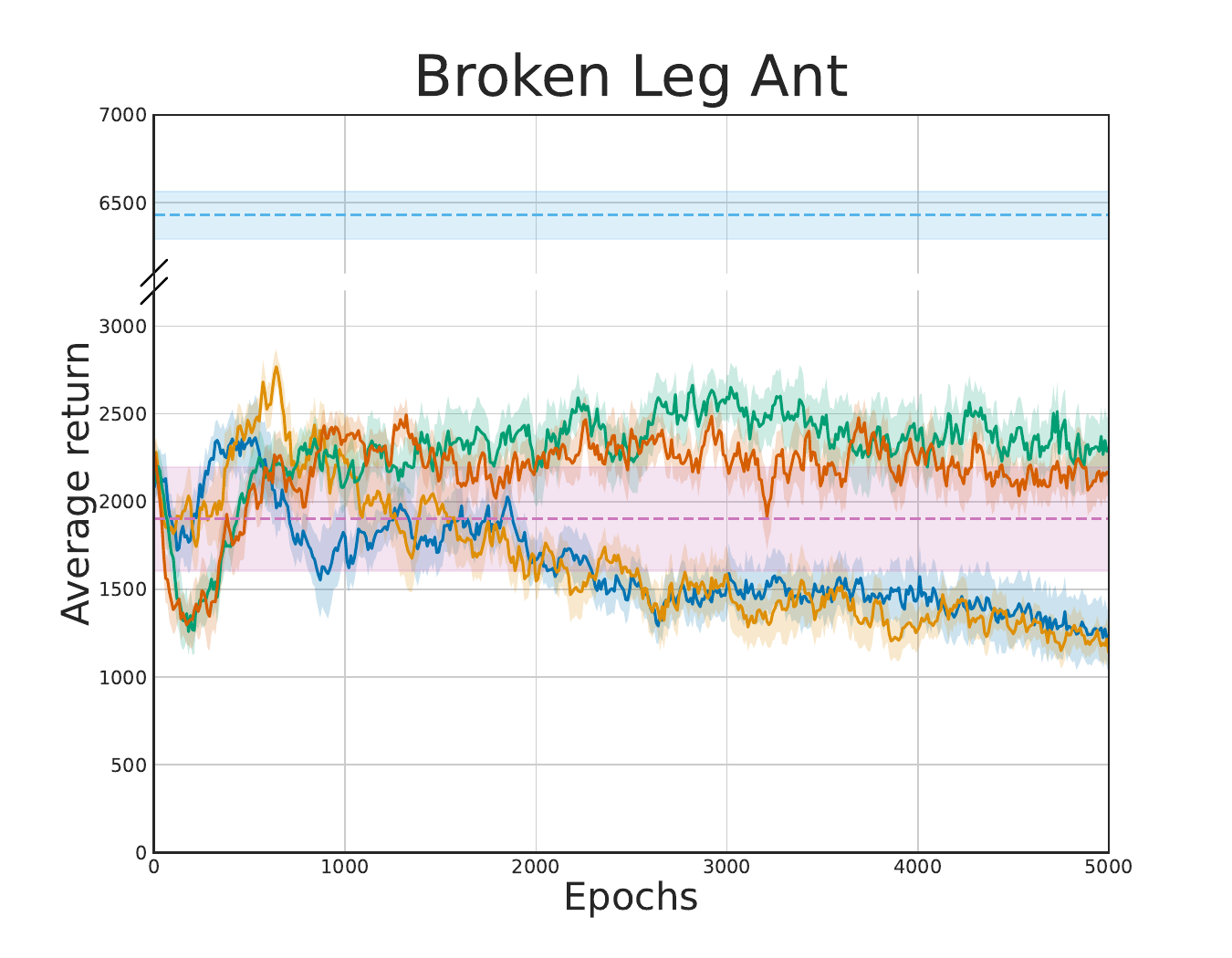}\quad
\includegraphics[width=.32\textwidth]{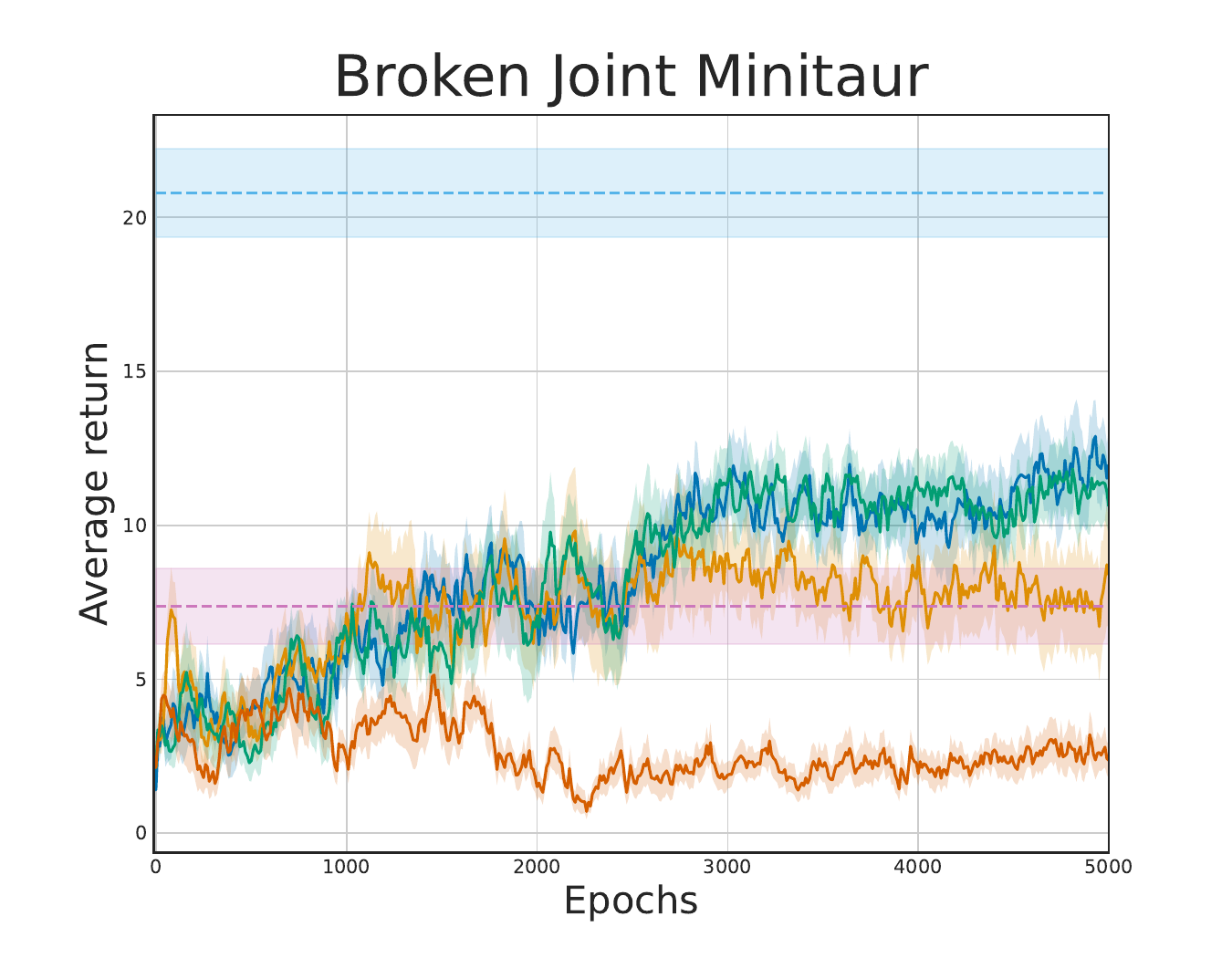}
\includegraphics[width=.32\textwidth]{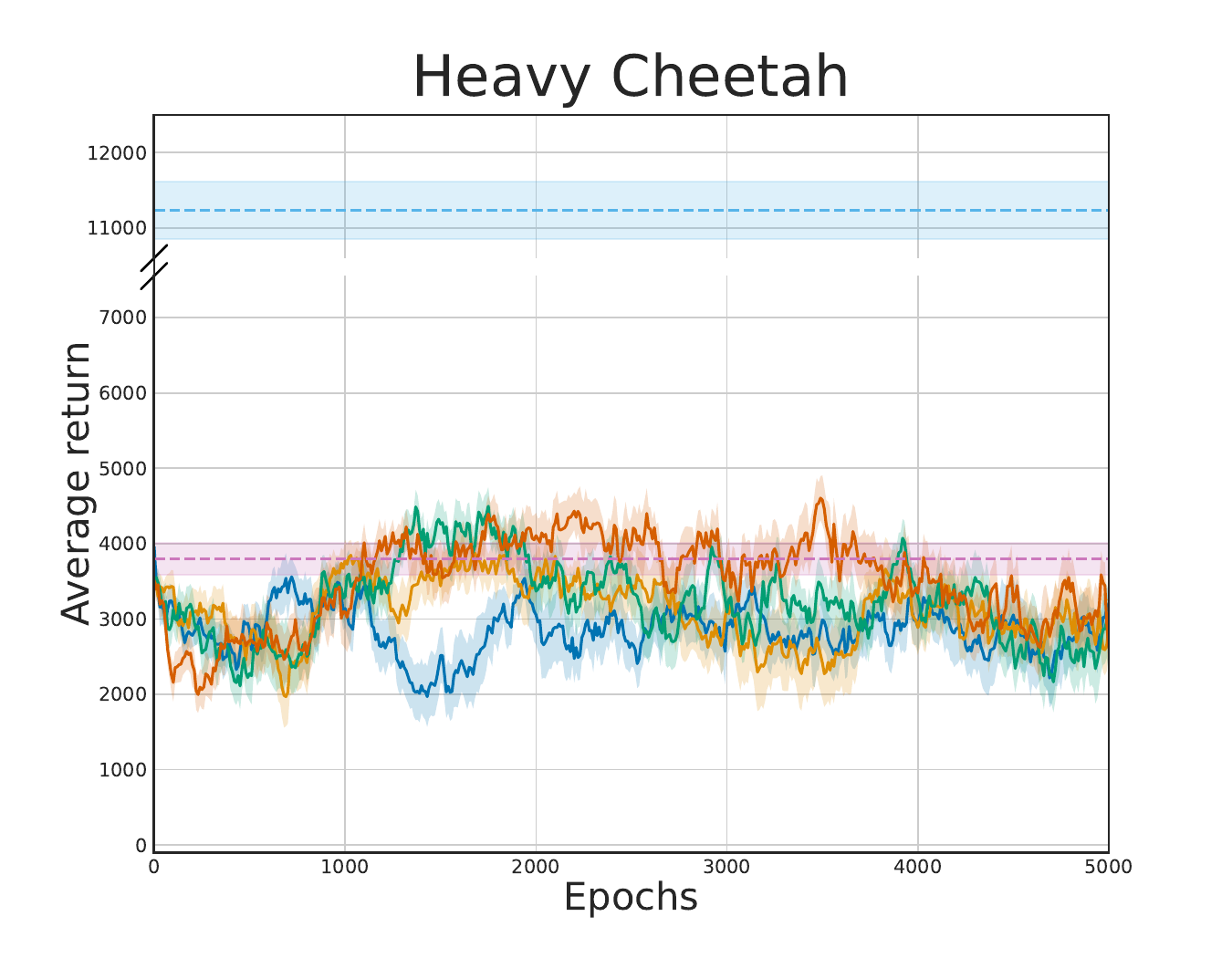}

\medskip

\includegraphics[width=.8\textwidth]{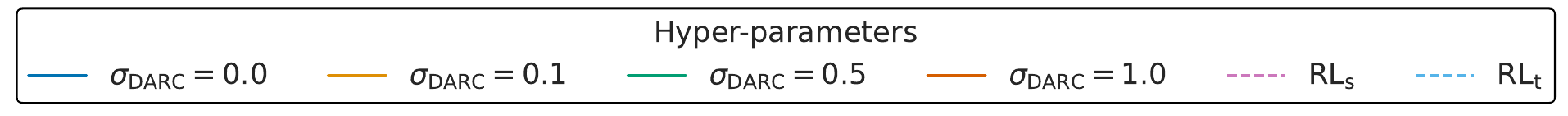}

\caption{Hyper-parameter sensitivity analysis for the DARC agent on the different environments where DARC works well.}
\label{fig:complete_hyperparam_darc}
\end{figure}

We detail in Figure~\ref{fig:complete_hyperparam_darc} DARC's sensitivity to its main hyper-parameter $\sigma_{\text{DARC}}$. We observe a clear dependence on the noise added to the discriminator, although there seems to be no pattern for choosing the right hyper-parameter. For instance, the best hyper-parameter for Broken Joint Cheetah and Broken Joint Ant is $\sigma_{\text{DARC}}=0.1$, but this value leads to worse performance than $\text{RL}_\text{s}$ on the two other presented environments. 

\subsection{ANE Hyperparameters Sensitivity Analysis} \label{app:ane_hyperparams_analysis}

We also detail the ANE's results for all environments in Figure~\ref{fig:complete_hyperparam_ane}. As a reminder, ANE adds a centered Gaussian noise with std $\sigma_{\text{ANE}}\in\{0.1, 0.2, 0.3, 0.5\}$ to the action during training.

\begin{figure}[ht!]
\centering
\includegraphics[width=.32\textwidth]{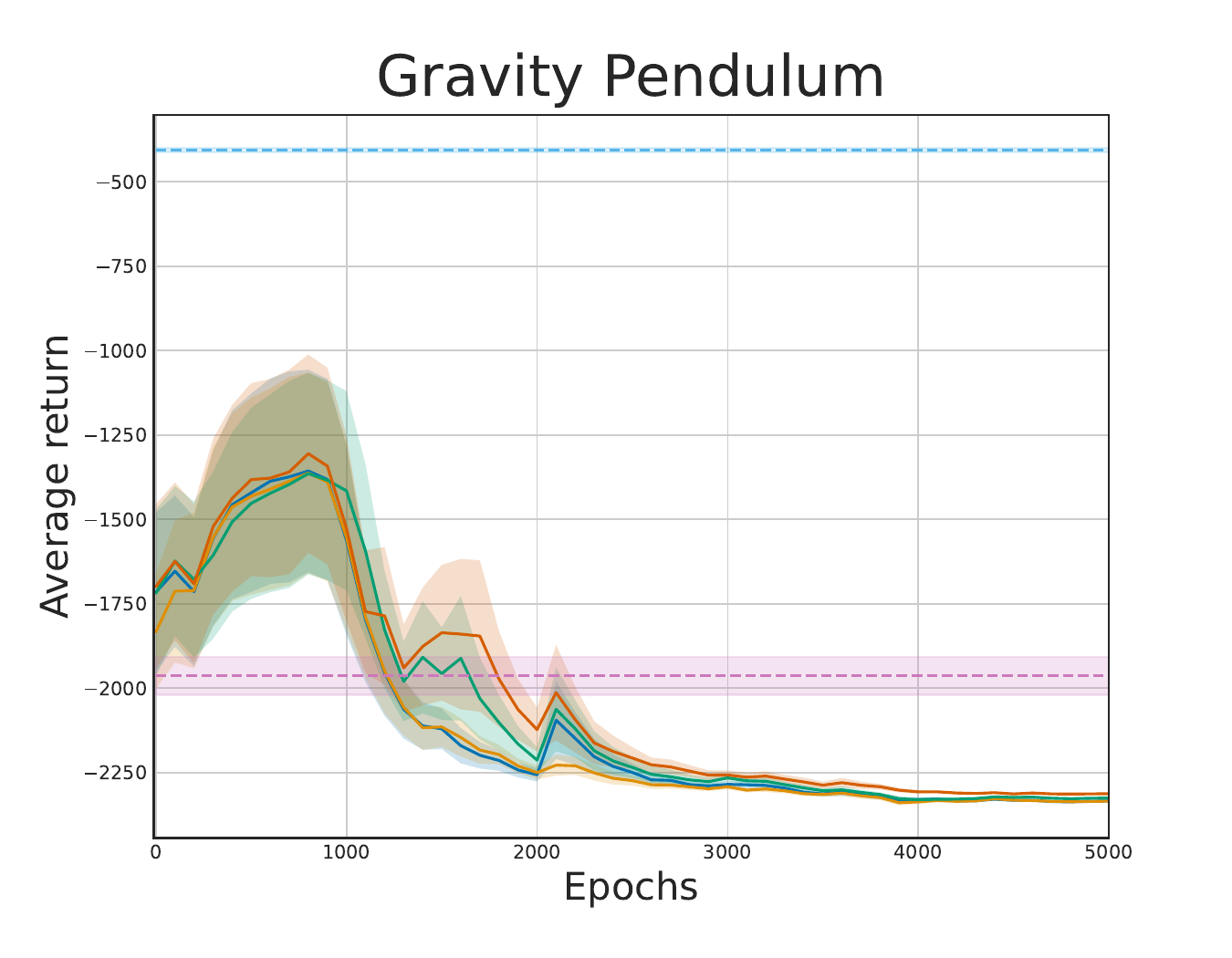}
\includegraphics[width=.32\textwidth]{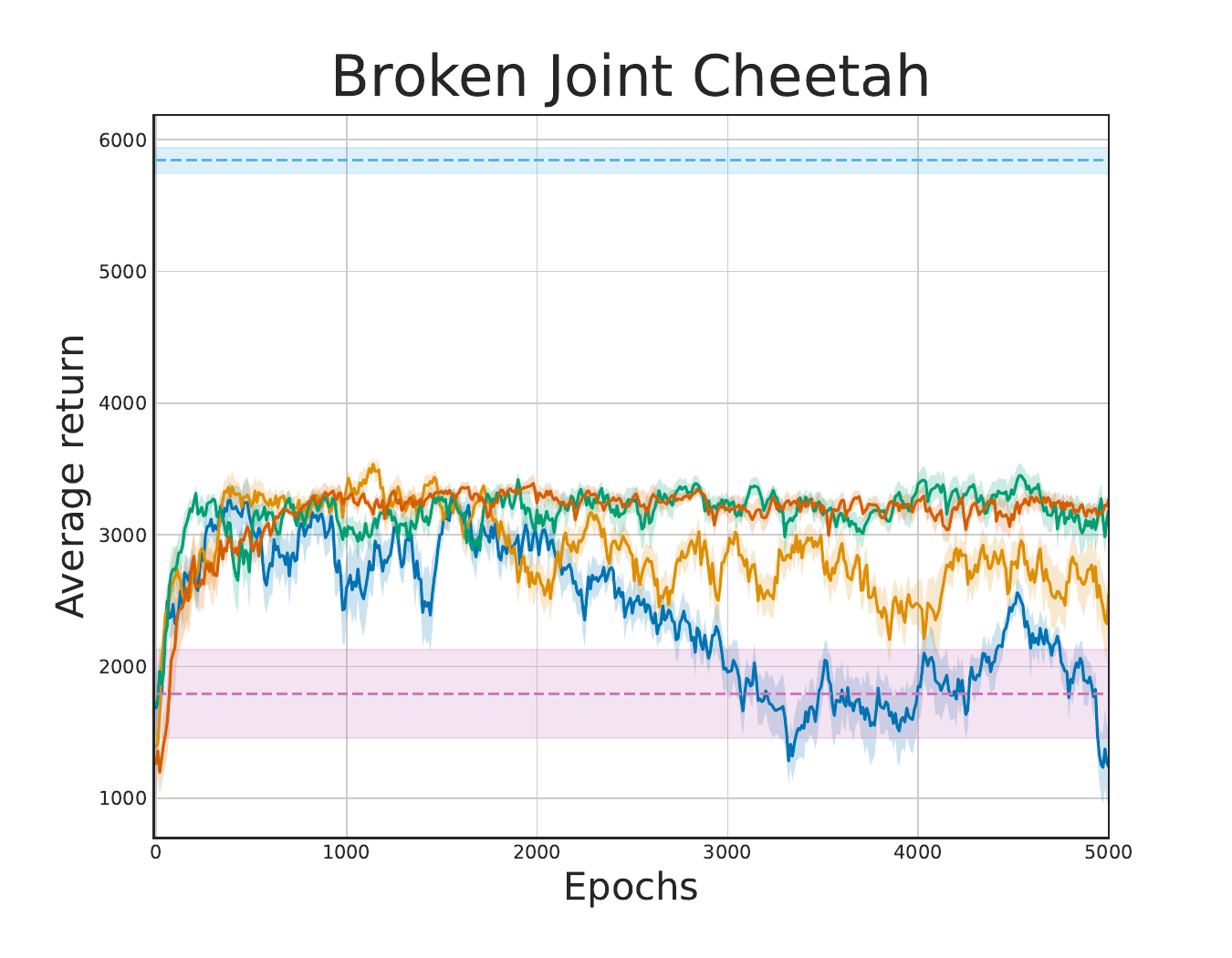}
\includegraphics[width=.32\textwidth]{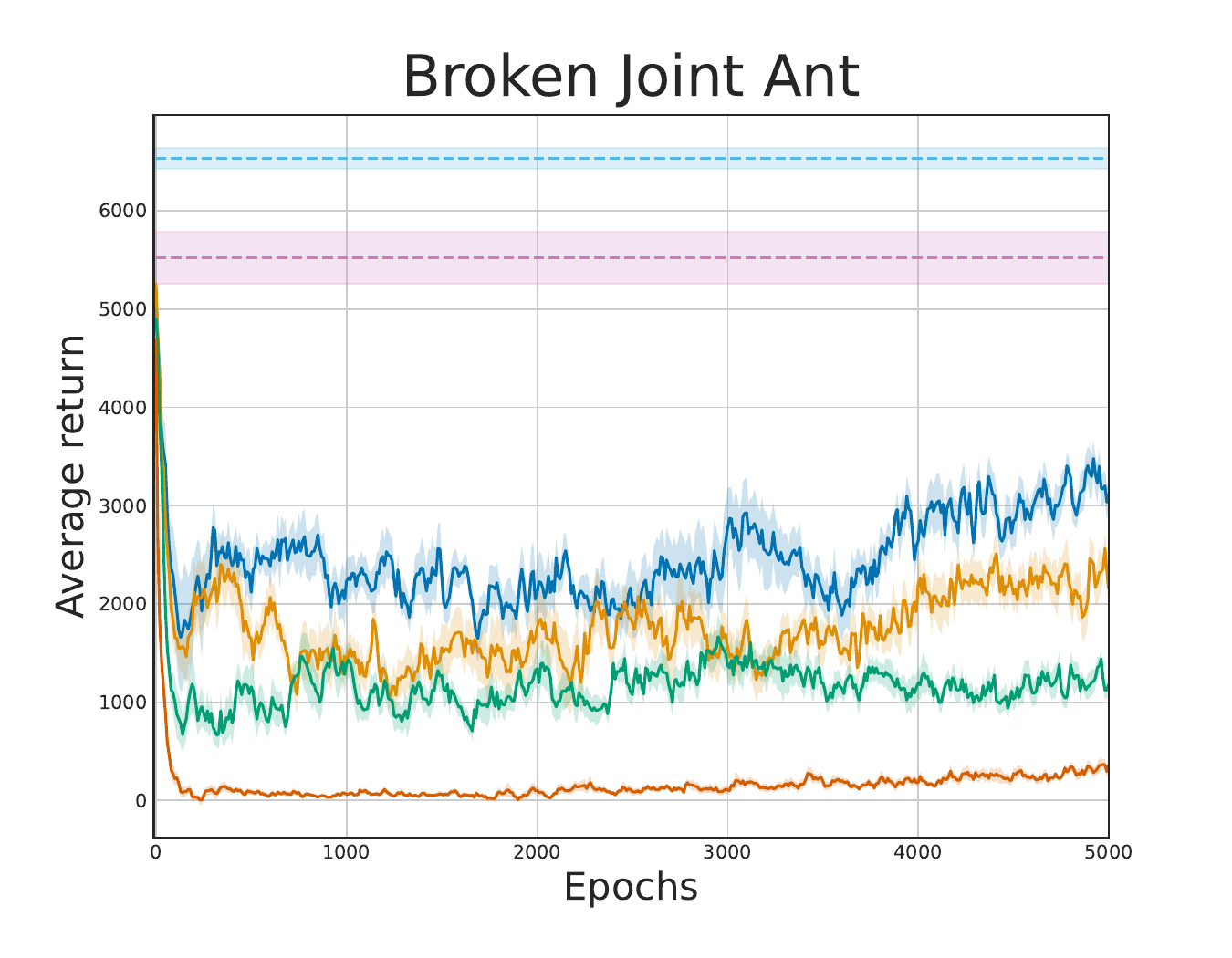}

\medskip

\includegraphics[width=.32\textwidth]{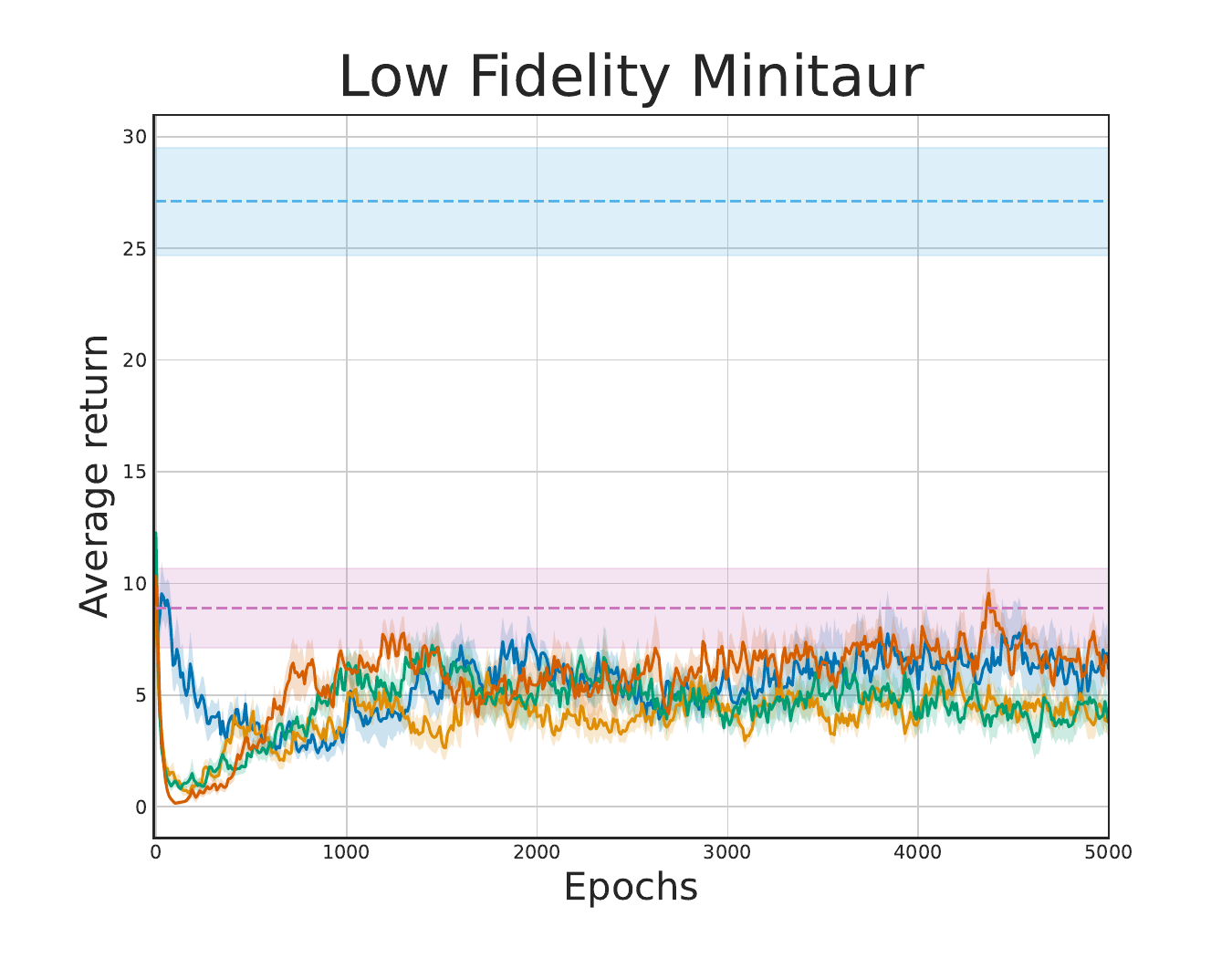}
\includegraphics[width=.32\textwidth]{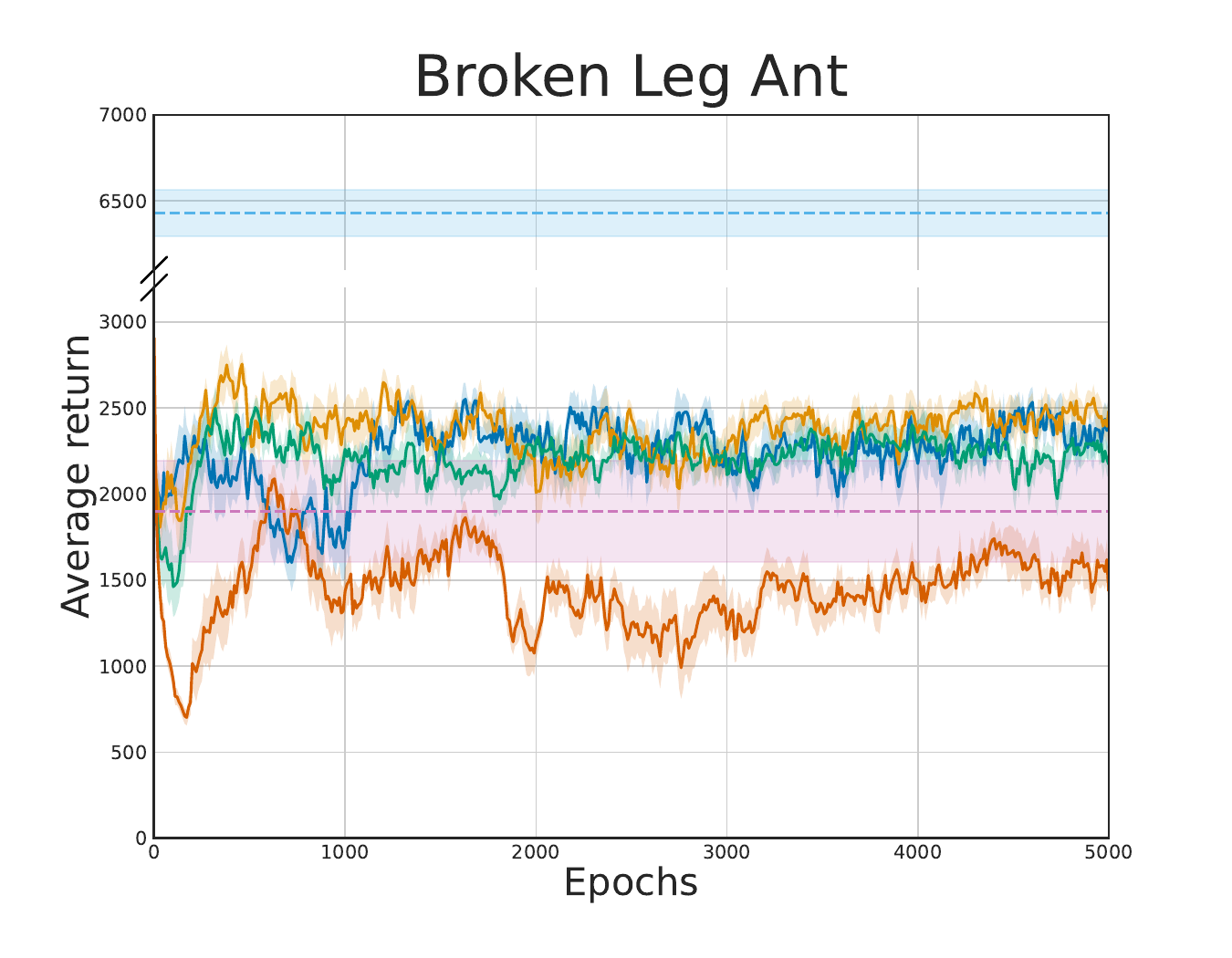}
\includegraphics[width=.32\textwidth]{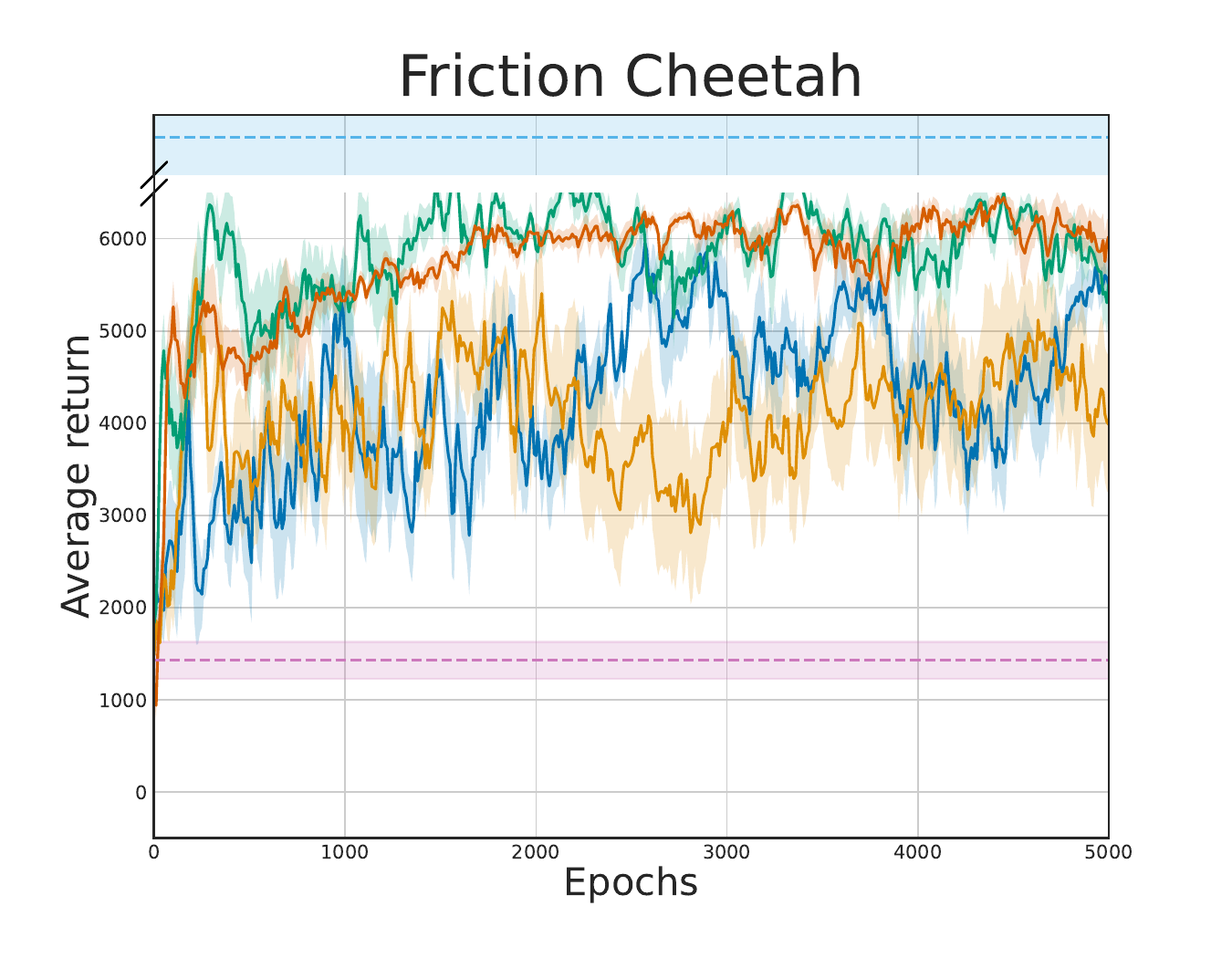}

\medskip



\includegraphics[width=.32\textwidth]{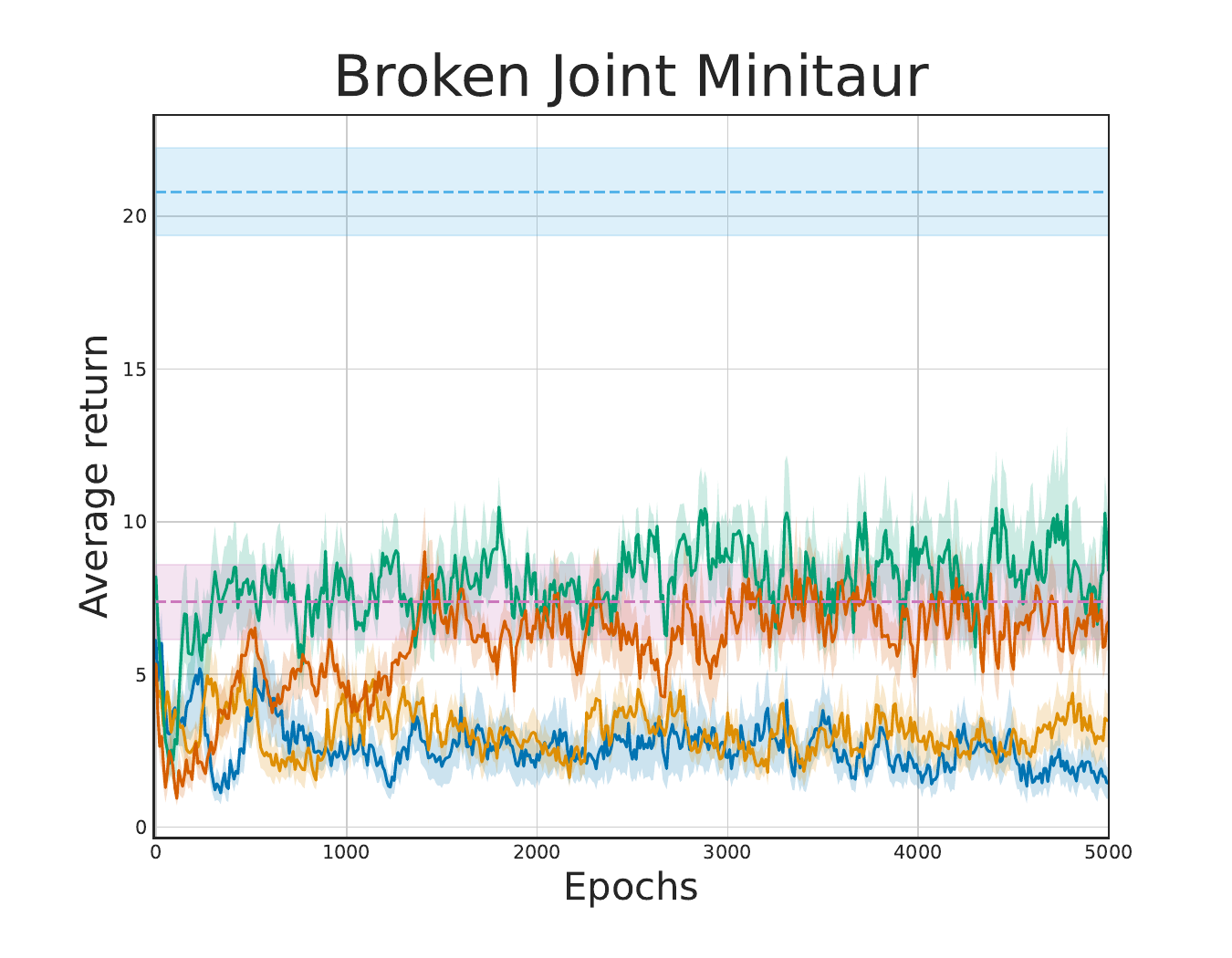}
\includegraphics[width=.32\textwidth]{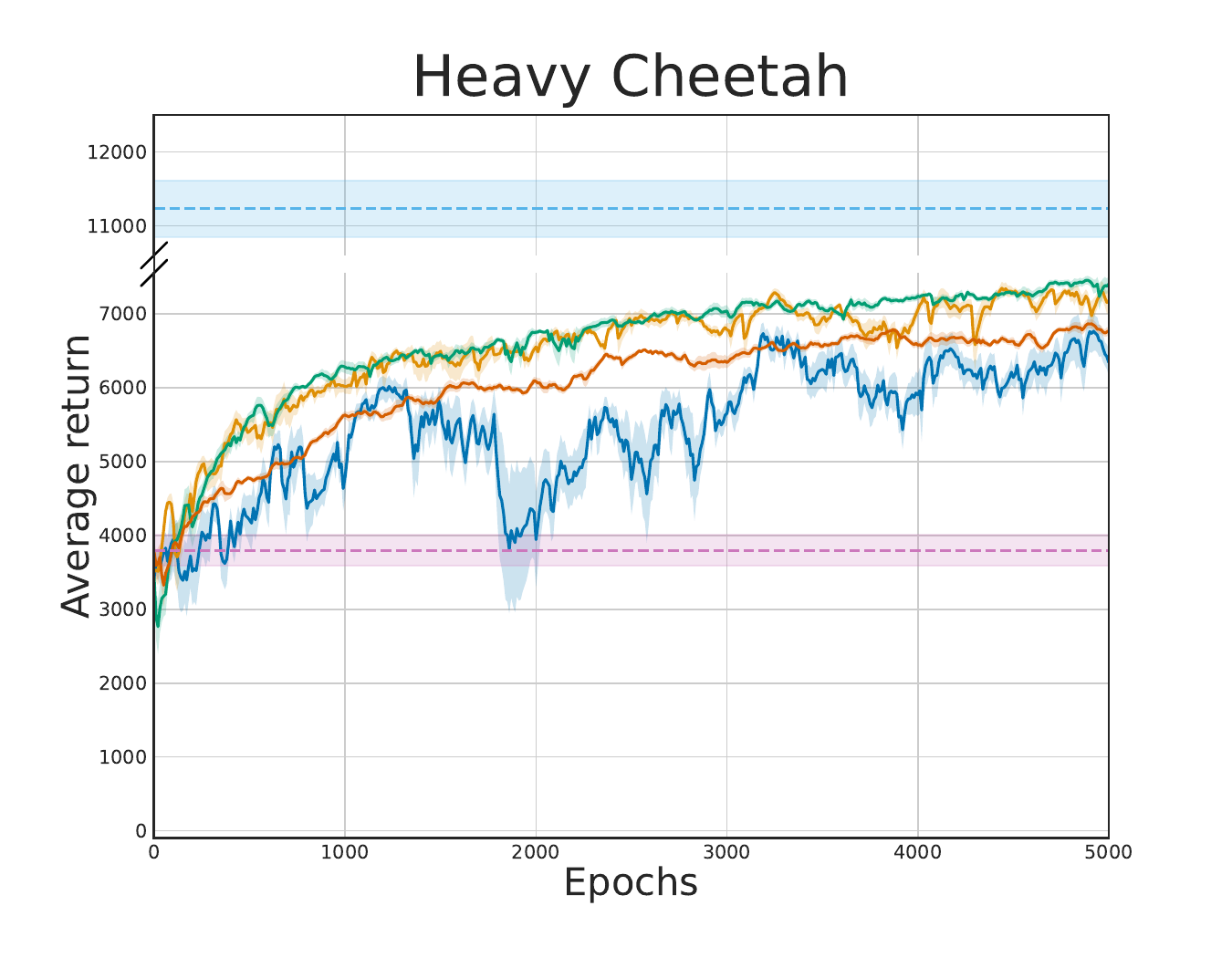}

\medskip

\includegraphics[width=.8\textwidth]{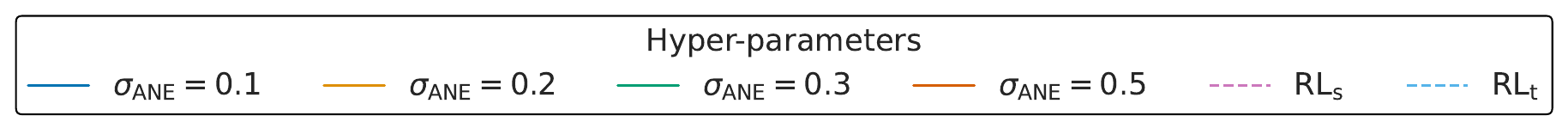}

\caption{Hyper-parameter sensitivity analysis for the ANE agent on the different environments.}
\label{fig:complete_hyperparam_ane}
\end{figure}

These figures are not easily interpretable. This technique may work very well as observed for Heavy Cheetah, but may fail for other environments such as Broken Joint Ant or Low Fidelity Minitaur.

\subsection{H2O results}\label{app:h2o_results}

Finally, we report H2O results in Figure~\ref{fig:h2o_results}. This method combines the regularization of DARC and CQL in the off-dynamics scenario when the agent has access to a large amount of target data. Since the agent also uses data from the source domain in its learning process, the strength of the regularization is lower than in CQL. It was set to $0.01$ in most of the benchmarks in H2O and to $1$ for the others. We did a grid search on these $2$ values. Given its poor results on the $2$ out of $3$ environments we tried and the high resources it requires, we did not try it on the other environments.

\begin{figure}[ht!]
\centering

\medskip

\includegraphics[width=.3\textwidth]{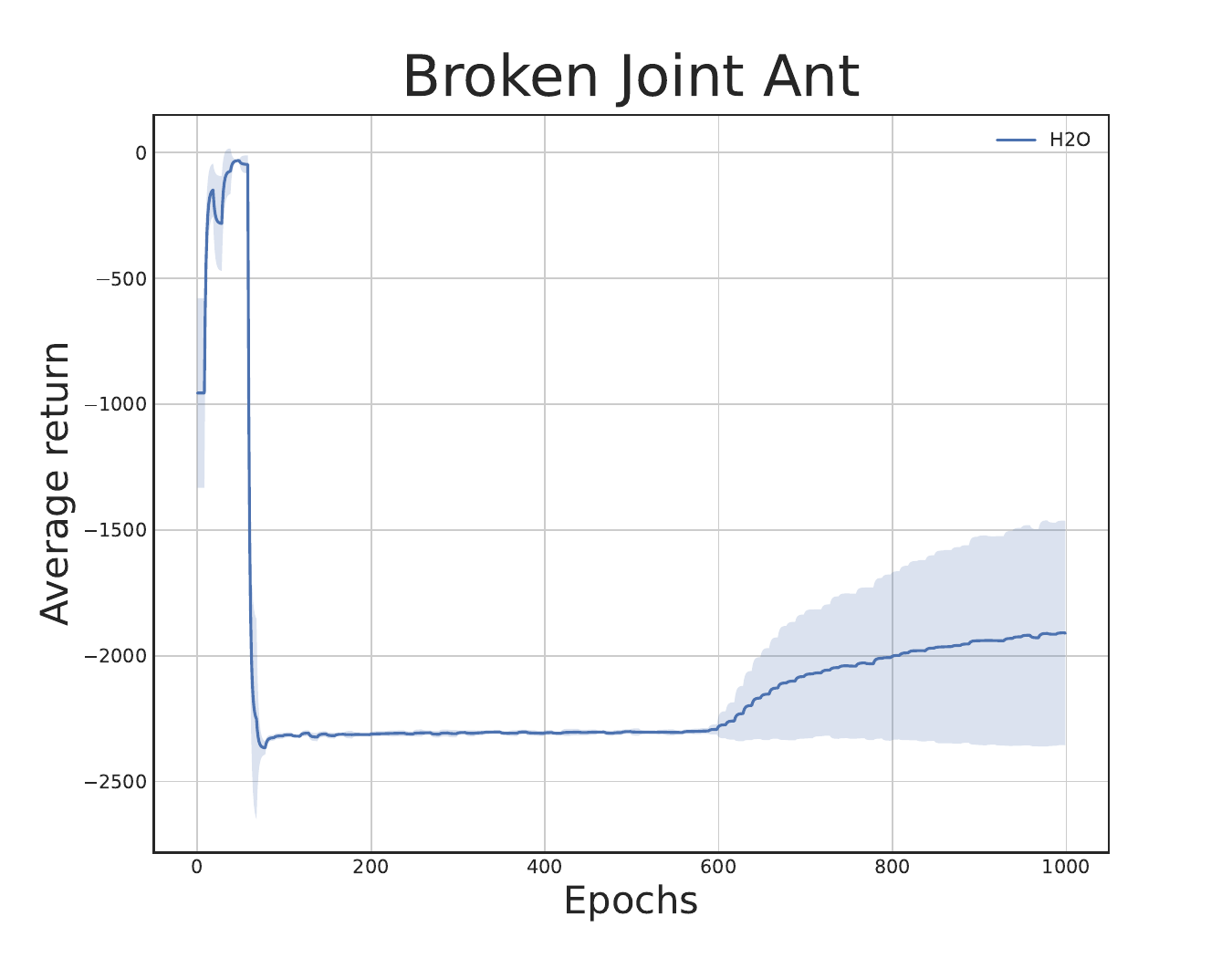}
\includegraphics[width=.3\textwidth]{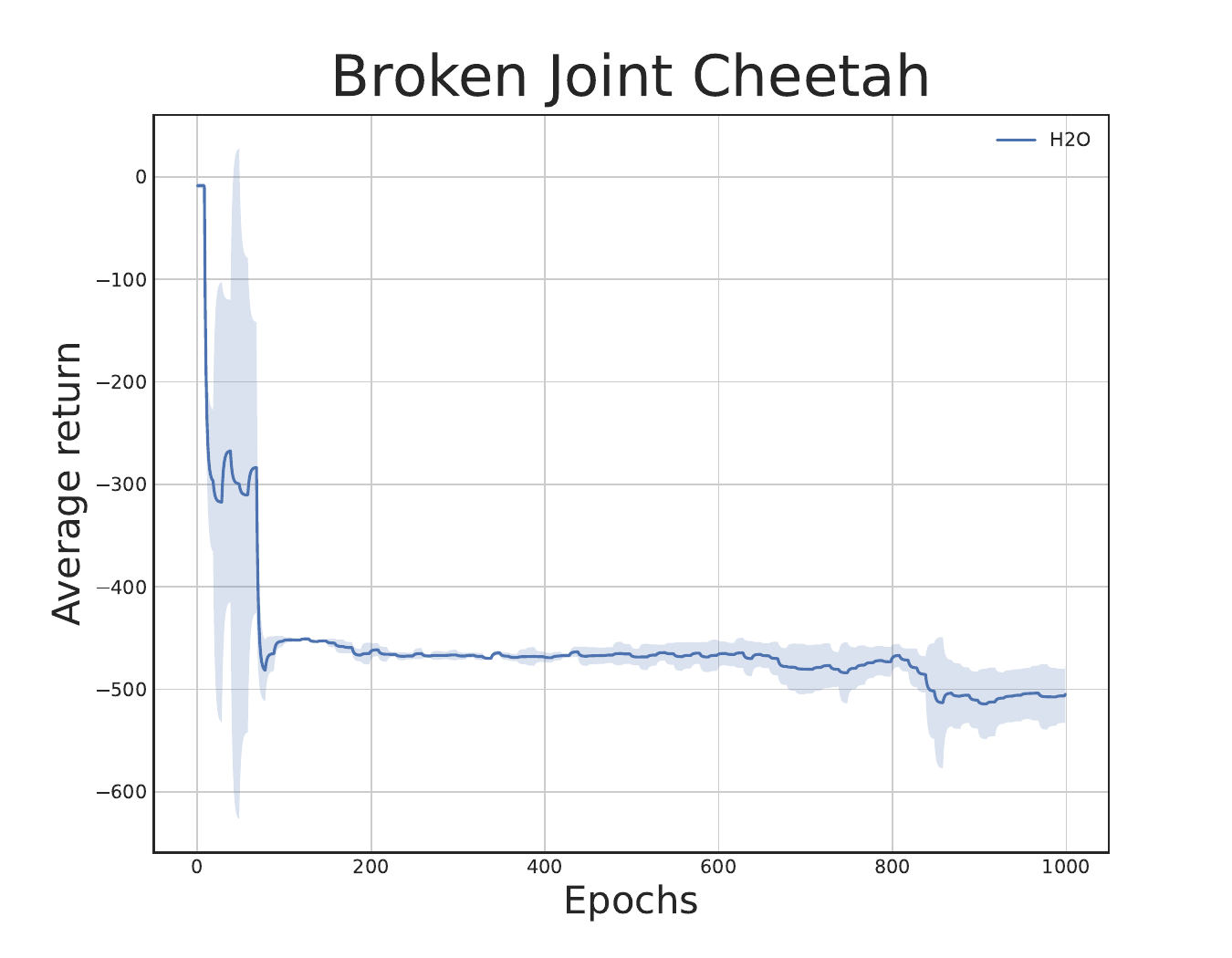}
\includegraphics[width=.3\textwidth]{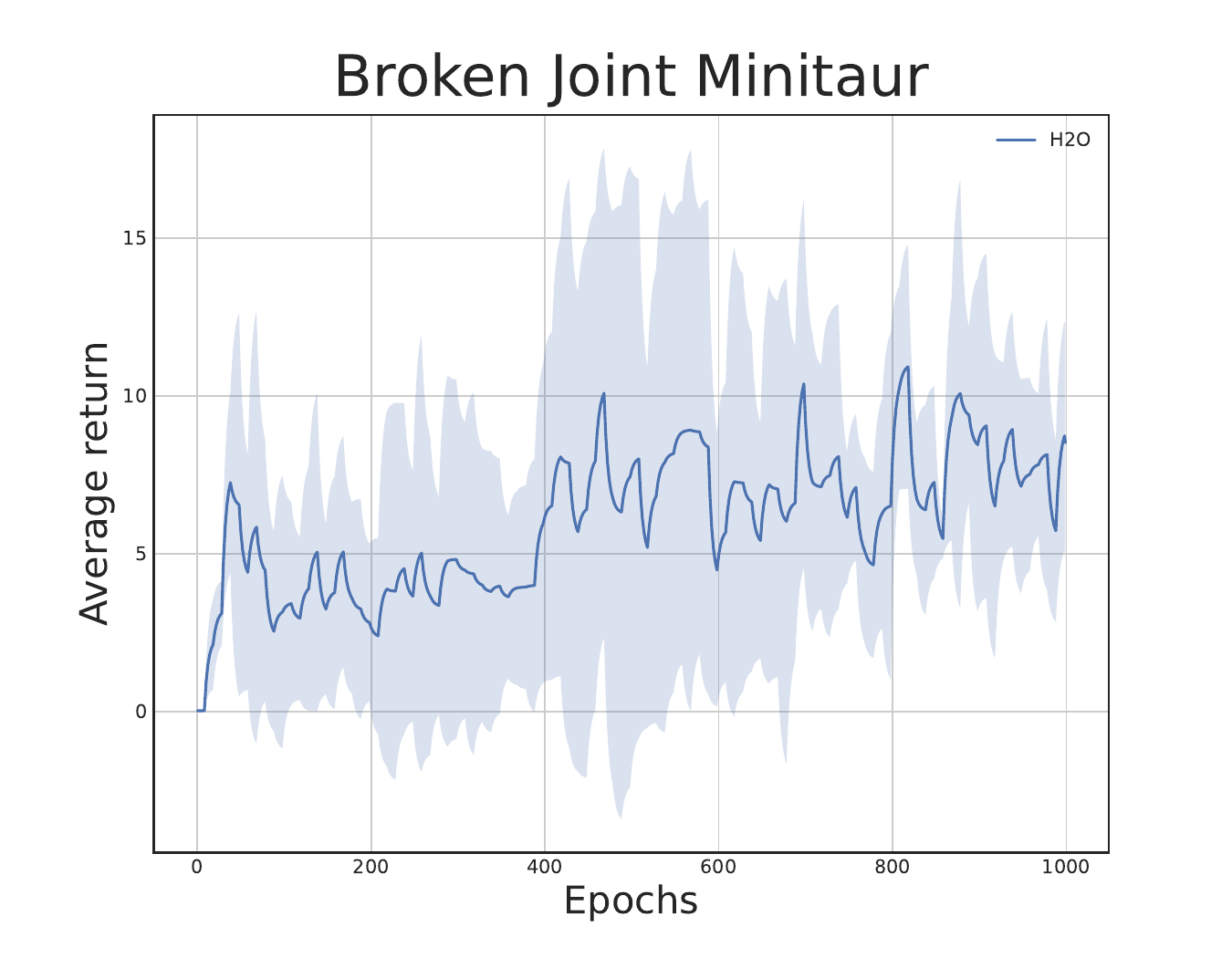}

\caption{H2O results on $3$ environments.}
\label{fig:h2o_results}
\end{figure}

\end{document}